\documentclass[onecolumn]{IEEEtran}
\usepackage{algorithmic}
\usepackage{algorithm}
\usepackage{array}
\usepackage[caption=false,font=normalsize,labelfont=sf,textfont=sf]{subfig}
\usepackage{textcomp}
\usepackage{stfloats}
\usepackage{verbatim}
\usepackage{amsmath,amssymb,amsthm}
\usepackage{graphicx, rotating}
\RequirePackage{multirow}
\usepackage{float}
\usepackage{url}
\usepackage[numbers,compress]{natbib}
\RequirePackage[colorlinks,citecolor=blue,linkcolor=blue,urlcolor=black]{hyperref}
\usepackage{paralist}
\usepackage{mathrsfs}
\usepackage{arydshln}
\usepackage{upgreek}
\usepackage{booktabs}
\theoremstyle{plain}
\newtheorem{theorem}{Theorem}
\newtheorem{lemma}[theorem]{Lemma}
\newtheorem{corollary}[theorem]{Corollary}
\newtheorem{definition}{Definition}[section]

\theoremstyle{remark}
\newtheorem{remark}{\textsc{Remark}}

\newcommand{\customvarkappa}{\rotatebox{2}{\scalebox{.9}[1.2]{$\varkappa$}}}

\newcommand{\bsb}{\boldsymbol}

\newcommand{\bsbB}{{\boldsymbol{B}}}
\newcommand{\bsbA}{{\boldsymbol{A}}}

\newcommand{\bsbV}{{\boldsymbol{V}}}
\newcommand{\bsbU}{{\boldsymbol{U}}}
\newcommand{\bsbD}{{\boldsymbol{D}}}

\newcommand{\bsbK}{{\boldsymbol{K}}}

\newcommand{\bsbX}{{\boldsymbol{X}}}
\newcommand{\bsbZ}{{\boldsymbol{Z}}}
\newcommand{\bsbI}{{\boldsymbol{I}}}

\newcommand{\bsbP}{{\boldsymbol{P}}}

\newcommand{\bsbSig}{{\boldsymbol{\Sigma}}}

\newcommand{\bsbzeta}{{\boldsymbol{\zeta}}}

\newcommand{\bsby}{{\boldsymbol{y}}}
\newcommand{\bsbs}{{\boldsymbol{s}}}

\newcommand{\bsbb}{{\boldsymbol{\beta}}}
\newcommand{\bsba}{{\boldsymbol{\alpha}}}
\newcommand{\bsbg}{{\boldsymbol{\gamma}}}

\newcommand{\bsbt}{{\boldsymbol{t}}}
\newcommand{\bsbh}{{\boldsymbol{h}}}

\newcommand{\bsbeps}{{\boldsymbol{\epsilon}}}
\newcommand{\bsbe}{{{\boldsymbol{\eta}}}}

\newcommand{\bsbxi}{{\boldsymbol{\xi}}}

\newcommand{\bsbz}{{\boldsymbol{z}}}

\newcommand{\bsbXi}{{\boldsymbol{\Xi}}}

\newcommand{\Proj}{{\mathcal P}}

\newcommand{\EP}{\,\mathbb{P}}
\newcommand{\EE}{\,\mathbb{E}}

\DeclareMathOperator{\rank}{rank}
\DeclareMathOperator{\sgn}{sgn}

\DeclareMathOperator{\vect}{\mbox{vec}\,}

\DeclareMathOperator{\range}{{\mathcal R}}
\DeclareMathOperator{\prox}{prox}
\DeclareMathOperator{\avg}{avg}
\DeclareMathOperator{\Ker}{Ker}
\DeclareMathOperator{\diag}{diag}
\newcommand{\ubar}[1]{\underline{#1\mkern-1.5mu}}

\newcommand{\overunderline}[1]{\underline{\smash{\overline{#1\mkern-2mu}}}}

\makeatletter
\newcommand{\dashover}[2][\mathop]{#1{\mathpalette\df@over{{\dashfill}{#2}}}}
\newcommand{\fillover}[2][\mathop]{#1{\mathpalette\df@over{{\solidfill}{#2}}}}
\newcommand{\df@over}[2]{\df@@over#1#2}
\newcommand\df@@over[3]{%
  \vbox{
    \offinterlineskip
    \ialign{##\cr
      #2{#1}\cr
      \noalign{\kern1pt}
      $\m@th#1#3$\cr
    }
  }%
}
\newcommand{\dashfill}[1]{%
  \kern-.5pt
  \xleaders\hbox{\kern.5pt\vrule height.4pt width \dash@width{#1}\kern.5pt}\hfill
  \kern-.5pt
}
\newcommand{\dash@width}[1]{%
  \ifx#1\displaystyle
    2pt
  \else
    \ifx#1\textstyle
      1.5pt
    \else
      \ifx#1\scriptstyle
        1.25pt
      \else
        \ifx#1\scriptscriptstyle
          1pt
        \fi
      \fi
    \fi
  \fi
}

\newcommand{\solidfill}[1]{\leaders\hrule\hfill}
\makeatother
\newcommand{\breg}{{\mathbf{\Delta}}}
\newcommand{\Breg}{{\mathbf{D}}}
\newcommand{\bregs}{{\bar{\mathbf{\Delta}}}}

\newcommand{\lsem}{\mathopen{\vcenter{\hbox{$\textnormal{\textbrokenbar\kern-0.1em}|$}}}}
\newcommand{\rsem}{\mathclose{\vcenter{\hbox{$|\kern-0.1em\textnormal{\textbrokenbar}$}}}}

\usepackage[normalem]{ulem}

\begin{document}

\title{Range Penalization: Theoretical Insights with Applications in Federated Learning}

\author{Yiyuan She, Zhaojun Hu, and Yifan Sun
}



\maketitle

\begin{abstract}

This paper introduces range regularization for federated learning with linear systematic components to enhance statistical accuracy and induce cross-client regularity conducive to quantization, coding, and resource efficiency. Our approach identifies features with shared weights across different clients and adaptively clusters the weights of personalized features at extreme values, a process we refer to as polar clustering. Theoretical analysis of the associated estimators poses significant challenges due to the seminorm nature and non-decomposability of the regularizer. We develop new proof techniques for the nonasymptotic analysis of statistical accuracy and faithful pattern recovery. Moreover, a fast optimization algorithm that leverages varying degrees of local strong convexity is proposed to reduce iteration complexity. Experiments support the efficacy and efficiency of the proposed approach.

\end{abstract}


\section{Introduction}\label{sec:intro}
\IEEEPARstart{T}{o} introduce the federated learning setting, consider $m$ devices, or \textit{clients},
each owning private
data $(\bsbX_k, \bsby_k)$. Due to escalating privacy concerns in today's digital landscape, these data sets cannot be shared between clients.   In this paper, we focus on federated learning with a linear systematic component. An objective for estimating the unknown coefficients (also called weights) is given by$$
\min_{\bsbB} \sum_{k=1}^m l_0(\bsbX_k \bsb{b}_k; \bsby_k) \text{ with }    {\bsbB}=[\bsb{b}_1, \ldots, \bsb{b}_m]=[\bsbb_1, \ldots, \bsbb_p]^T.
$$
However, this separable objective is less interesting because each column of $\bsbB$ can be estimated independently, without utilizing  information from other clients. To enhance statistical accuracy, federated learning researchers often assume the existence of \emph{shared} parameters \citep{arivazhagan2019federated}
$$\bsbb_j\propto \bsb1,
j\in \mathcal J^c$$ where   $\mathcal J\subset\{1, \ldots, p\}$ denotes the index set of nonuniform rows. In estimating shared parameters, federated learning  integrates insights  from  multiple clients, without the need for centralized storage or data sharing:    each client performs local updates based on its own data; these updates or statistics, rather than the raw data, are then transmitted to and aggregated by a central server.
Because identifying which features have shared coefficients across different clients can be challenging,  one of our goals is to automate this process.
Furthermore, there are \textit{personal} parameters due to model heterogeneity \citep{wu2020personalized}. The inclusion of both shared and personal parameters is an example of ``partial model personalization'' \citep{pillutla2022federated}: neither a global model, as in traditional federated learning \citep{mcmahan2017communication}, nor complete personalization is assumed.

To handle personalized parameters with potentially distinct components, some structural assumption is necessary. A natural possibility is to assume that the coefficients can be clustered across clients; see, for example, \cite{ghosh2020efficient} and \cite{liu2025robust}.
A common way to encourage such clustering is to {sparsify} pairwise differences: $\sum_{1\le j \le p}\allowbreak \sum_{1\le k<k'\le m}P(|\bsbb_{j}[k] -\bsbb_{j}[k']|;\lambda)$, where
$\lambda$ is a regularization parameter; this technique has been   actively explored, including     \cite{she2010sparse}, \cite{chi2015splitting}, and \cite{tang2016fused} among others.
At first glance, such a formulation appears to promote parsimony and parameter sharing in federated learning, and hence might seem capable of improving statistical accuracy.
However, this intuition is incomplete because it ignores the intrinsic cost of identifying the latent clustering structure, which is not simply a constant multiple of the number of distinct components
in  each   $\bsbb_j\in\mathbb R^m$. More specifically,   let $J^*$  denote the number of heterogeneous rows  in $\bsbB^*$ and   suppose each such row contains only a small, indeed constant-order, number of distinct components---a commonly used parsimonious assumption. Then our first theorem establishes an intrinsic information-theoretic limit for the prediction error rate,
 \begin{align}
J^* m + p + J^*\log(ep/J^*) \label{introrefrate}
\end{align}
up to trivial factors, regardless of the estimator.
The derivation and discussion are deferred to Appendix \ref{subsec:minimax} due to space constraints.
The leading term $J^* m $   shows that the statistical cost of identifying the clusters can completely offset the apparent gain from within-row clustering, rather than reducing to  merely    $J^*$ times a constant.  Existing pairwise-difference regularization methods fail to attain the lower bound above:
even when $P$  is chosen as the ideal $\ell_0$-penalty, the resulting error bound remains substantially \emph{larger} than the minimax rate (see, e.g.,  Appendix A.5 of \cite{SheetalCRL22}), and   hence is far from   optimal. This gap motivates a different regularization strategy.


Another critical issue in federated learning is computational efficiency \citep{almanifi2023communication}, which further explains why this paper does not aim to fully cluster     each $\bsbb_j$. Pairwise-difference regularization introduces $\mathcal O(m^2)$
  terms for every row of $\bsbB$,   and its optimization typically relies on operator-splitting methods such as ADMM or AMA. This makes the optimization scale large and the overall convergence slow in practice. In addition, such formulations are often burdensome to tune: empirically, a fine grid over the regularization parameter is  needed, and weighted pairwise schemes further rely on ad hoc  weight constructions \citep{chi2015splitting}. These considerations further motivate our focus on a simpler convex formulation together with an efficient optimization strategy.

Personalized federated learning has also been studied through several other mechanisms, including multi-task or shared-representation learning \citep{maurer2016benefit}, mixtures of global and local models \citep{deng2020adaptive}, hierarchical Bayesian formulations \citep{kim2025fedhb}, and meta-learning-based adaptation \citep{fallah2020personalized}.
These methods handle client heterogeneity through latent common representations, global-local interpolation, hierarchical probabilistic coupling, or rapid adaptation from a shared initialization, respectively, rather than through direct structural regularization of the client-specific coefficient matrix.
\\

Motivated by these considerations, this paper advocates a ``range-penalization'' strategy for federated learning. The proposed regularizer is designed not only to identify shared parameters, but also to adaptively form clusters at the extreme values of personalized parameters, a process we call \emph{polar clustering}. This combination is practically appealing for several distinct reasons.

\begin{enumerate}
\item \uline{Quantization and coding.}
A narrower dynamic range is more amenable to low-precision representation. For example, replacing 64-bit floating-point numbers by 4-bit or even 1-bit bilevel representations can reduce transmission cost by an \textit{order} of magnitude.
In addition to incurring smaller quantization distortion, tighter ranges can also yield a more concentrated symbol distribution,  favorable for entropy coding and related  schemes \citep{han2016deep}, and are better aligned with communication-efficient update encoding in federated learning \citep{sattler2020robust}.

\item \uline{Stability.}
Range reduction suppresses large coefficient fluctuations across clients. During training, this helps prevent extreme client-specific values from dominating the update dynamics,  improving numerical stability and reducing the risk of unstable behavior or divergence. 

\item \uline{Statistical regularization.}
Range reduction and polar clustering  provide effective regularization in estimation. By curbing excessively large client-specific coefficients and imposing structured shrinkage, they help control overfitting while avoiding unnecessary shrinkage of intermediate values and preserving informative heterogeneity.

\item \uline{Resource efficiency}.
Models with weights confined to a reduced range are easier to store and deploy on mobile and edge devices with limited memory and power budgets \citep{shah2021model}. This not only reduces memory footprint but can also lower power consumption in resource-constrained deployment \citep{shi2022toward}.

\item \uline{Privacy.}
Range reduction is also beneficial from a privacy-preserving perspective. When quantization is used in privacy-preserving federated learning, a smaller weight range is less sensitive to quantization error and reveals less fine-grained numerical information \citep{amiri2021compressive}.
Moreover, polar clustering makes it harder to infer highly specific characteristics of individual clients from the learned coefficients, thereby offering additional anonymization.

 \end{enumerate}
These considerations indicate a distinct need for a regularizer that can (i) adaptively reduce the range of weights, (ii) maintain intermediate values without the additional shrinkage  seen in, for example,        $\ell_1$-type regularization, (iii) encourage uniformity to enhance parameter sharing, and (iv) enable automatic clustering of extreme values. The group range penalization proposed in this paper  provides a versatile solution that meets all four objectives.

Nevertheless, for block range penalization, existing literature still lacks a comprehensive theoretical analysis  of the optimal regularization level, the sharp error rates of the resulting estimator, and the conditions and probabilities under which authentic structural patterns can be recovered.
The \textbf{non-decomposability} of the proposed regularizer, which is especially challenging in high-dimensional analysis, calls for innovative techniques. Traditional approaches based on dual (semi)norms lead to suboptimal parameter choices and overly conservative error rates.
This paper presents a novel proof device that integrates optimization and statistical analysis to bound the stochastic term in nondecomposable settings. By leveraging the weak differentiability of the range seminorm, we derive less stringent regularity conditions that capitalize on both groupwise and within-group structures. We also explore pattern recovery in group range penalization, thereby extending the concept of ``irrepresentable conditions'' from lasso analysis. Furthermore, we propose a new momentum-based acceleration scheme that reduces iteration complexity by adapting to varying degrees of local restricted strong convexity, thereby substantially cutting communication costs.

The remainder of the paper is organized as follows: Section \ref{sec:rangefacts} examines fundamental properties of the range function, including its subgradients, directional derivative, proximity operator, dual seminorm and conjugate. Section \ref{sec:flsetting} introduces the group range penalized federated learning setting along with essential notations and conditions. Section \ref{sec:nonasymstat} details our principal findings and methods for finite-sample theoretical analysis, focusing on the estimator's error and pattern recovery.
  Section \ref{sec:comp} studies how to reduce iteration complexity for problems that may have   nonconvex losses in high-dimensions,  through the use of   varying sequences of convexity measures and relaxation parameters.   A summary is provided  in Section \ref{sec:summ}.
Due to limited space, the technical details, including a minimax theorem and dual seminorm analysis,    are provided in Appendix \ref{sec:tech}, while simulations and real data analysis examples are presented in Appendix \ref{sec:exps}.\\

\paragraph*{Notations} Throughout the paper, plain symbols denote scalars, whereas bold symbols denote vectors or matrices. Given a vector $\bsbb\in \mathbb R^m$, define the range function
\begin{align}
\range(\bsbb) =\| \bsbb\|_{\range}= \max \beta_k - \min \beta_k.
\end{align}
Although we often  write the range as $\| \cdot\|_{\range}$, it is \emph{not} a norm but a seminorm (see Section \ref{sec:rangefacts}). The indicator function $1_A(x)$  takes $1$ if $x\in A$ and 0 otherwise, while $\iota_A(x)=-\log 1_A(x)$ takes $0$ if $x\in A$ and $+\infty$ otherwise. Given a set $A$ and real number $c$, $A+c :=\{a+c: a\in A\}$. Given a real number $x$, $\sgn(x)=1$ if $x>0$, $-1$ if $x<0$, and $0$ otherwise.
 Given $\bsba, \bsbb\in \mathbb R^m$, $\bsba\preceq \bsbb$ means $\alpha_k\le \beta_k$ for all $ 1\le k \le m$. Define $[m]:=\{1, \ldots, m\}$ for any $m\in \mathbb N$. The cardinality of an index set $\mathcal I$ is denoted by $| \mathcal I|$, and its complement   by $\mathcal I^c$.

Given $\bsbb\in \mathbb R^m$ with  $\max \beta_j > \min \beta_j$, let the index sets for the \textit{maximum}, \textit{minimum}, and \textit{intermediate} values be defined as follows:\begin{align}
\overline {\mathcal M}(\bsbb) := \{j : \beta_j = \max \beta_j\} , \  \underline{\mathcal M}(\bsbb) := \{j : \beta_j = \min \beta_j\}, \\
\mathcal M(\bsbb) := \{j: \min\beta_j < \beta_j < \max \beta_j\}, \   \overunderline{\mathcal M} (\bsbb)= \overline{\mathcal M}(\bsbb) \cup \underline{\mathcal M}(\bsbb).
\end{align}
where $\mathcal M(\bsbb) $ can be empty.  In addition, let
\begin{align}
\overline{ M}(\bsbb) = |\overline {\mathcal M}(\bsbb)|, \underline{ M}(\bsbb) = |\underline{\mathcal M}(\bsbb)|,
M(\bsbb) = | {\mathcal M}(\bsbb)|, \overunderline{M}(\bsbb) = |\overunderline{\mathcal M}(\bsbb)|.
\end{align}
When    $\max \beta_j = \min \beta_j$, we still use  the notations $\overline {\mathcal M}(\bsbb), \underline{\mathcal M}(\bsbb), \overunderline{\mathcal M}(\bsbb)  $, which are however all identical, and so    $ \overunderline{M}(\bsbb)  =  \overline{M}(\bsbb)  =  \underline{M}(\bsbb) =m$ and $\mathcal M(\bsbb)=\emptyset$.

Given a matrix $\bsbA$, we use  $\bsbA[\mathcal I, \mathcal J]$ to  denote   a submatrix of $\bsbA$ with  rows and columns indexed by $\mathcal I$ and $\mathcal J$, respectively. For a vector $\bsba$, we sometimes use $\bsba_{\mathcal I}$ to denote $\bsba[\mathcal I]$. The standard column-stacking vectorization operator is denoted by   $\vect(\cdot)$.
Given any matrix $\bsbA$, we use     $\Proj_{\bsbA}$  to denote the orthogonal projection matrix onto the column space of    $\bsbA$, i.e., $\Proj_{\bsbA}=\bsbA(\bsbA^{T}\bsbA)^{+}\bsbA^{T}$ with $^+$ the Moore-Penrose inverse, and by $\Proj_{\bsbA}^{\perp}$ the projection onto  its orthogonal complement. We treat orthogonal projections onto a subspace as \textit{equivalent} to the subspace itself when there is no ambiguity. Given  $\bsbA = [\bsba_1, \ldots, \bsba_p]^T\in \mathbb R^{n\times m}$,  $\|\bsbA\|_F:=(\sum \|\bsba_j\|_2^2  )^{1/2}$ (the Frobenius norm), $\|\bsbA\|_2:=\sigma_{\max}(\bsbA)$ (the operator norm), $\|\bsbA\|_\infty := \max_{1\le j \le p} \|\bsba_j\|_1 $. In this paper, $\|\bsbA\|_\infty$     is also denoted by $\| \bsbA\|_{1,\infty}$  and  we define   $\|\bsbA\|_{2,\infty} := \max_{1\le j \le p} \|\bsba_j\|_2 $ and $\|\bsbA\|_{ \range, 1} := \sum_{1\le j \le p} \range( \bsba_j) $ in a similar manner.
Throughout the paper, we use $C,c$ to denote positive constants, which are not necessarily the same at each occurrence. We write    $a\lesssim b$ if $a \le C b$ for some constant $C>0$ and   $a\asymp b$ if both $a \lesssim b$ and $b\lesssim a$.
In addition, a hat consistently denotes an estimate of the corresponding quantity; for example, $\hat \bsbb$ denotes an estimate of  $\bsbb$.
\section{Range: A Nondecomposable Seminorm}
\label{sec:rangefacts}
Despite the  benefits of range regularization in quantization, coding, resource efficiency, and privacy protection, as discussed in Section \ref{sec:intro}, range-based regularizers have received much less theoretical attention than the well-studied $\ell_1$-norm. Although range penalization appeared earlier as a statistical regularizer in \cite{Shethesis}, its analytic properties and nonasymptotic behavior have not been systematically studied.
Therefore, this section presents some basic facts about  \(\range(\cdot)\) that are useful for theoretical analysis and computation.

The primary challenges we face with \(\range(\cdot)\) stem from the fact that   it is merely a seminorm \textit{and} nondecomposable. The former implies that it has a nontrivial kernel space.
The latter issue is particularly severe because traditional \(\ell_1\)-type analysis techniques for lasso, group lasso, and nuclear norm require decomposability \citep{Wainwright2019}.

We first introduce the notion of a seminorm; see, for example,  \cite{Pasq2024}.
\begin{definition}
Given a vector space $V$,  $f: V\rightarrow \mathbb R$ is called a \emph{seminorm} if (i) $f( x + y )\le f( x) + f( y), \forall x, y\in V$ and (ii) $f(  c x)= |c| f( x), \forall c\in \mathbb R, x\in V$.
\end{definition}
According to the definition,  by setting $y = -x$, we have $f( x)\ge 0$ for any $ x \in V$.
However, $f(x)=0$ does not necessarily mean $x=0$. The kernel  of a seminorm $f$, denoted by $$\Ker f:=\{x\in V: f(x) = 0\},$$ is  easily shown to be   a linear subspace.

\begin{theorem}\label{thm:semi}
 $\| \cdot \|_{\range}$ is a seminorm on $\mathbb R^m$ and ${\Ker \| \cdot \|_{\range}}$ is the subspace spanned by ${\bsb1}_m$.
\end{theorem}

For $\range(\cdot)$ defined on $\mathbb{R}^m$, $\Proj_K = \Proj_{\bsb1_m}$  according to the theorem.  From now on, we use $\Proj_K =   \Proj_{\Ker \| \cdot \|_{\range} }$  to denote the orthogonal projection onto the kernel of $\range(\cdot) $, \emph{without} specifying the dimension. That is, each occurrence of $\Proj_K$ may represent different projections due to varying dimensions.

Unlike the $\ell_1$-norm or the nuclear norm,  $\| \cdot \|_{\range}$  is ``nondecomposable'' in the sense that   the following does not hold in general:
$$\|\bsbb \|_{\range} = \|\Proj^*\bsbb  \|_{\range} + \| \Proj^{*,\perp}\bsbb\|_{\range}, \forall \bsbb
$$
where  $\Proj^*$ is   an orthogonal  projection  associated with the true model and  $\Proj^{*,\perp}$ is  its orthogonal complement (cf. \eqref{trumodelspace}).

Decomposability plays a crucial role in the nonasymptotic analysis of regularized  M-estimators (as discussed in  Chapters 9 and   10 of \cite{Wainwright2019}). In the absence of the property, techniques like index splitting should be replaced with certain \textit{nonlinear} operations. To overcome this challenge, we explore the differentiability of   range.

Though   $\range(\bsbb)$ is not additive in each component of
$\bsbb$ like the $\ell_1$  norm,  it is convex.
Recall that for a convex function $f$ defined on $\mathbb R^m$,   $\bsbs$ is  subgradient of $f$  at $\bsbb$ if
\begin{align}
f(\bsbg) -f(\bsbb) \ge \langle \bsbs, \bsbg - \bsbb\rangle, \forall \bsbg \label{defsubgrad}
\end{align}
and the set of all subgradients at $\bsbb$ is denoted by $\partial f(\bsbb)$.
Define the \textit{one-sided} directional derivative of   $f$ (not necessarily convex) at $\bsbb$ with increment $\bsba$:
$$
\delta f(\bsbb; \bsba) = \lim_{t\rightarrow 0+} \frac{f(\bsbb+t\bsba) - f(\bsbb)}{t}.
$$
For a real-valued convex function defined throughout  the   Euclidean space,  $\partial f(\bsbb)\ne \emptyset$, and
 $\delta f(\bsbb; \bsba)$ exists for all $\bsbb$ and $  \bsba$ \citep{Rockafellar1970}.
\begin{theorem}\label{thm:subgrad}
Given $\bsbb\in \mathbb R^m$, the subdifferential $\partial \range(\bsbb)$ can be determined  as follows:  If  $\max \beta_k > \min \beta_k$,
\begin{align}
\begin{split}\partial \range(\bsbb)=\{\bsbs\in \mathbb R^m:\ &\bsbs_{\overline{\mathcal M}(\bsbb)}\succeq \bsb0, \langle \bsb1,  \bsbs_{\overline{\mathcal M}(\bsbb)}\rangle =1, \\ & \bsbs_{\mathcal M(\bsbb)} = \bsb0, \\&  \bsbs_{\underline{\mathcal M}(\bsbb)}\preceq \bsb0, \langle \bsb1,  \bsbs_{\underline{\mathcal M}(\bsbb)}\rangle =-1\}.
\end{split}\end{align}
  If  $\max \beta_k = \min \beta_k$, \begin{align}
\partial \range(\bsbb) & =\{\bsbs\in \mathbb R^m:   \bsbs =  \bsbs_1 - \bsbs_2, \bsbs_1\succeq \bsb0, \bsbs_2 \succeq \bsb0, \langle \bsb1, \bsbs_1\rangle=\langle \bsb1, \bsbs_2\rangle  = 1\}\notag\\
& = \{\bsbs\in \mathbb R^m:   \langle \bsb1, \bsbs\rangle = 0, \| \bsbs\|_1 \le 2 \}.\label{rangsubdiff}
\end{align}

Moreover, the \emph{one-sided} directional derivative of the range seminorm is given by
\begin{align*}
\delta \range   (\bsbg; \bsbb - \bsbg) &  = \begin{cases}
\max_{k\in \overline {\mathcal M}(\bsbg)} (\beta_k - \gamma_k) - \min_{k\in \underline {\mathcal M}(\bsbg)}(\beta_k - \gamma_k) , & \mbox{ if } \Proj_K^\perp\bsbg\ne \bsb0\\
\|\bsbb - \bsbg\|_{\range}, & \mbox{ if } \Proj_K^\perp \bsbg=\bsb0.\end{cases}
\end{align*}
\end{theorem}

The subgradients of the range seminorm   exhibit a much more intricate structure compared to those of the $\ell_1$-norm.
In \textit{both} cases,     $\bsbs\in \partial \range(\bsbb)$  necessarily satisfies $\langle \bsb1 , \bsbs\rangle = 0$ and $ \| \bsbs\|_1 \le 2$.
Moreover, because in the scenario where   $\Proj_K^\perp\bsbg =\bsb0$,   $ \overunderline{\mathcal M}(\bsbb)  =  \overline{\mathcal  M}(\bsbb)  =  \underline{\mathcal  M}(\bsbb) $  according to our notation,    we can   express the one-sided directional derivative as
\begin{align}
\delta \range (\bsbg; \bsbb - \bsbg)
 & = \max_{ k\in \overline {\mathcal M}(\bsbg)}     (\beta_k-\gamma_k)-  \min_{  k\in\underline {\mathcal M}(\bsbg)}(\beta_k -\gamma_k), \ \forall\bsbb,\bsbg.
\end{align}
Particularly, when $\bsbb\in \Proj_K$ or $\bsbg\in \Proj_K$, $\delta \range (\bsbg; \bsbb - \bsbg) = - \| \bsbg\|_{\range}$ or $  \| \bsbb\|_{\range}$, respectively, and so the generalized Bregman of the range    (defined later in  Section \ref{sec:nonasymstat}) satisfies $\breg_{\range} ( \bsbb , \bsbg)=0$.
In addition, it is easy to see that   one-sided directional derivatives satisfy a \emph{nonnegative scaling  property}: $\delta \range(\bsbb; c \bsba) = c \delta \range(\bsbb;  \bsba) $ for any  $c\ge 0$. This property will be used to define and justify a \textbf{cone} used in the regularity conditions in Section \ref{sec:nonasymstat}.

\begin{figure}[!ht]
\centering
\includegraphics[width=1\linewidth]{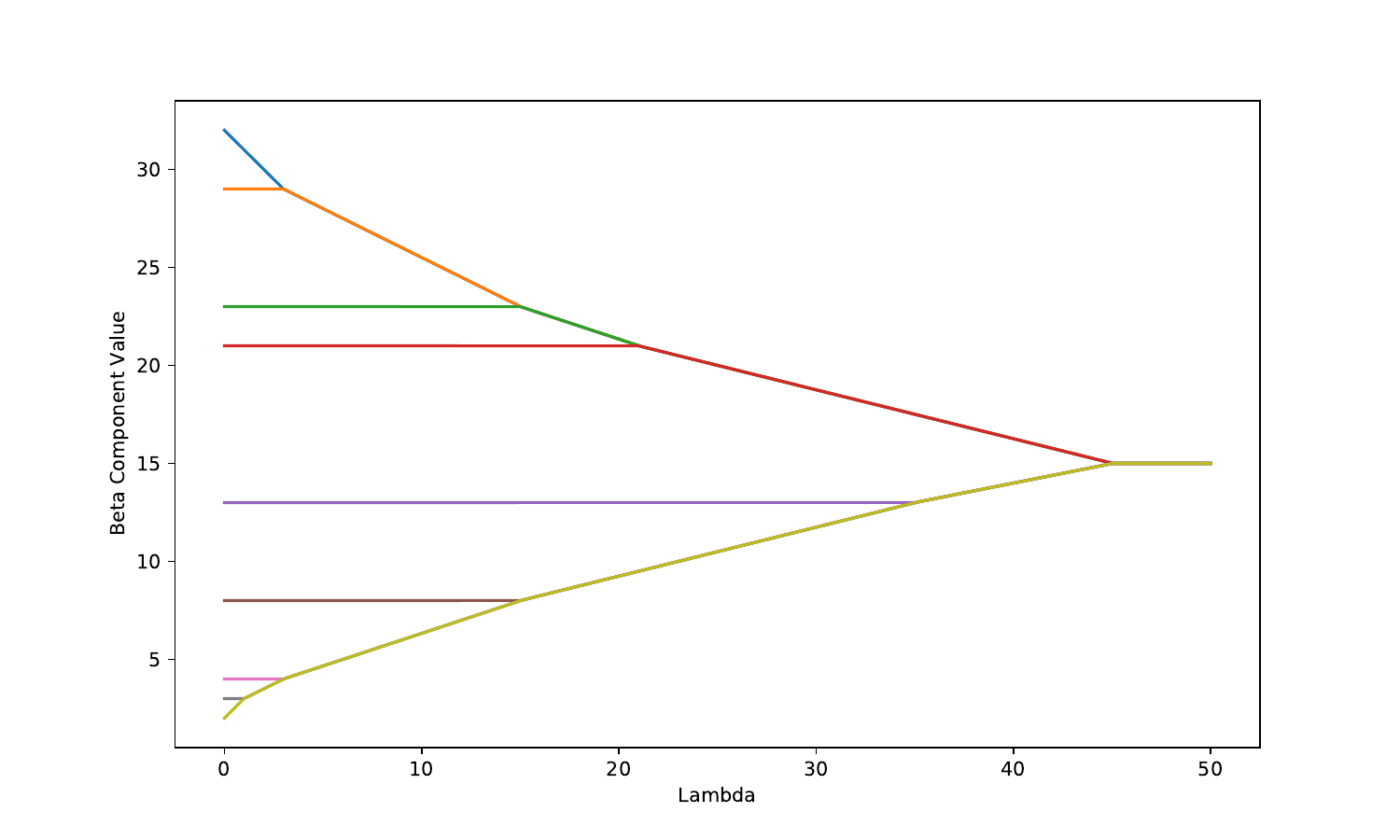}
\vspace{-.3in}\caption{\footnotesize As $\lambda$ varies in the  proximity problem, enforcing the range penalty
can produce   estimates with \emph{distinct} components, estimates that \emph{cluster} at extreme values, and   a \emph{uniform} estimate. The effects of shrinkage on extreme values and polar clustering  are  evident. \label{fig:prox}}
\end{figure}

A notable fact is that we can provide an explicit  computable form  of the proximity operator associated with  $\range$, to be used in Section \ref{sec:comp}:
\begin{align}
\prox_{\mathcal R}(\bsby) := \arg\min\frac{1}{2} \| \bsby - \bsbb\|_2^2 + \lambda \range(\bsbb). \label{penRprox}
\end{align}
Figure \ref{fig:prox} presents a plot of the solution path when varying the values of $\lambda$.  \begin{theorem}\label{thm:prox}
Given $\bsby =[y_k]\in \mathbb R^m$ ($m\ge 2$), $\lambda \ge 0$, let  the optimal solution of   \eqref{penRprox} be  denoted by $\bsbb^o$.
Then,    $\langle \bsb1, \bsby\rangle = \langle \bsb1, \bsbb^o\rangle$. Moreover, with ordered components $y_{(1)}\ge \cdots \ge y_{(m)}$ of $\bsby$, define
\begin{align*}
&\bar c_1 = 0,  \bar c_k = \sum_{1\le j < k} y_{(j)} - y_{(k)}  (2\le k \le m),  \bar c_{m+1}=+\infty,
\\
&\ubar{c}_1= 0, \ubar{c}_k = \sum_{m-k+1<  j \le m} y_{(m-k+1)} - y_{(j)} (2\le k \le m), \ubar c_{m+1}=+\infty,
\end{align*}
and
\begin{align}
\overline \lambda := \frac{1}{2}\sum_{k=1}^m | y_k - \bar y| \ (=\| \bsby\|_{\range ^*} = \frac{1}{2}\|\Proj_K^\perp \bsby\|_1).\label{lambdabardef}
\end{align}
 Suppose $\lambda \in  [\bar{c}_{   k_1}, \bar c_{k_1+1})$ and  $\lambda \in  [\ubar{c}_{   k_2}, \ubar c_{k_2+1})$ for some $k_1=k_1(\lambda), k_2=k_2(\lambda)\in [m]$ and let $ \mathcal I_1$ be the index set associated with $y_{(1)}, \ldots, y_{(k_1)}$,   $ \mathcal I_2$ be the index set associated with $y_{(m)}, \ldots, y_{(m-k_2+1)}$, and $\mathcal I_0 = [m]\setminus (\mathcal I_1 \cup \mathcal I_2)$.  Then the components of $\bsbb^o$ are determined as follows for different cases:

(i) $\lambda< \overline \lambda$. In this case, $k_1 + k_2 \le   m$ (but     $\mathcal I_0$ may be empty). The components of $\bsbb^o$ may  cluster at \emph{extreme values} while the remaining components retain their values as  $y_i$ ``\emph{intact}''  without  additional shrinkage:
\begin{align}
\begin{cases}
\beta_k^o = \frac{1}{|\mathcal I_1|} \sum_{k\in \mathcal I_1} y_k -\frac{1}{|\mathcal I_1|} \lambda\ \, \in \big(y_{(k_1+1)}, y_{(k_1)}\big], \  & k\in \mathcal I_1  \\
   \beta_k^o = y_k, \ & k\in \mathcal I_0 \\
 \beta_k^o =  \frac{1}{|\mathcal I_2|} \sum_{k\in \mathcal I_2} y_k +  \frac{1}{|\mathcal I_2|}\lambda\ \, \in \big[y_{(m-k_2+1)}, y_{(m-k_2)}\big), \ & k\in  \mathcal I_2.
\end{cases}
\label{proxbetasol-gen}
\end{align}
In particular, if $\lambda$ satisfies $\underline \lambda\le \lambda <\overline \lambda $ with \begin{align}
\underline \lambda := \min_k\,  (\bar c_k \vee \ubar{c}_{m-k} ) \label{lambdabar1}
\end{align}
then $k_1+ k_2 = m$ ($\mathcal I_0= \emptyset$), and so the components of $\bsbb^o$ exhibit \emph{two} distinct clusters.

(ii)  $\lambda \ge\overline \lambda$. In this case,   $k_1+ k_2 \ge   m$. The components of $\bsbb^o$ are \emph{uniform}:
\begin{align}
\bsbb^o =\Proj_{K} \bsby =\frac{\sum y_k}{m}\, \bsb1.
\end{align}

\end{theorem}

The range penalization    shrinks the extreme values of $\bsbb^o$ toward 0, thereby reducing the range,  as observed by
\begin{align*}\max \beta_k^o - \min \beta_k^o &=\frac{1}{|\mathcal I_1|} \sum_{k\in \mathcal I_1} y_k -\frac{1}{|\mathcal I_2|} \sum_{k\in \mathcal I_2} y_k -   \lambda(\frac{1}{|\mathcal I_1|} + \frac{1}{|\mathcal I_2|} ) =\avg (\bsby_{\overline{\mathcal M}(\bsbb^o)}) - \avg( \bsby_{\underline{\mathcal M}(\bsbb^o)}) - \lambda \tau^2 (\bsbb^o).
\end{align*}
Here, $\tau$ is  a useful measure that will recur frequently in theoretical analysis (such as inflation control):
\begin{align}
\tau(\cdot) := \left( \frac{1}{\overline M (\cdot)} + \frac{1}{\underline M (\cdot)}\right)^{1/2} \label{taudefvec}
\end{align}
which is bounded by
\begin{align}
\frac{\sqrt 2}{\sqrt m}\le \tau(\bsbb)\le  {\sqrt 2} , \ \forall \bsbb \in \mathbb R^m.
\end{align}
A basic result, established by minimizing the variance of
$\beta_k$ ($1\le k \le m$), is outlined below. (The detailed proof will be provided later when we discuss a more general  Theorem \ref{thm:tau}.)

\begin{lemma}\label{lem:tautoEuclide}
Given any    $\bsbb\in \mathbb R^m$,    $\range(\bsbb)   \le \tau(\bsbb) \| \Proj_K^\perp  \bsbb \|_2$.
\end{lemma}
In conclusion, since $\range$ is a \emph{seminorm}, applying range penalization not only aids in detecting signals within the {kernel space} (reflecting shared parameters among various clients in federated learning) but also captures an additional  parsimony within its {complement space}. Specially, in cases where there are significant clusters at extreme values, $\range(\bsbb)$ can be significantly smaller (compared to the length of $\bsbb$), e.g.,
$$\overline M(\bsbb) \wedge \underline M(\bsbb) \gtrsim m\implies\range(\bsbb) \lesssim \frac{ \| \Proj_K^\perp  \bsbb \|_2}{\sqrt m}.$$    The  relation  contrasts with pairwise-difference regularization, which may discourage size-balanced clustering (cf. \cite{SheetalCRL22} for further discussion); one helpful intuition in this regard  comes from the $\ell_1$-representation
$\range(\bsbb) =
(1/m)\sum_{k=1}^m (|\max_\ell \beta_\ell-\beta_k|  +|\beta_k-\min_\ell \beta_\ell|)
  $.
 The resulting   error rate reduction  will be elaborated in Section \ref{sec:nonasymstat}. Therefore, this single seminorm regularizer effectively achieves dual objectives.

As mentioned in Section \ref{sec:intro}, range shrinkage and polar clustering   serve as a form of regularization in estimating the unknowns.
On the other hand,
 the range penalization does \textbf{not} impact the ``intermediate" components of $\bsby$, as evidenced by \eqref{proxbetasol-gen}.  This absence of undesired shrinkage distinguishes the range proximity operator from many commonly used penalties.  (For example,        the $\ell_1$-norm penalty $\lambda \| \bsbb\|_1$ employs    soft-thresholding $\Theta_{\text{soft}}(\bsby; \lambda)=   [(|y_k|-\lambda)_+   \sgn( y_k)]$, which shrinks \textit{all} components. The same is true for the ridge penalty   $\lambda \| \bsbb\|_2^2/2$  with proximity  $\bsby/(1+\lambda)$ and for the Stein shrinkage $\lambda \| \bsbb\|_2$ with   proximity  $\Theta_{\text{soft}}(\|\bsby\|_{2};\lambda)\bsby/\|\bsby\|_2$.)

The result for \(\lambda \in [\underline{\lambda},   \overline{\lambda})\) demonstrates the possibility of obtaining bileveled coefficients. Such coefficients admit an extremely compact representation, requiring only two floating-point values and one bit for  each \(\beta_k\) (for \(1 \le k \le m\)). In many machine learning tasks, bileveled coefficients also reduce model complexity, leading to faster computations, lower memory usage, and reduced power consumption. This is especially valuable in resource-constrained environments like edge computing and mobile applications.

 We introduce some variants of the range
seminorm to facilitate theoretical analysis.
\begin{theorem}\label{thm:dualnorm}
Consider $\|\cdot \|_{\range}$ on $\mathbb R^m$. For any $ \bsba\in \mathbb R^m$,
\begin{align}
\sup_{\| \bsbb\|_{\range} \le1} \langle \bsbb, \bsba\rangle=\frac{1}{2}\| \bsba\|_1 + \iota_{\langle 1, \bsba\rangle = 0}=\begin{cases} \frac{1}{2} \| \bsba\|_1, & \mbox{ if } \langle 1, \bsba\rangle = 0\\ +\infty, & \mbox{ if }\langle 1, \bsba\rangle \ne 0, \end{cases}
\end{align}
from which it  follows that  the range's \emph{``\textbf{dual seminorm}''} is given by
\begin{align}
\| \bsba \|_{\range^*}: =\sup_{\| \bsbb\|_{\range} \le1, \bsbb\in \Proj_K^\perp} \langle \bsbb, \bsba\rangle=\frac{1}{2} \| \Proj_K^\perp \bsba\|_1
\label{dualseminormvec}
\end{align}
\eqref{dualseminormvec} is still a seminorm on $\mathbb R^m$ with the same kernel.

Moreover, the Fenchel conjugate of $\lambda \| \cdot\|_{\range}$:     $
\range^* _{\lambda}(   \bsba) := \sup_{  \bsbb } \langle \bsbb, \bsba\rangle
 -\lambda \| \bsbb\|_{\range}$ with $\lambda\ge 0$ is given by
\begin{align}
\range^* _{\lambda}(   \bsba) & = \iota_{\| \bsba\|_{\range^*} \le  \lambda,  \bsba\in \Proj_{K}^\perp}=  \iota_{\| \bsba\|_1 \le 2\lambda, \, \langle 1, \bsba\rangle = 0}=\begin{cases} 0, & \mbox{ if } \langle 1, \bsba\rangle = 0, \| \bsba\|_1 \le 2\lambda\\ +\infty, & \mbox{ otherwise} . \end{cases}
\end{align}
\end{theorem}
Interestingly,  $(1/2) \| \Proj_K^\perp \cdot\|_1$ and   $2 \| \Proj_K^\perp \cdot \|_\infty$ are \textbf{not} dual seminorms. Instead,   $\range(\cdot)$ and   $(1/2) \| \Proj_K^\perp \cdot\|_1$
are dual to each other on the kernel complement space $ \Proj_K^\perp $.

 To conclude this section, we can extend the range concepts to block vectors and matrices.
Given a block vector $\bsbb=[\bsbb_1^T, \ldots, \bsbb_p^T]^T$ where the subvectors $\bsbb_j$ may have different sizes, define its block(group)-range seminorm by
\begin{align}
\range^{(b)}(\bsbb): = &  \sum_{j=1}^p \|\bsbb_j\|_{ \range},
\end{align}
where the superscript $(b)$ indicates the block version of the range seminorm. The kernel of $\range^{(b)}(\cdot)$ is the Cartesian product of the kernels of $\range(\cdot)$  on each individual block.
Henceforth, we will use $\Proj_K^{(b)}$ to denote the orthogonal projection onto the space,  which is a block diagonal matrix, $\diag\{\Proj_K, \ldots, \Proj_K \}$, consisting of a series of $\Proj_K$'s. Each $\Proj_K$ within this matrix is a kernel projection that may vary in size.   Hence     $\Proj_K^{(b)}\bsbb$ gives the block averaged $\bsbb$: $[(\Proj_K \bsbb_1)^T, \ldots, (\Proj_K \bsbb_p)^T]^T$.
We   denote its orthogonal complement as $\Proj_K^{(b),\perp} = \bsbI - \Proj_K^{(b)}$.
To analyze the group-range regularized estimator, define a block version of  $\tau$ for the block vector $\bsbb$:
\begin{align}
\tau^{(b)}(\bsbb) := \big(\sum_{j=1}^p\tau^2(\bsbb_j)\big)^{1/2}. \label{taublock}
\end{align}
\begin{remark}[\textbf{Penalty vs. Constraint}]
This paper focuses on a range penalty, such as   $\lambda \sum_{j=1}^p \|\bsbb_j\|_{\range}$, rather than a constraint formulation of the form $\sum \| \bsbb_j \|_{\range} \le c$. Our theoretical analysis will show that the penalty parameter   $\lambda$ attains a universal rate. From a computational perspective, although the Moreau decomposition and the dual seminorm can be used to derive the proximity operator associated with the range constraint, our experience suggests that the resulting implementation is less efficient in practice. This advantage is particularly relevant in our federated learning setting, where \emph{group} range regularization is required.

 \end{remark}
\begin{remark}[\textbf{Polar Shrinkage and Practical Relevance}]
Block range regularization induces a form of shrinkage that differs fundamentally from more conventional center-oriented penalties.
For example, the centered group lasso penalty in \eqref{cglpen} of Appendix \ref{sec:exps}, proposed for shared-parameter identification, assumes that each row is homogeneous or nearly so, and therefore pulls all its entries toward a common center through block soft-thresholding. By contrast, the block range penalty proposed here acts through the \emph{boundary values}, with the extremes serving as the effective shrinkage targets.

Polar clustering proposed here is not merely a mathematical artifact, but a meaningful structural regime that can arise in practice.
 Indeed, settings in which two dominant extremes each have substantial support appear  in a variety of applications, including precision medicine \citep{moody2019computational,wang2009bimodality}, agriculture \citep{tsai2019gmr}, and finance \citep{hwang2021lgd}.
Our air-quality analysis in Appendix \ref{sec:exps} also reveals polar extrema across clients. Such patterns are realistic in environmental and atmospheric applications, where the same covariate may induce contrasting low/high regimes under broad geographic  conditions \citep{yang2017pm25wind,sun2022verticalwindshear}.

\end{remark}
\section{Range-penalized Federated Learning}
\label{sec:flsetting}
To introduce the federated learning setting, assume the   data  are distributed across  $m$ clients,   denoted by $(\bsbX_1, \bsby_1)\in \mathbb R^{n_1\times p}\times \mathbb R^{n_1},\ldots, (\bsbX_m, \bsby_m)\in \mathbb R^{n_m\times p}\times \mathbb R^{n_m} $ with $N= \sum n_k$ as the total sample size. We treat   all   $\bsbX_1,\ldots, \bsbX_m $ as deterministic.

Define the systematic components as  linear functions of predictors with properly scaled coefficients:
\begin{align}
\bsbe_k := \bsbX_k \bsb{b}_k/\rho_k\label{etadef}
\end{align}
where $\rho_k$ typically takes $\| \bsbX_k\|_2$ or $\sqrt{n_k}$. We incorporate the scaling factors $\rho_k$ to account for heterogeneous sample sizes, so that the rescaled coefficients $\bsb{b}_1, \ldots, \bsb{b}_m$  become comparable across clients.     Let  $\bsbB=[\bsb{b}_1, \ldots, \bsb{b}_m]=[\bsbb_1, \ldots, \bsbb_p]^T\in \mathbb R^{p\times m}$ represent the overall (scaled) coefficient matrix and introduce notations
\begin{align}
\bsb{b} = \vect(\bsbB),\quad \bsbb = \vect(\bsbB^T). \label{betabB}
\end{align}
According to \eqref{etadef}, when  $\rho_k = \mathcal O(\sqrt {n_k} ) = \mathcal O( \sqrt n)$, the root-$n$ consistency of the coefficient estimates  implies that the order of $\hat {\bsb{b}}$ or $\hat \bsbb$ is independent of  $n$, which   simplifies notations and  derivations.  Define
 a ``\textbf{group range}''   function on matrix $\bsbB$
 \begin{align}
\|\bsbB\|_{\range, 1}  := \sum_{j=1}^p \range (\bsbb_j),
\end{align}
which is a seminorm as well. Treating $\bsbb=\vect(\bsbB^T)$ as a block vector, we may also write it as $\range^{(b)}(\bsbb)$. 

Define some useful  matrices related to the designs:
\begin{align}
&\bsbX_k^\circ := \bsbX_k/\rho_k,\
\bar \bsbX := \diag\{\bsbX_k^\circ \} = \begin{bmatrix}\bsbX_1^\circ & & \\ & \ddots & \\ & & \bsbX_m^\circ\end{bmatrix},\label{designs1}\\
& \tilde \bsbX  = [\tilde \bsbX_1, \ldots, \tilde \bsbX_p]: =\bar \bsbX \bsbK_{m,p} \in \mathbb R^{N\times pm},\label{designs2}
\end{align}
where $\bsbK_{m,p}\in \mathbb R^{pm \times pm}$ is the  \textit{commutation matrix}. Notice that each $\tilde \bsbX_j\in \mathbb R^{N\times m}$ has a block diagonal structure. It follows that
  \begin{align}
\bsbe&=[\bsbe_1^T, \ldots, \bsbe_m^T]^T=\bar \bsbX \bsb{b}=  \bar \bsbX \vect(\bsbB)   =\bar \bsbX \bsbK_{m,p} \bsbK_{p,m}\vect(\bsbB)=  \tilde \bsbX\vect(\bsbB^{T})=\tilde \bsbX    \bsbb. \label{sysexpress}
\end{align}
Consider a group range penalized federated learning problem
\begin{align}
\min \sum_{k=1}^m l  (\bsb{b}_k; \bsby_k, \bsbX_k) +\lambda \sum_{j=1}^p \range (\bsbb_j)=\sum_{k=1}^m l_0 (\bsbe_k; \bsby_k) +\lambda \sum_{j=1}^p \range (\bsbb_j).
\label{rangeFL}
\end{align}
Given \eqref{sysexpress} with $\bsbe_k$   the $k$th block row of  $\tilde \bsbX \bsbb$,
we  also express the loss as
\begin{align}
\sum_{k=1}^m l_0 (\bsbe_k)= \tilde l_0(\bsbe) = L(\bsbb).
\end{align}
As previously discussed (see Lemma  \ref{lem:tautoEuclide} and Section \ref{sec:intro}),  (group) range reduction captures two forms of parsimony and also offers practical benefits for quantization, resource efficiency, and privacy protection in federated learning.


Next, we introduce the concept of \emph{effective noises} into the model.
 Let $\bsbb^*$ (and the associated   $\bsbe^*, \bsb{b}^*, \bsbB^* $) denote the statistical truth and define the effective noises:
\begin{align}
\bsbeps_k: = & -\nabla l_0(\bsbe_k^*),\quad \bsbeps: =   [\bsbeps_1^T, \ldots, \bsbeps_m^T]^T,
\end{align}
Because $\bsb{\epsilon}_k$ and $\bsb{\epsilon}$ are determined solely by the systematic component, the $\rho_k$-rescaling leaves the effective noise distribution unchanged.
We always assume that the statistical truth corresponds to a \textit{regular} (differentiable) point of the loss function. By contrast, it often appears as an \textit{irregular} point with respect to the  penalty, which enables the structural parsimony of the model to be exploited in high-dimensional estimation. We typically assume the  vectorized  effective noise $\bsbeps$ has finite Orlicz $\psi_2$-norm  $\sigma$, in the sense that   $\|\langle \bsba, \bsbeps \rangle\|_{\psi_2} \lesssim   \|\bsba\|_2 \sigma$ for all $\bsba$, where  $\| z\|_{\psi_2} := \inf \{t>0: \EE [\exp(z^2/t^2)]\le 2\}$. (In this paper, a sub-Gaussian random vector is further assumed to have mean zero.) Note that finiteness of the $\psi_2$-norm  does \emph{not} require the components of $\bsbeps$ to be independent, and it is automatically satisfied when the loss function is Lipschitz.  It is possible to obtain results for distributions with much heavier tails (such as subWeibull);  see the discussion following Theorem \ref{thm:seesawerr}.

To facilitate our analysis, some theorems will assume  that   $\tilde l_0$ has a  Lipschitz continuous gradient at the true $\bsbe^*$:
\begin{align}
\|\nabla \tilde l_0(\bsbe) - \nabla \tilde l_0(\bsbe^*)\|_2 \le  \varrho\|\bsbe  -  \bsbe^*\|_2, \label{gradLip}
\end{align}
for some $\varrho>0$.
This type of    conditions is frequently assumed in the optimization and theoretical analysis of a general loss function.
As an important application,  consider a regression federated learning problem with   $l_0$ being quadratic: $l_0(\bsbe_k; \bsby_k) = \|\bsbe_k - \bsby_k\|_2^2/2$. Then,  $\bsbeps_k =-\nabla l_0(\bsbe_k^*)= \bsby_k - \bsbe_k^*$, and $\varrho=1$.
\section{Main Results of Nonasymptotic Statistical Analysis}
\label{sec:nonasymstat}
The practical benefits of range penalization demand rigorous statistical analysis, which is currently lacking  in the existing literature. The primary goal of the section is to  derive a sharp  error rate of the estimator  defined by \eqref{rangeFL}. Rather than showing a bound dependent on the magnitude of  $\bsbb^*$, which is relatively simple using the subadditivity of the seminorm, researchers often seek error rates expressed in terms of the structural properties  of $\bsbb^*$. However, traditional approaches result in suboptimal error rates and   restrictive regularity conditions. To address the challenges, we will develop a new proof device, as detailed in Section \ref{subsec:errrate}.  Another topic of interest is the recovery of  groupwise and within-group patterns in the true model. This  includes the detection of  polar clusters and uniform rows in $\bsbB^*$ for identifying shared parameters in federated learning; see Section \ref{subsec:pattrec}.

We begin by  defining some essential notations and symbols to facilitate our discussion.

\paragraph*{Generalized Bregman function}
To facilitate our analysis, we introduce   the generalized Bregman function for a directionally differentiable $f$   which is always assumed to be defined on the entire Euclidean space in this paper,
\begin{align}
\breg_f (\bsbb, \bsbg):= f(\bsbb) - f(\bsbg) - \delta f(\bsbg; \bsbb - \bsbg).
\end{align}
If $f$ is a convex function (not necessarily differentiable), then $\breg_f\ge 0$.   When $f$ is both differentiable and strictly convex, $\breg_f$ becomes the standard Bregman divergence $\Breg_f$. We   use $\breg_f \ge \breg_g$ to denote $\breg_f(\bsbb,\bsbg) \ge \breg_g(\bsbb,\bsbg)$ for all $ \bsbb,\bsbg  $. For more properties of the generalized Bregman function, see \cite{Shebregman2021}.

The Bregman associated with half of the squared {Euclidean} norm  is denoted by $\Breg_2$:     $\Breg_2(\bsbb, \bsbg)= \Breg_{\|\cdot \|_2^2/2}(\bsbb, \bsbg)=\|  \bsbb - \bsbg \|_2^2/2$.
 In general, $\breg_f$ is not symmetric and we define its symmetrization as $\bregs_f(\bsbb, \bsbg) = (\breg_f(\bsbb, \bsbg) + \breg_f(\bsbg, \bsbb))/2$. For the range seminorm, which is not differentiable everywhere, we denote    $\breg_{\|\cdot \|_{\range}}(\bsbb, \bsbg)  = \| \bsbb\|_{\range} - \| \bsbg\|_{\range} - \delta\| \cdot \|_{\range} (\bsbg; \bsbb - \bsbg) $      by $\breg_{\range}(\bsbb, \bsbg) $ for short. Based on Theorem \ref{thm:subgrad}, this is equal to 0 when   $ \bsbb$ or $\bsbg \in \Proj_K$.

When $\breg_f(\bsbb, \bsbg) \ge \mu \Breg_2(\bsbb, \bsbg) $ with $\mu>0$, $f$ is called  $\mu$-strongly convex. Theoretical analysis will reveal that when pursuing structural parsimony via block range, our problem may  exhibit restricted strong convexity under suitable conditions,  even when $p>n$.   This property is also useful for accelerating convergence in optimization in Section \ref{sec:comp}.
We will use the generalized Bregman functions for measuring errors and defining regularity conditions.

\paragraph*{Extreme value patterns}

Given   $\bsbzeta=[\zeta_k]\in \mathbb R^m$, let $q(\bsbzeta)=M(\bsbzeta) + 1 + 1_{\max \zeta_k > \min\zeta_k}$.   Based on  the polar clustering   pattern of $\bsbzeta$, we can define a binary membership matrix $\mathcal Q(\bsbzeta)\in \mathbb R^{m\times q(\bsbzeta)}$  to characterize its  clustering structure. The associated linear subspace  induces an orthogonal projection matrix $\Proj_{\mathcal Q(\bsbzeta)} =\mathcal Q(\bsbzeta) ((\mathcal Q(\bsbzeta))^T\mathcal Q(\bsbzeta))^{-1}(\mathcal Q(\bsbzeta))^T $  such that
$
\Proj_{\mathcal Q(\bsbzeta)}\bsbzeta  = \bsbzeta$ and   $\rank(\Proj_{\mathcal Q(\bsbzeta)}) = q(\bsbzeta)$.
For example, when $\bsbzeta= c \bsb1$, $\mathcal Q(\bsbzeta)= \bsb1$ and $\Proj_{\mathcal Q(\bsbzeta)}=\Proj_{\bsb1_m}$;   when $\zeta_1=\cdots =\zeta_{m_1} = \max\zeta_k$ and $ \zeta_{m_1+1}=\cdots= \zeta_{m_1+m_2} = \min \zeta_k $, $\mathcal Q(\bsbzeta) = \diag\{\bsb1, \bsb1, \bsbI\}$, leading to   $\Proj_{\mathcal Q(\bsbzeta)} = \diag\{\Proj_{\bsb1_{m_1}},\Proj_{\bsb1_{m_2}},\allowbreak \bsbI_{(m-m_1-m_2)\times (m-m_1-m_2)} \}$.  Here, the ordering of the columns in $\mathcal Q(\bsbzeta)$ is inconsequential, as  $\Proj_{\mathcal Q(\bsbzeta)}$ is  {permutation invariant}:
 $\Proj_{\mathcal Q(\bsbzeta) \bsbP}=\Proj_{\mathcal Q(\bsbzeta) }$  for any permutation matrix $\bsbP\in \mathbb R^{q(\bsbzeta)\times q(\bsbzeta)}$.
In addition, define the  \emph{standardized} $\mathcal Q( \bsbzeta)$ as
\begin{align}
  \bsbU(\bsbzeta)=\mathcal Q( \bsbzeta) ((\mathcal Q (  \bsbzeta))^T \mathcal Q(  \bsbzeta)  )^{-1/2} \ (= \mathcal Q( \bsbzeta) (\diag\{\bsb1^T \mathcal Q(  \bsbzeta) \})^{-1/2}). \end{align}
Clearly,   $\bsbU(\bsbzeta)$ is orthogonal and  $\Proj_{\mathcal Q(\bsbzeta)} =\Proj_{\bsbU(\bsbzeta)}=\bsbU(\bsbzeta)(\bsbU(\bsbzeta))^T $.

For a block vector $\bsbb=[\bsbb_1^T, \ldots, \bsbb_p^T]^T$,  we can    similarly define  $\mathcal Q^{(b)}(\bsbb) =\diag\{\mathcal Q(\bsbb_j)\}$,  $\bsbU(\bsbb) = \diag\{\bsbU (\bsbb_j)\}$, and  $\Proj_{\mathcal Q^{(b)}(\bsbb)} = \diag\{\Proj_{\mathcal Q(\bsbb_j) }  \}=   \bsbU(\bsbb) \allowbreak(\bsbU(\bsbb) )^T$.
As an illustration, when the components of $ \bsbb_j$ are ordered consecutively as  $ {\overline{\mathcal M}}_j,   {\mathcal M}_j,  {\underline{\mathcal M}}_j $ with cardinalities     $  {\overline{ M}}_j,  { M}_j,  {\underline{ M}}_j $ for $1\le j \le p$,     $
\bsbU(\bsbb_j) = \diag\{ \frac{1}{\sqrt{ {\overline M_j}}} \bsb1,    \bsbI_{  M_j}, \frac{1}{\sqrt{ {\underline M}_j}}\bsb1 \}.
$ Note that by definition,  for each $j\in   ({\mathcal J}(\bsbb))^c$ or $\bsbb_j\in \Proj_K$, $ {\mathcal M}_j=\emptyset$ and
$\bsbU(\bsbb_j) = \bsb1 /\sqrt m $.

For a matrix $\bsbB\in \mathbb R^{p\times m}$, we define several associated measures using the block vector   $\bsbb=\vect(\bsbB)$.
Define $\mathcal J(\cdot)\subset[p]$ as the index set of nonuniform rows by
\begin{align}
\mathcal J(\bsbB) =\mathcal J(\bsbb)   = \{j\in [p] : \overunderline M(\bsbb_j) >\overline M(\bsbb_j)   \}.
\end{align}    Abbreviate $\mathcal J(\bsbB^*)$ to $\mathcal J^*$,   denote the complement of $\mathcal J^*$  by $\mathcal J^{*c}$, and let $J^* = |\mathcal J^*|$.
Similarly, $\hat {\mathcal J} = \mathcal J(\hat \bsbB), \hat J = |\hat {\mathcal J} |$.  For the $j$th row of $\bsbB^*$, or $\bsbb_j^*$, define $\mathcal M_j^*:= \mathcal M( \bsbb_j^*)\subset [m]$,    $M_j^*= | \mathcal M_j^*|$,    ${\overunderline {\mathcal M}}_j ^* :={\overunderline {\mathcal M}}(\bsbb_j ^*)$ and    ${\overunderline  M}_j ^*=|{\overunderline {\mathcal M}}_j ^* |$. Let
\begin{align}
M^* = \sum_{j\in \mathcal J^*} M_j^*.
\end{align}
Finally, we introduce a short notation
\begin{align}
\mathcal P^*= \Proj_{\tilde \bsbX {\mathcal  Q^{(b)}(  \bsbb^*)}},
\end{align}
and $\Proj^{*,\perp}= \bsbI - \Proj^*$. Then $\mathcal P^*$, the linear subspace of the true model, can be equivalently expressed as
\begin{align}
\mathcal P^*=\text{span}&\big( \{\tilde \bsbX_j[:, k]: j\in \mathcal J^*, k\in \mathcal M_j^*\}  \cup \{\tilde \bsbX_j [:, {\overline {\mathcal M}}_j ^* ]\bsb1: j\in [p] \}\cup \{\tilde \bsbX_j[:, {\underline {\mathcal M}}_j ^* ] \bsb1: j\in [p]\} \big). \label{trumodelspace}\end{align}
Direct calculation shows
 $$\rank (\mathcal P^*)\le\sum_j M_j^* + p - J^{*}+2J^* \le    M^*+2 p \lesssim M^*  + p \ (\lesssim J^* m + p).$$

\subsection{Statistical Error}
\label{subsec:errrate}

The nonasymptotic analysis of \eqref{rangeFL} poses significant challenges. Particularly,
$\tau(\cdot)$,  unlike the support size in variable selection problems, is \textit{not} subadditive.  Traditional methods   relying on H\"older's inequality or dual norms for addressing the stochastic term necessitate a larger  choice  of regularization parameters and  impose more stringent regularity conditions for the (group) range regularizer.  For example, the resulting error rate, $$p + m^2 J^*\log m+m^2 J^*   \log p,$$       is much worse than the rate of \eqref{seesawrateforcomp} under the same subGaussian assumption. For more discussions of these limitations, readers are encouraged to consult Theorem \ref{th:dualmethoderrrate} and Corollary \ref{cor:dualmethod} (which are presented in Appendix \ref{subsec:dualnormanal} due to limited space).
 To overcome the challenges, we introduce a versatile framework for deriving sharp error rates and optimal regularization parameters. The key trick is to integrate statistical and optimization analysis  to bound stochastic terms in a `seesaw' manner.

Let
\begin{align}
\rho_0 :=  (\varrho \vee 1) \| \tilde \bsbX\|_2= (\varrho \vee 1)\max_{1\le k \le m} \|\bsbX_k^\circ\|_2    , \label{rho0choice}
\end{align}
where  $\varrho$ is defined in \eqref{gradLip} and $\bsbX_k^\circ$ is the
design scaled by     $\rho_k$   to    balance heterogeneous data sizes in federated learning. Simple calculation shows that for regression, $\varrho = 1$, and for logistic regression, $\varrho = 1/4$.
If we set $\rho_k = \ \| \bsbX_k\|_2$ in  implementation,
$$
\rho_0  =\varrho\vee 1.
$$

The measures of  $\tau$ and its block version $\tau^{(b)}$ (cf.   \eqref{taudefvec}, \eqref{taublock})   are crucial for controlling model complexity. Indeed, in the non-asymptotic analysis of the group-range regularizer, we find the sum of $\tau$ squared across $\mathcal J(\bsbb)$,  $$   \sum_{j\in \mathcal J(\bsbb)} (\tau^2(\bsbb_j))=\sum_{j\in \mathcal J(\bsbb)} (\frac{1}{\overline M(\bsbb_j)}+\frac{1}{\underline M(\bsbb_j)}),$$ serves as an analogue to the `support size' in variable selection.
For short, let $\tau_j^* = \tau(\bsbb_j^*)  , 1\le j \le p$.

To build intuition, we first present a result for losses that are strongly convex in  $\bsbe$ (rather than in $\bsbb$), a setting that covers important applications such as regression. A more general result, without requiring global    convexity of $\tilde l_0(\bsbe)$, is given  in Theorem \ref{thm:seesawerr}.
\begin{theorem}
\label{thm:seesawerr-scvx}
 Assume the loss $\tilde l_0$ is differentiable and satisfies the Lipschitz-gradient condition   \eqref{gradLip} concerning    the statistical truth $\bsbe^*$,  and   $\tilde l_0$ is   $\nu$-strongly convex in  $\bsbe$ (as is the case in regression,  where $\nu=1$). Suppose the   effective noise  $\bsbeps$ is centered and has a $\psi_2$-norm  bounded by $\sigma$, with its components   not necessarily being independent. Let $\lambda = C\sigma(\rho_0  /(1\wedge\nu))(m + \sqrt{m \log p})    $, where   $C$ is a sufficiently large constant.
Under the following   regularity condition,
\begin{align}
  &\frac{\kappa^2}{\sum_{j\in \mathcal J^*} (\tau_j^*)^2}  \big( \sum_{j\in \mathcal J^*}       \max_{
\underline {\mathcal M} (\bsbb_j^* + \bsba_j)} \bsba_j[k]- \min_{  \overline {\mathcal M} (\bsbb_j^* + \bsba_j)}     \bsba_j[k]  \big)^2     \le
   \bregs_{\tilde l_0} (\bsbe^* + \tilde \bsbX \bsba,    \bsbe^*)\label{seesawcompcondbound0}
\end{align}
for   $\bsba$ restricted to
\begin{align}
\mathcal C_0:=\big\{\bsba:\sum_{j\in \mathcal J^*}   \max_{
\underline {\mathcal M} (\bsbb_j^* + \bsba_j)} \bsba_j[k]- \min_{  \overline {\mathcal M} (\bsbb_j^* + \bsba_j)}     \bsba_j[k]\ge   \sum_{j\in \mathcal J^{*c}}  \range (\bsba_j  )\big\} , \label{seesawrestrictregion-general0}
\end{align}
the estimator obtained  by minimizing \eqref{rangeFL} satisfies the following bounds in prediction   and `support size':  \begin{align}
&  \EE\big[\bregs_{\tilde l_0}   (\tilde \bsbX \hat \bsbb, \tilde \bsbX   \bsbb^*)\big]    \lesssim        \frac{\rho_0^2\sigma^2}{\kappa^2(\nu^2\wedge 1)}  \sum_{j\in \mathcal J^*} (\tau_j^*)^2 ( m^{2}+m\log p)+(\nu\vee \frac{1}{\nu})\sigma^2 (M^* + p) , \label{seesawthbound:pred}\\
 &        \EE\big[ \sum_{j \in \hat {\mathcal J} } (\tau(\hat \bsbb_j))^2  \big]\lesssim    \frac{\rho_0^2 }{\kappa^2} (\nu\vee \frac{1}{\nu})   \sum_{j\in \mathcal J^*} (\tau_j^*)^2  +(1\vee \frac{1}{\nu^2}) \, \frac{M^{*}+p}{m^2 + m\log p
} .
\end{align}
 \end{theorem}
\begin{remark}[\textbf{Error Rate}] \label{rem:errrate} Our proof yields a weaker regularity requirement, as stated in \eqref{seesawregcond-scvx}.  But to connect more directly with the existing high-dimensional statistics literature and to better gauge the size of $\kappa$,  it is helpful to discuss the error rate under more demanding regularity conditions than those required by the theorem. We begin with the following result, which in particular implies Lemma \ref{lem:tautoEuclide}.

 \begin{theorem}\label{thm:tau} Given an arbitrary   $\bsbzeta^0\in \mathbb R^m$,  for all $ \bsbzeta\in \mathbb R^m$,
\begin{align}
&\range(\bsbzeta^0) - \range(  \bsbzeta) \le -  \delta\range (\bsbzeta^0;   \bsbzeta - \bsbzeta^0)  \le \tau(\bsbzeta^0)\| \Proj_K^\perp (  \bsbzeta - \bsbzeta^0)\|_2 .
\end{align}
\end{theorem}

Using   Theorem  \ref{thm:tau}, it can be shown that \eqref{seesawcompcondbound0} is implied by the following `compatibility'-type bound (under \eqref{seesawrestrictregion-general0}):

\begin{align}
  &\kappa^2 \frac{\big( \sum_{j\in \mathcal J^*}  \tau_j^* \| \Proj_K ^\perp\bsba_j\|_2 \big)^2}{ \sum_{j\in \mathcal J^*} (\tau_j^*)^2}   \le        \bregs_{\tilde l_0} (\bsbe^* + \tilde \bsbX \bsba,    \bsbe^*),\label{seesawcompcondbound1}
\end{align}
or alternatively, a `restricted-eigenvalue' (RE) type bound after applying the Cauchy-Schwarz inequality:
\begin{align}
  &\kappa^2       \| \Proj_K^{(b),\perp} \bsba_{ \mathcal J^*}\|_2^2 \le        \bregs_{\tilde l_0} (\bsbe^* + \tilde \bsbX \bsba,    \bsbe^*),\label{seesawcompcondbound2}
\end{align}
where  in both cases, $\bsba$ is restricted to \eqref{seesawrestrictregion-general0}. The two types of regularity conditions are widely used in high-dimensional statistics \citep{van2009conditions}. The key intuition is that, although $L(\bsbb)$  may fail to be globally strongly convex in $\bsbb$, the regularizer confines the error vector to a small geometric region on which the loss retains sufficient positive curvature. For example, for regression with $l_0(\bsbe_k) = \| \bsbe_k - \bsby_k\|_2^2$, the condition becomes
$$
\kappa^2       \| \Proj_K^{(b),\perp} \bsba_{ \mathcal J^*}\|_2^2 \le
    \| \tilde \bsbX \bsba\|_2^2, \quad \forall\bsba\in \mathcal C_0.$$

Therefore, $\kappa$ (or $K$ in  Theorem  \ref{thm:seesawerr}) can be treated as a constant in regular problems. Ignoring trivial factors, the prediction error bound in \eqref{seesawthbound:pred}  is  of the order
\begin{align}
 \sum_{j\in \mathcal J^*}  (\frac{1}{\overline M_j^*}+\frac{1}{\underline M_j^*})(m^2+m\log p)+ \sum_{j\in \mathcal J^*} M_j^* + p.\label{seeaswerrrate}
\end{align}
When substantial clusters exist at extreme values, $$\overline M_j^*\wedge \underline M_j^*\ge c  m, \forall j $$  \eqref{seeaswerrrate} is bounded by \begin{align}
 J^* \log p +J^*m + p \label{seesawrateforcomp}
\end{align} up to a multiplicative constant (as compared to    the minimax-optimal rate mentioned in Section \ref{sec:intro}). Considering the \emph{average} prediction error per client,  in a balanced setting with    $N = nm$,
   the rate becomes of order   $(J^*m + p + J^*\log p ) /N$, up to constants and trivial factors.  The average error remains well controlled provided   $J^ *$ is small relative to the per-client sample size $N/m$ and $p\ll nm$.   In the regime of mild heterogeneity ($J^* \approx 0$), the prediction error reduces to order $\mathcal{O}(p)$, matching the rate of shared-parameter recovery methods such as the Centered Group Lasso (Appendix B.1), while avoiding the need to optimize an explicit unknown center parameter.

Although   this restricted eigenvalue condition might be more intuitive, the derivation from \eqref{seesawcompcondbound0} to \eqref{seesawcompcondbound2} also shows that the regularity condition  in our theorem is significantly less demanding.\\
\end{remark}

\begin{remark}[\textbf{Cone Restriction}]
 We show how the range-based restriction region is weaker than those  familiar constraint cones  arising  in $\ell_1$-type analysis.
 Based on \eqref{dddeltarel} in the proof,  $\mathcal C_0$ can be equivalently expressed as  \begin{align*}
\sum_{j\in \mathcal J^*}  \big( \min_{k\in \underline {\mathcal M}_j^*} \bsba_j[k] -\max_{k\in \overline {\mathcal M}_j^*}      \bsba_j[  k]    \big)\ge      \sum_{j\in \mathcal J^*}  2\bregs_{\range}(\bsbb_j^* + \bsba_j, \bsbb_j^*)+  \sum_{j\in \mathcal J^{*c}}  \range (\bsba_j  ),
\end{align*}
which is not, in its present form, a cone. Note that the index sets involved on the left-hand side, $ \underline {\mathcal M}_j^*$ for $\min$ and $ \overline {\mathcal M}_j^*$   for $\max$,
are quite different from those in \eqref{seesawrestrictregion-general0}.

The construction of   $\mathcal C_0$   leverages both the \textit{between-group} structure and   the \textit{within-group} structure of extreme values. Notably, the resulting restriction region for group-range is     much {narrower}   compared to one based solely on between-group structure:
 \begin{align*}
  &  \quad
\sum_{j\in \mathcal J^*}  \big( \min \bsba_j[\underline {\mathcal M}_j^*] -\max      \bsba_j[  \overline {\mathcal M}_j^*]    \big)    -  \sum_{j\in \mathcal J^{*c}}  \range (\bsba_j  )-2    \sum_{j=1}^p \bregs_{\range}(\bsbb_j^* + \bsba_j, \bsbb_j^*) \\
& \le \sum_{j\in \mathcal J^*}  \big( \min \bsba_j[\underline {\mathcal M}_j^*] -\max      \bsba_j[  \overline {\mathcal M}_j^*]    \big)    -  \sum_{j\in \mathcal J^{*c}}  \range (\bsba_j  )  \\
& \le   \sum_{j\in \mathcal J^*}  \{\max_{}     \bsba_j[  \overline {\mathcal M}_j^*] -  \min_{ } \bsba_j[ \underline {\mathcal M}_j^*]\}-  \sum_{j\in \mathcal J^{*c}}  \range (\bsba_j  )   \\
&  \le   \sum_{j\in \mathcal J^*}  \{\max_{  }     \bsba_j  -  \min_{ } \bsba_j\}-  \sum_{j\in \mathcal J^{*c}}  \range (\bsba_j  ) =
  \sum_{j\in \mathcal J^* }  \|   \bsba_j   \|_{\range}   - \sum_{j\in \mathcal J^{*c} }  \|     \bsba_j  \|_{\range}.
 \end{align*}
The first inequality is due to the convexity of $\range(\cdot)$, and the second and third inequalities are straightforward (and conservative).

In high-dimensional analysis, \emph{cone} restrictions are often employed to formulate restricted eigenvalue or restricted strong convexity conditions that hold under suitable regularity of the design \citep{raskutti2010restricted,negahban2012unified}. By omitting the (symmetrized) generalized Bregman term, we can define  \begin{align*}
\mathcal C_{1}= \big\{ \bsba: \sum_{j\in \mathcal J^*}  \big( \min \bsba_j[\underline {\mathcal M}_j^*] -\max      \bsba_j[  \overline {\mathcal M}_j^*]    \big)    \ge   \sum_{j\in \mathcal J^{*c}}  \range (\bsba_j  ) \big\},
\end{align*}
which clearly forms a {cone} ($\bsba\in \mathcal C_{1}\Rightarrow k \bsba\in \mathcal C_{1},\forall k\ge 0$).  Based on the previous derivation,
 $$\mathcal C_{0}\subset \mathcal C_{1} \subset  \{ \bsba:         \sum_{j\in \mathcal J^* }  \|   \bsba_j   \|_{\range}   \ge       \sum_{j\in \mathcal J^{*c} }  \|     \bsba_j  \|_{\range}\} . $$ The tightening of the restriction region corresponds to a larger value of $\kappa$, thereby leading to a smaller error. \\
\end{remark}

In the remainder of this section, we present a more general result together with the `seesaw' proof strategy. In the spirit of \citep{Loh2015}, this theorem does not require global convexity of the loss, and the single Bregman-form condition   \eqref{thmregcond-general0} is implied by all the aforementioned regularity conditions on restricted regions.
In what follows, given $\mathcal J \subset [p]$, we  denote by $\bsbb_{\mathcal J}$ the block subvector formed by  $\bsbb_j, j\in \mathcal J$.
\begin{theorem}\label{thm:seesawerr} Assume the loss $\tilde l_0$ is differentiable and satisfies the Lipschitz-gradient condition concerning    the statistical truth $\bsbe^*$, as given in \eqref{gradLip}, but is not necessarily convex.

Given $\lambda_0$ and $df^*$ and a sufficiently large constant $c_0$, define an event  through two random quantities $\Proj^{*,\perp}\bsbeps$ and $\Proj^{*}\bsbeps$:\begin{align}
\mathcal E:=\{ &\sup_\bsbb \, \|\Proj_{\Proj^{*,\perp}\tilde \bsbX {\mathcal  Q^{(b)}(  \bsbb)}}       \bsbeps\|_2^2  - \lambda_0^2  (\tau^{(b)}(\bsbb_{\mathcal J(\bsbb)}))^2-c_0 p\le 0, \ \|  \Proj^{*} \bsbeps \|_2^2 \le \sigma^2 df^*  \}. \label{seesawstochevent}
\end{align}
Assume  there exist  $\nu, K>0$  such that for any $\bsba=[\bsba_1^T, \ldots, \bsba_p^T]^T$
\begin{align}
&  \lambda \sum_{j\in \mathcal J^*}  (\,     \max_{k\in  \underline {\mathcal M} (\bsbb_j^* + \bsba_j)} \bsba_j[k]- \min_{ k\in \overline {\mathcal M} (\bsbb_j^* + \bsba_j)}     \bsba_j[k]\,) -\lambda \sum_{j\in \mathcal J^{*c}} \range (\bsba_j  )   \notag\\
\le\, &     2\bregs_{\tilde l_0}   (\tilde \bsbX   \bsbb^* +\tilde \bsbX \bsba, \tilde \bsbX   \bsbb^*)    - \nu \Breg_2(   \tilde \bsbX  \bsba  ,\bsb0)+C(K \sum_{j\in \mathcal J^*} (\tau_j^*)^2 \lambda^2 +\sigma^2df^{*}),   \label{thmregcond-general0} \end{align}
where   $\lambda = c(  \rho_0/(\nu \wedge 1)) \lambda_0$ and  $c, C$  are sufficiently large constants.

Then  on event $\mathcal E$,
\begin{align}
&     \|    \tilde \bsbX (\hat \bsbb -     \bsbb^*)\|_2^2   \lesssim         \frac{K\rho_0^2 }{\nu(\nu^2\wedge 1)} \sum_{j\in \mathcal J^*} (\tau_j^*)^2  \lambda_0^2 + \frac{\sigma^2}{ \nu^2\wedge 1}df^{*}, \label{prederrrate} \\
&         \sum_{j \in \hat {\mathcal J}} (\tau(\hat \bsbb_j))^2  \lesssim      K\rho_0^2(\nu\vee \frac{1}{\nu})   \sum_{j\in \mathcal J^*} (\tau_j^*)^2  + \frac{\nu^2\sigma^2}{\nu^2 \wedge 1}  \frac{df^{*}}{\lambda_0^2 } . \label{tauerrrate}
\end{align}
Assuming   $\EE[\bsbeps]=\bsb0$ and $\| \bsbeps\|_{\psi_2}\le \sigma$, we can  take $\lambda_0=c\sigma ( m + \sqrt{m\log p })$ and $df^* = c(M^*+p)$ such that $\mathcal E$ occurs with \emph{overwhelming} probability  $1-C \exp(-cp)$, where $ c,C$ are constants.
\end{theorem}

 \eqref{prederrrate} reveals a structured composition:   a ``degrees-of-freedom  term'' that reflects the complexity of the true model, and an ``inflation  term'' proportional to $\lambda^2$ that accounts for the costs  incurred in the search for structural patterns.
The theorem introduces  error bounds with  overwhelming   probability, which can also be reformulated into  {expectation}-form results as   in Theorem  \ref{thm:seesawerr-scvx}.  The fact that     $\EP[\mathcal E^c]$  is {\emph{exponentially}} small in
$p$ is in stark contrast to the polynomial dependence on the dimension typically seen in high-dimensional variable selection.

We remark that the subGaussian assumption in Theorems \ref{thm:seesawerr-scvx} and \ref{thm:seesawerr} can be relaxed to a subWeibull assumption: $\epsilon_i$ are independent and $\|\epsilon_i\|_{\psi_{\alpha}}\le \sigma$ for some $0<\alpha<2$,  as   a generalized Hanson-Wright inequality is applicable \citep{gotze2021concentration}. 
We will not explore this further in the current paper.
We also provide  a   comparison between the results of this section and those from the traditional dual approach   in Appendix \ref{subsec:dualnormanal} due to limited space.

\begin{proof}
First, from the optimality of $\hat \bsbb$,   a basic inequality can be derived using  the generalized Bregman calculus   \citep{Shebregman2021}:
\begin{align*}
\tilde l_0 (  \hat \bsbe) +     \breg_{\tilde l_0} (  \bsbe^*,  \hat \bsbe)  +  \lambda \sum_{j=1}^p \breg_{\range}(\bsbb_j^*, \hat \bsbb_j ) \le \tilde l_0 ( \bsbe^*) + \lambda \range^{(b)} (\bsbb^*) -  \lambda \range^{(b)} (\hat \bsbb )
\end{align*}
and based on the definition of the effective noise,
\begin{align}
2   \bregs_{\tilde l_0} ( \hat \bsbe,   \bsbe^*)  +  \lambda \sum_{j=1}^p \breg_{\range}( \bsbb_j^*,\hat \bsbb_j)\le\langle \bsbeps,   \tilde  \bsbX ( \hat    {\bsbb} - \bsbb^* )\rangle   + \lambda \range^{(b)} (\bsbb^*) -  \lambda \range^{(b)} (\hat \bsbb ).\label{seesawbasiceq1}
\end{align}

To achieve less stringent regularity conditions, we employ an orthogonal decomposition technique, inspired by  \cite{She2017selfact}. This is particularly helpful because   $\tau(\cdot)$ lacks subadditivity.   To begin with, define $$
\hat {\mathcal P}= \Proj_{\tilde \bsbX {\mathcal  Q^{(b)}(\hat \bsbb)}},\ \tilde {\mathcal P} = \Proj_{\Proj^{*,\perp}\hat \Proj  }= \Proj_{\tilde \Proj^{*,\perp}\bsbX {\mathcal  Q^{(b)}(\hat \bsbb)}},
$$
  and express the stochastic term
as\begin{align}
  \langle \bsbeps,  \tilde  \bsbX (\hat {\bsbb} - \bsbb^*)\rangle =\, &   \langle \bsbeps, \Proj^{*}   \tilde  \bsbX (\hat {\bsbb} - \bsbb^*)\rangle+   \langle \bsbeps, \Proj^{*,\perp}   \tilde  \bsbX (\hat {\bsbb} - \bsbb^*)\rangle
=   \langle\Proj^{*} \bsbeps, \Proj^{*}   \tilde  \bsbX (\hat {\bsbb} - \bsbb^*)\rangle+   \langle \bsbeps, \Proj^{*,\perp}  \hat \Proj \tilde  \bsbX  \hat {\bsbb}  \rangle
\notag\\
=\,&   \langle \Proj^{*} \bsbeps, \Proj^{*}   \tilde  \bsbX (\hat {\bsbb} - \bsbb^*)\rangle+   \langle  \tilde {\mathcal P}   \bsbeps, \Proj^{*,\perp}   \tilde  \bsbX  (\hat {\bsbb} - \bsbb^*) \rangle,\notag
\end{align}
  since $\tilde \bsbX \bsbb^*\in \Proj^*$
and $\tilde \bsbX \hat \bsbb \in \hat \Proj       $. It follows that for any $a>0$,
\begin{align}
\langle \bsbeps,  \tilde  \bsbX (\hat {\bsbb} - \bsbb^*)\rangle \le\,&  \frac{a}{2} \| \Proj^{*} \bsbeps\|_2^2 + \frac{1}{2a} \| \Proj^{*}   \tilde  \bsbX (\hat {\bsbb} - \bsbb^*)\|_2^2+  \frac{a}{2} \| \tilde {\mathcal P} \bsbeps\|_2^2   + \frac{1}{2a}\| \Proj^{*,\perp}   \tilde  \bsbX   (\hat {\bsbb} - \bsbb^*)\|_2^2\notag \\
 \le\,  & \frac{a}{2} \| \Proj^{*} \bsbeps\|_2^2   +  \frac{a}{2} \| \tilde {\mathcal P} \bsbeps\|_2^2  + \frac{1}{2a} \|     \tilde  \bsbX (\hat {\bsbb} - \bsbb^*)\|_2^2,\label{seesawstobound0}
\end{align}
where  the last inequality uses orthogonality. In contrast to naively applying H\"older's inequality (or the dual method), the first two terms on the right-hand side of \eqref{seesawstobound0} allow us to leverage some fine structures  to more effectively bound the stochastic term.\\

On  one hand, {statistical  analysis} enables us to   bound the noise terms using proper measures (such as $\tau(\cdot)$ or $\tau^{(b)}(\cdot)$), along with sufficiently large  regularization parameters.
We first bound the size of   $R_0$ defined as
\begin{align*}
 R_0 := &  \sup_\bsbb \big[\|\Proj_{\Proj^{*,\perp}\tilde \bsbX {\mathcal  Q^{(b)}(  \bsbb)}}       \bsbeps\|_2^2  - \lambda_0^2  (\tau^{(b)}(\bsbb_{\mathcal J(\bsbb)}))^2-c_0 p\sigma^2 \big]_+ \\
=&\sup_{0\le J \le p}\sup_{\bsbb:|\mathcal J(\bsbb)|= J} \, \big[\|\Proj_{\Proj^{*,\perp}\tilde \bsbX {\mathcal  Q^{(b)}(  \bsbb)}}       \bsbeps\|_2^2  - \lambda_0^2  (\tau^{(b)}(\bsbb_{\mathcal J(\bsbb)}))^2-c_0 p\sigma^2 \big]_+
\end{align*}
where $c_0$ is a sufficiently large constant.

Let $\bsbA_1(\bsbb) =  \Proj_{\Proj^{*,\perp}\tilde \bsbX {\mathcal  Q^{(b)}(  \bsbb)}} $. Given $\mathcal J\subset [p]$, define $\mathcal Q_{\mathcal J}$ as a block diagonal matrix with the $(j,j)$-th block being $\bsbI_m$ if $j\in \mathcal J $ and $\bsb1_m$ if $j\in \mathcal J^c$. Let
$$\bsbA_0(\mathcal J)   = \Proj_{\Proj^{*,\perp}\tilde \bsbX\mathcal Q_{\mathcal J} }$$ which satisfies $\rank(\bsbA_0  (\mathcal J))\le \rank(\mathcal Q_{\mathcal J})\le  m J + p - J  \le p + m J$. Below we   show a concentration result \eqref{HWboundA1}, fixing $\mathcal J$ and thus its cardinality $J = |\mathcal J|$ and assuming $\bsbA_0(\mathcal J)\ne \bsb0$  (otherwise \eqref{HWboundA1} holds trivially).

Noticing that    (i) $\EE[ \bsbeps]=\bsb0$, (ii)   $\|\bsbeps\|_{\psi_2}\leq \sigma$, (iii) $\|  \bsbA_0\|_2= 1$,   (iv) $\|  \bsbA_0\|_F^2= 1\cdot \rank(\bsbA_0) \le p + m J$,  we apply a  Hanson-Wright type inequality (specifically, only the upper tail bound is needed; see Section 6 in \cite{Vershynin2018}) to obtain for any $t\ge 0$,
\begin{align}
&    \mathbb{P}[ \|\bsbA_{0}(\mathcal J)  \bsbeps \|^2_2-L_0 (p + m J )\sigma^2\geq \sigma^2 t] \notag\\
\le \, & \mathbb{P}[ \|\bsbA _{0}(\mathcal J) \bsbeps \|^2_2-C_{0} \|  \bsbA_0(\mathcal J)\|_F^2 \sigma^2\geq (L_0-C_{0})(p + m J)\sigma^2+\sigma^2 t]
\notag\\  \leq \,&  \exp (-c\{  (L_0-C_{0}) (p + m J) +t\})  \le   \exp (-ct  ),\label{HWboundA1}
\end{align}
where   $C_{0}>0,L_0\ge C_{0}, c >0$ are  sufficiently large constants.

Moreover,  because   $ \Proj_{\Proj^{*,\perp} {\tilde \bsbX {\mathcal  Q^{(b)}}(\bsbb)} } \subset   \Proj_{\Proj^{*,\perp}\tilde \bsbX Q_{\mathcal J(\bsbb)} }   $, we can write   $ \Proj_{\Proj^{*,\perp}\tilde \bsbX Q_{\mathcal J(\bsbb)} } = \bsbU_0 (\mathcal J(\bsbb))(\bsbU_0(\mathcal J(\bsbb)))^T$, $\Proj_{\Proj^{*,\perp} {\tilde \bsbX {\mathcal  Q^{(b)}}(\bsbb)} }    = \bsbU_1 (\bsbb)(\bsbU_1(\bsbb))^T$, and $\bsbU_1 (\bsbb)= \bsbU_0 (\mathcal J(\bsbb)) \allowbreak\bsbV(\bsbb)$ with $\bsbU_0,\bsbU_1, \bsbV$  all orthogonal. It follows that\begin{align*}
\|\bsbA_1(  \bsbb)\bsbeps\|_2^2  &= Tr\{ \bsbeps^T  \bsbU_1 (\bsbb) (\bsbU_1 (\bsbb))^T \bsbeps\}   = Tr\{ \bsbeps^T  \bsbU_0 (\mathcal J(\bsbb)) \bsbV(\bsbb) (\bsbV(\bsbb))^T (\bsbU_0 (\mathcal J(\bsbb))) ^T \bsbeps\} \le \| \bsbA_0  (\mathcal J(\bsbb))\bsbeps\|_2^2.
\end{align*}
Let
\begin{align}
\lambda_0 = \sigma\sqrt{A_0}(m+\sqrt{m\log p}),
\end{align}
with  constant $A_0\ge L_0$. Then
\begin{align*}
\lambda_0^2 (\tau^{(b)} (\bsbb_{\mathcal J(\bsbb)}))^2 &= \sigma^2 A_0 (\sqrt m + \sqrt {\log p})^2\sum_{j\in \mathcal J(\bsbb)}  (\frac{m}{\overline M(\bsbb_j)}+\frac{m}{\underline M(\bsbb_j)})\ge 2A_0  \sigma^2 J(\bsbb)(m+\log p).
\end{align*}
Introducing $\bsbA_0(\mathcal J(\bsbb))$ allows us to simplify  the combinatorial counting for given $J:0\le J \le p$    as follows (which otherwise would involve $\overline{\mathcal M}_j, \underline{\mathcal M}_j$ within each block)
\begin{align*}
&\EP[ \sup_{\bsbb:|\mathcal J(\bsbb)|= J} \, \|\Proj_{\Proj^{*,\perp}\tilde \bsbX {\mathcal  Q^{(b)}(  \bsbb)}}       \bsbeps\|_2^2  - \lambda_0^2 (\tau^{(b)} (\bsbb_{\mathcal J(\bsbb)}))^2 - 2\sigma^2L_0 p \ge \sigma^2 t] \\
\le \, &  \EP[ \sup_{\mathcal J\subset[p]:|\mathcal J |= J} \, \|\bsbA_0(\mathcal J)   \bsbeps\|_2^2  - \sigma^2 L_0 (p + m J)\ge \sigma^2 A_0  J(m+\log p)+\sigma^2L_0 p + \sigma^2 t]\\
\le \, & {p\choose J}  \exp(-cA_0  J(m+\log p)-cL_{0  }p)\exp (-ct  )
\le  C \exp(-c Jm)\exp(-cL_0p) \exp(-ct),
\end{align*}
where we  set the constant $A_0$ to be sufficiently large and applied the Stirling bound. Applying the union bound again,

\begin{align*}
\EP[R_0 \ge \sigma^2 t] & \le
\sum_{J=0}^p C \exp(-cJm-cL_0p)\exp(-ct)\le C \exp(-cp) \exp(-ct)\le C \exp(-ct).
\end{align*}
The tail bound and the non-negativity of   $R_0$ imply $\EE[R_0]\le C \sigma^2$.

Similarly,
since $\rank(\mathcal P^*)\le M^* + 2p$, we have from \eqref{HWboundA1}\begin{align}\label{HWboundPstar}
    \mathbb{P}[ \|  \Proj^{*} \bsbeps \|^2_2-L_0 (M^*+2p )\sigma^2\geq \sigma^2 t]\leq \exp (-ct  ).
\end{align}
 Let
\begin{align}
 df^*=4L_0(M^* + p)
\end{align}
and define
\begin{align*}
 R_1 := &   \big[  \|  \Proj^{*} \bsbeps \|^2_2- \sigma^2df^*  \big]_+ .
\end{align*}
 Then
 \begin{align*}
 \EP[R_1\ge \sigma^2 t]& \le  \mathbb{P}[ \|  \Proj^{*} \bsbeps \|^2_2-df^{*}\sigma^2\geq  \sigma^2 t]\\
&\leq\mathbb{P}[ \|  \Proj^{*} \bsbeps \|^2_2-L_0 (M^*+2p )\sigma^2\geq 2L_0 (M^*+p )\sigma^2+ \sigma^2 t]\\
& \le \exp (-c(M^*+p))\exp(-ct)\le  \exp(-c p)\exp(-ct),
\end{align*}
and so $\EE[R_1]\le C \sigma^2$. Given the above choices of $\lambda_0$ and $df^*$ (and taking a large constant $c_0$),    $$
\EP(\mathcal E^c)\le C \exp(-cp).
$$
Thus,    $\mathcal E$ occurs with \textit{overwhelming} probability as stated in Theorem \ref{thm:seesawerr}.

So far, we have obtained
\begin{align}
\|    \tilde \Proj    \bsbeps\|_2^2    \le   \lambda_0^2  (  \tau^{(b)}(\hat \bsbb_{\hat {\mathcal J}}))^2+c_0 p\sigma^2 + R_0    \label{seesawineq1}
\end{align}
 where $\tau^{(b)}(\hat \bsbb_{\hat {\mathcal J}})= \big(\sum_{j\in \hat {\mathcal J} } \hat \tau_j^2 \big)^{1/2}  $ with  $\hat \tau_j  =\tau   (\hat \bsbb_j)$.  The remaining challenge  boils down to bounding   $\hat \tau_j  $ to determine the error rate (which, however, can be circumvented  with a clever trick).
\\

On the other hand, we can derive useful properties of the estimator through  {optimization  analysis}. Let $\hat \bsbU =\diag\{\hat \bsbU_j \}$ denote   $ \bsbU (\hat \bsbb)$ (or $\diag\{\bsbU  (\hat \bsbb_j)\}$) which is orthogonal.
 Then we can write  $\hat \bsbb = \hat \bsbU \hat \bsbg$. Construct
$$
\hat \bsbt = \hat \bsbU^T \hat \bsbs, \  \forall \hat \bsbs\in \partial \range^{(b)}(\hat \bsbb)
$$
where, according to Theorem \ref{thm:subgrad}, the choice of the subgradient does not affect $\hat \bsbt$. Indeed, assuming without loss of generality the components of $\hat \bsbb_j$ are ordered consecutively as  $ \hat{\overline{\mathcal M}}_j,   \hat{\mathcal M}_j,  \hat{\underline{\mathcal M}}_j $,  a straightforward calculation
 shows $$
\hat \bsbt_j =    \hat{\bsbU}_j^T \hat \bsbs_j =
\begin{cases}
\begin{bmatrix}\frac{1}{\sqrt{\hat{\overline M}_j}}  \\ \bsb0_{\hat M_j} \\ -\frac{1}{\sqrt{\hat{\underline M}_j}} \end{bmatrix}, & j \in \hat{ \mathcal J}  \\
\bsb0, & j \in \hat{ \mathcal J}^{c}.
\end{cases}
$$
\begin{lemma} \label{lem:redproblemopt}
Let $\hat \bsbb \in \arg\min_{\bsbb} L(\bsbb) + \lambda \range^{(b)}(\bsbb)$, where $L$ is differentiable. Then with $\hat \bsbU, \hat\bsbt$ constructed as above,   $\hat \bsbg=\hat\bsbU^T\hat\bsbb$ is a globally optimal solution to
\begin{align}
\min_{\bsbg}L  (  \hat \bsbU \bsbg)  + \lambda \langle \hat \bsbt,    \bsbg\rangle. \label{gammaopt}
\end{align}
\end{lemma}
  The proof details are    omitted. Notably,  Lemma  \ref{lem:redproblemopt} does \textit{not} require  $L$ to be convex;        the objective  in \eqref{gammaopt} is differentiable.  Therefore, $\hat \bsbb$ satisfies
\begin{align}
\hat \bsbU^T\nabla L(\hat \bsbb) + \lambda \hat\bsbt = \bsb0,
\end{align}
or
\begin{align}
\lambda \hat \bsbt &= -\hat \bsbU ^T\tilde \bsbX^T(\nabla \tilde l_0(\hat \bsbe) - \nabla \tilde l_0( \bsbe^*) )+\hat \bsbU ^T   \tilde \bsbX^T\bsbeps. \label{seesawkkt0}
\end{align}
  Multiplying both sides of \eqref{seesawkkt0} by $ \hat \bsbU $ yields
\begin{align}
 \lambda\Proj_{\hat \bsbU} \hat \bsbs = -\Proj_{\hat \bsbU} \tilde \bsbX^T(\nabla \tilde l_0(\hat \bsbe) - \nabla \tilde l_0( \bsbe^*) )+ \Proj_{\hat \bsbU}   \tilde \bsbX^T\bsbeps, \notag 
\end{align}
  or
\begin{align}
  \lambda \Proj_{\hat \bsbU_j}\bsbs_j = -     \Proj_{\hat \bsbU_j}\tilde \bsbX_j^T(\nabla \tilde l_0(\hat \bsbe) - \nabla \tilde l_0( \bsbe^*) )+ \Proj_{\hat \bsbU_j}  \tilde \bsbX_j^T\bsbeps, 1\le j \le p.
  \label{seesawstateq}
\end{align}
Applying Theorem \ref{thm:subgrad} again, we have   \begin{align}
 \Proj_{\hat \bsbU_j} \bsbs_j=
\begin{cases}
\begin{bmatrix}\frac{1}{\hat{\overline M_j}}\bsb1\\ \bsb0_{\hat M_j} \\ -\frac{1}{\hat{\underline M_j}}\bsb1\end{bmatrix}, &j \in \hat{\mathcal J } \\
\bsb0, &j \in \hat{\mathcal J}^{c}.
\end{cases}
\end{align}
Taking the Euclidean norm and then squaring both sides of \eqref{seesawstateq} produces
bounds for $\hat \tau_j^2$  which is \textit{exactly}   $\| \Proj_{\hat \bsbU_j} \bsbs_j\|_2^2$ for $j\in \hat{\mathcal J}$:
\begin{align*}
\lambda^2 \hat \tau_j^2 & \le 2\|       \Proj_{\hat \bsbU_j}  \tilde \bsbX_j^T\bsbeps\|_2^2 + 2\| \Proj_{\hat \bsbU_j}\tilde \bsbX_j^T(\nabla \tilde l_0(\hat \bsbe) - \nabla \tilde l_0( \bsbe^*) )\|_2^2 \\
 &\le 4\| \Proj_{\hat \bsbU_j}  \tilde \bsbX_j^T \Proj^{*,\perp}\bsbeps\|_2^2 +4\|      \Proj_{\hat \bsbU_j}  \tilde \bsbX_j^T \Proj^*\bsbeps\|_2^2  + 2\| \Proj_{\hat \bsbU_j}\tilde \bsbX_j^T(\nabla \tilde l_0(\hat \bsbe) - \nabla \tilde l_0( \bsbe^*) )\|_2^2.
\end{align*}
Based on the Lipschitz-gradient
assumption and the choice of $\rho_0$,
\begin{align*}
\|    \tilde \bsbX ^T(\nabla \tilde l_0(  \bsbe) - \nabla \tilde l_0( \bsbe^*) )\|_2^2\le\|  \tilde \bsbX  \|_2^2\cdot \varrho^2 \|    \bsbe  -   \bsbe^*\|_2^2  \le \rho_0^2 \|    \bsbe  -   \bsbe^*\|_2^2,
\end{align*}
and  $\|\tilde \bsbX\|_2^2 \le \rho_0^2$.
Hence summing over $j$ gives\begin{align*}
\lambda^2 (  \tau^{(b)}({\hat \bsbb}_{\hat {\mathcal J}}))^2 & \le 4\| \Proj_{\hat \bsbU}  \tilde \bsbX^T\Proj^{*,\perp}\bsbeps\|_2^2 + 4\| \Proj_{\hat \bsbU}  \tilde \bsbX^T\Proj^{*}\bsbeps\|_2^2 + 2\| \Proj_{\hat \bsbU }\tilde \bsbX ^T(\nabla \tilde l_0(\hat \bsbe) - \nabla \tilde l_0( \bsbe^*) )\|_2^2\\
 & \le 4\| (\Proj^{*,\perp}\hat \Proj  \tilde \bsbX \Proj_{\hat \bsbU})^T  \bsbeps\|_2^2 + 4\rho_0^2 \|  \Proj^*\bsbeps\|_2^2+ 4 \rho_0^2 \Breg_2(  \hat \bsbe ,   \bsbe^*)\\
 & \le 4\| \Proj_{\hat \bsbU}  \tilde \bsbX^T \Proj^{*,\perp}(\tilde \Proj \bsbeps)\|_2^2 + 4\rho_0^2 \|  \Proj^*\bsbeps\|_2^2+ 4 \rho_0^2 \Breg_2(  \hat \bsbe ,   \bsbe^*)\\
&\le 4 \rho_0^2 \| \tilde \Proj  \bsbeps\|_2^2 + 4\rho_0^2 \|  \Proj^*\bsbeps\|_2^2+ 4 \rho_0^2 \Breg_2(  \hat \bsbe ,   \bsbe^*) ,
\end{align*}
where we used   $  \tilde \bsbX \Proj_{\hat \bsbU}=  \hat \Proj  \tilde \bsbX \Proj_{\hat \bsbU}$ and $\Proj^{*,\perp}\hat \Proj=\tilde \Proj\Proj^{*,\perp}\hat \Proj$.
Setting $\lambda = \alpha  \rho_0 \lambda_0$ (with $\alpha>0$ to be specified later) results in
\begin{align}
\lambda_0^2 (\tau^{(b)}({\hat \bsbb}_{\hat {\mathcal J}}))^2 \le \frac{4}{\alpha^2}\big(\| \tilde  \Proj  \bsbeps\|_2^2 + \|  \Proj^*\bsbeps\|_2^2+ \Breg_2(  \hat \bsbe ,   \bsbe^*)  \big). \label{seesawineq2}
\end{align}

Interestingly, in \eqref{seesawineq1} and \eqref{seesawineq2},  $\| \tilde \Proj  \bsbeps\|_2^2 $ and  $\lambda_0^2 (\hat \tau^{(b)})^2 $
bound each other in a `seesaw' fashion up to some manageable additive terms. By combining the statistical bound with the optimization-based bound, we obtain a  {key} result (assuming $\alpha>2$):
\begin{align}
\|\tilde \Proj  \bsbeps \|_2^2\le \frac{4}{\alpha^2 - 4}\big( \|  \Proj^*\bsbeps\|_2^2+   \Breg_2(  \hat \bsbe ,   \bsbe^*)\big)+ \frac{\alpha^2}{\alpha^2 - 4}(c_0 p\sigma^2 + R_0).\label{seesawstochSq}  \end{align}
 Meanwhile,
\begin{align}
\lambda_0^2 (\tau^{(b)}({\hat \bsbb}_{\hat {\mathcal J}}))^2 \le \frac{4}{\alpha^2 - 4}\big( \|  \Proj^*\bsbeps\|_2^2+ \Breg_2(  \hat \bsbe ,   \bsbe^*)+ c_0 p\sigma^2+R_0\big). \label{seesawtaubnd}
 \end{align}
Plugging \eqref{seesawstochSq} into \eqref{seesawstobound0}  bounds  the stochastic term by
\begin{align}
\langle \bsbeps,  \tilde  \bsbX (\hat {\bsbb} - \bsbb^*)\rangle
 \le  & \, ( \frac{a}{2} + \frac{2a}{\alpha^2 - 4} )\sigma^2 df^{*}     +(\frac{1}{a}+ \frac{2a}{\alpha^2 - 4} )\Breg_2(  \hat \bsbe ,   \bsbe^*)  +    \frac{a\alpha^2}{2(\alpha^2 - 4)}(c_0p\sigma^2+ R_0) + ( \frac{a}{2} + \frac{2a}{\alpha^2 - 4} )R_1
\end{align}
or the following bound on  $\mathcal E$:
\begin{align}
\langle \bsbeps,  \tilde  \bsbX (\hat {\bsbb} - \bsbb^*)\rangle
 \le  & \, ( \frac{a}{2} + \frac{2a}{\alpha^2 - 4} +    \frac{a\alpha^2}{\alpha^2 - 4})c\sigma^2  df^{*}     +(\frac{1}{a}+ \frac{2a}{\alpha^2 - 4} )\Breg_2(  \hat \bsbe ,   \bsbe^*).
\end{align}
Notably, this bound  does \emph{not} involve  the regularizer, distinguishing it from those obtained through H\"older's inequality or dual approaches. When
$\alpha$ is chosen to be large, $\Breg_2(  \hat \bsbe ,   \bsbe^*)$ can be overshadowed by the generalized Bregman   of the loss under suitable regularity conditions.

Now, \eqref{seesawbasiceq1} becomes
\begin{align}
&2   \bregs_{\tilde l_0} ( \hat \bsbe,   \bsbe^*)  -(\frac{1}{a}+ \frac{2a}{\alpha^2 -4} )\Breg_2(  \hat \bsbe ,   \bsbe^*)+  \lambda \sum_{j=1}^p \breg_{\range}( \bsbb_j^*,\hat \bsbb_j)\notag\\
 \le \, & c(\frac{a}{2} + \frac{2a}{\alpha^2 -4} +    \frac{a\alpha^2}{\alpha^2 - 4})\sigma^2     df^{*} + \lambda \range^{(b)} (\bsbb^*) -  \lambda \range^{(b)} (\hat \bsbb )+       \frac{a\alpha^2}{2(\alpha^2 - 4)}R_0 + ( \frac{a}{2} + \frac{2a}{\alpha^2 - 4} )R_1. \label{seesawfinal0}
\end{align}
Let
$$
F=K(\tau^{(b)}(\bsbb_{\mathcal J^*}^*))^2=K \sum_{j\in \mathcal J^*} (\tau_j^*)^2
$$
with $\tau_j^* = \tau(\bsbb_j^*)$. Observe the following relationship:
\begin{align}
\delta \range(\hat{ \bsbb}_j; \bsbb_j^*  - \hat{\bsbb}_j) &  = \range(  \bsbb_j^*) - \range (\hat{\bsbb}_j) - \breg_{\range}( {\bsbb}_j^*, \hat{\bsbb}_j)  = -2 \bregs_{\range}(\hat{\bsbb}_j, {\bsbb}_j^*  )- \delta \range( { \bsbb}_j^*; \hat{\bsbb}_j   -  {\bsbb}_j^*).
\label{dddeltarel}
\end{align}
Suppose the following condition holds:
\begin{align}
&       \lambda \sum_{j=1}^p \delta  \range (\hat \bsbb_j; \bsbb_j^* - \hat \bsbb_j )   \le   (2- \alpha_{1})\bregs_{\tilde l_0}   (   \hat \bsbe,     \bsbe^*)   - \alpha_{2} \Breg_2    (   \hat \bsbe,      \bsbe^*) +C(F \lambda^2 +df^{*}\sigma^2)\label{seesawregcond0}
\end{align}
for some $\alpha_1\ge 0, \alpha_2>0$  and   constant  $C>0$. Then adding  \eqref{seesawfinal0} and \eqref{seesawregcond0}   with   $  \frac{1}{a} + \frac{2a}{\alpha^2 -4} \le \frac{ \alpha_{2}}{2}$ or
$$a=\frac{4}{\alpha_2}, \alpha^2 =\frac{32}{\alpha_2^2}+4 $$     gives
\begin{align*}
 \alpha_{1}\bregs_{\tilde l_0}   (  \hat \bsbe,    \bsbe^*)+\frac{\alpha_{2}}{2} \Breg_2  (   \hat \bsbe,    \bsbe^*)   \le \, &    C F \lambda^2 +(\frac{2}{\alpha_{2}}+\frac{\alpha_{2}}{4}+C(\alpha_2 \vee \frac{1}{\alpha_2}))\sigma^2df^{*}  +   C(\alpha_2 \vee \frac{1}{\alpha_2})  R_0 +     C(\alpha_2 \vee \frac{1}{\alpha_2})  R_1
\end{align*}
and thus
\begin{align}
& \EE[\alpha_{1}\bregs_{\tilde l_0}   (\tilde \bsbX \hat \bsbb, \tilde \bsbX   \bsbb^*) \vee \alpha_{2} \|    \tilde \bsbX (\hat \bsbb -     \bsbb^*)\|_2^2 \vee   \frac{1}{\alpha_2}  \lambda_0^2 (\tau^{(b)}({\hat \bsbb}_{\hat {\mathcal J}}))^2]\notag\\
  \lesssim \ &       \frac{1}{\alpha_{2}^2\wedge 1}(K\rho_0^2 \sum_{j\in \mathcal J^*} (\tau_j^*)^2  \lambda_0^2 +\alpha_2\sigma^2df^{*}) +(\alpha_2 \vee \frac{1}{\alpha_2})\sigma^2  \notag \\
 \lesssim \ &       \frac{1}{\alpha_{2}^2\wedge 1}(K\rho_0^2 \sum_{j\in \mathcal J^*} (\tau_j^*)^2  \lambda_0^2 +\alpha_2\sigma^2df^{*})  \label{genres0}
\end{align}
and on event $\mathcal E$,
\begin{align}
\alpha_{1}\bregs_{\tilde l_0}   (\tilde \bsbX \hat \bsbb, \tilde \bsbX   \bsbb^*) \vee \alpha_{2} \|    \tilde \bsbX (\hat \bsbb -     \bsbb^*)\|_2^2 \vee   \frac{1}{\alpha_2}  \lambda_0^2 (\tau^{(b)}({\hat \bsbb}_{\hat {\mathcal J}}))^2
  \lesssim        \frac{1}{\alpha_{2}^2\wedge 1}(K\rho_0^2 \sum_{j\in \mathcal J^*} (\tau_j^*)^2  \lambda_0^2 +\alpha_2\sigma^2df^{*}). \label{genres}
\end{align}
The bound for $(\tau^{(b)}({\hat \bsbb}_{\hat {\mathcal J}}))^2$  is due to \eqref{seesawtaubnd}:   $\lambda_0^2 (\tau^{(b)}({\hat \bsbb}_{\hat {\mathcal J}}))^2 \le  ({\alpha_2}/{8})    ( \alpha_2 \sigma^2 df^* + \alpha_2 \Breg_2(  \hat \bsbe ,   \bsbe^*) +\alpha_2c_0 p\sigma^2 + \alpha_2R_0 + \alpha_2R_1)\le (   {\alpha_2}/{8})    (c \alpha_2 \sigma^2 df^* + \alpha_2 \Breg_2(  \hat \bsbe ,   \bsbe^*))  $. \eqref{genres} implies \eqref{prederrrate}  and \eqref{tauerrrate} with $\alpha_1=0$.

Let   $\bsba = \hat \bsbb - \bsbb^*$. Based on  Theorem \ref{thm:subgrad},
$\delta \range(\hat{ \bsbb}_j; \bsbb_j^*  - \hat{\bsbb}_j)=\max_{ k\in \overline {\mathcal M}(\hat \bsbb_j)}     (\bsbb_j^* [k] -\hat\bsbb_j  [k])-  \min_{ k\in \underline {\mathcal M}(\hat \bsbb_j)}(\bsbb_j^*[k] -\hat \bsbb_j [k])=\max_{ k\in \overline {\mathcal M}(\hat \bsbb_j)}     (- \bsba_j [k])  -  \min_{ k\in \underline {\mathcal M}(\hat \bsbb_j)}(- \bsba_j [k])$, which, for $j \in \mathcal J^{*c}$, equals $- \max_{ k\in [m]}    \hat \bsbb_j [k] +     \min_{ k\in [m]}  \hat \bsbb_j [k] =-\range(\hat \bsbb_j)$,  as   $\bsbb_j^*\in \Proj_K$. Thus \eqref{seesawregcond0} can equivalently written as
\begin{align}
&  \lambda \sum_{j\in \mathcal J^*}  (\,     \max_{  \underline {\mathcal M} (\bsbb_j^* + \bsba_j)} \bsba_j[k]- \min_{  \overline {\mathcal M} (\bsbb_j^* + \bsba_j)}     \bsba_j[k]\,) -\lambda \sum_{j\in \mathcal J^{*c}} \range (\bsba_j  )   \notag\\
\le\, &     (2 -\alpha_{1})\bregs_{\tilde l_0}   (\tilde \bsbX   \bsbb^* +\tilde \bsbX \bsba, \tilde \bsbX   \bsbb^*)    - \alpha_{2} \Breg_2(   \tilde \bsbX  \bsba  ,\bsb0)+C(F \lambda^2 +\sigma^2df^{*}).   \label{seesawregcond-general0} \end{align}
The regularity condition in Theorem     \ref{thm:seesawerr} corresponds to $\alpha_1=0, \alpha_2=\nu$.

When $\tilde l_0$ is $\nu$-strongly convex, $\bregs_{\tilde l_0}      \ge \nu\Breg_2$.
 A    sufficient condition for \eqref{seesawregcond-general0} can be obtained by (i) setting $\alpha _{2} = \nu(2-\alpha_1)/2$,
(ii) omitting  the term $\sigma^2df^{*}$, and then (iii) applying the AM-GM inequality:
\begin{align}
  \sum_{j\in \mathcal J^*}  (\,     \max_{  \underline {\mathcal M} (\bsbb_j^* + \bsba_j)} \bsba_j[k]- \min_{  \overline {\mathcal M} (\bsbb_j^* + \bsba_j)}     \bsba_j[k]\,)  -  \sum_{j\in \mathcal J^{*c}}    \range (\bsba_j )
\le    \frac{1}{\kappa}  \big(\sum_{j\in \mathcal J^*} (\tau_j^*)^2\big)^{1/2} \{
\bregs_{\tilde l_0}     (\tilde \bsbX   \bsbb^* +\tilde \bsbX \bsba, \tilde \bsbX   \bsbb^*)     \}^{1/2}     \label{seesawregcond-scvx}
\end{align}
where $1/\kappa=\sqrt{2 C(2-\alpha_1) K } $.

It is easy to see that the regularity condition outlined  in \eqref{seesawcompcondbound0} and \eqref{seesawrestrictregion-general0} indicates \eqref{seesawregcond-scvx}, and thus \eqref{seesawregcond-general0}   with  $\alpha_1 = 1$, $\alpha_2 = \nu/2$, $K = c/\kappa^2 $. The conclusion in Theorem \ref{thm:seesawerr-scvx}
follows.
 \end{proof}

\subsection{Pattern Recovery}
\label{subsec:pattrec}
This part studies when the estimator achieves exact pattern recovery.
This is analogous to the studies on sign consistency for lasso (see, e.g., \cite{Zhao2006} and \cite{Wain2009}), which were among the first significant contributions to high-dimensional statistics.

When $L(\bsbb)$ and thus   $f(\bsbb)$ are convex,   $\hat \bsbb$ is a globally optimal solution if and only if it satisfies the KKT condition:  $\bsb0\in \partial f(\hat \bsbb)$.
 For simplicity, in this subsection, we consider regression in federated learning  with $
l_0(\bsbe_k) =    \|\bsbe_k - \bsby_k\|_2^2/2
$, $1\le k \le m$, and the corresponding   model is $\bsby  = \tilde \bsbX\bsbb^*+\bsbeps$ and $\bsbeps =[\bsbeps_1^T, \ldots , \bsbeps_p^T]^T $.
Define   $p_e$ as the probability that there exists a globally optimal solution $\hat \bsbb$, characterized by   $\bsb0\in \partial f(\hat \bsbb)$, satisfies  $$\mathcal Q^{(b)}(\hat \bsbb) = \mathcal Q^{(b)}(  \bsbb^*). $$

The theorem below requires a ``mutual coherence'' condition for the design, as well as  a
  gap or ``beta-min'' condition on the signal-to-noise ratio, to ensure successful pattern recovery.
 Define \begin{align*}
G^*   =\min_{j\in \mathcal J^*} G_j^*, \text{ with }
 G_j^*  =  (\max\bsbb_j^*-\max  \bsbb_j^*[{\mathcal M}^*_j] ) \wedge  (\min  \bsbb_j^*[{\mathcal M}^*_j] - \min\bsbb_j^* ).
\end{align*}
Let $\bsbU^* = \bsbU(\bsbb^*)$ and
\begin{align*}
\begin{cases}\bsbt^*  =  \bsbU^{*  T} \bsbs(\bsbb^*)
\\
 \overunderline\bsbs^*  =   \bsbU^* \bsbt^* = \Proj_{\bsbU^* } \bsbs(\bsbb^*), \end{cases}
\quad  \forall \bsbs(\bsbb^*)\in \partial \range^{(b)}(\bsbb^*).
\end{align*}
As aforementioned,  $\bsbt^*, \overunderline\bsbs^*$ are well-defined regardless of the subgradient choice (cf.  Theorem \ref{thm:subgrad}).
 Let  $\tilde \bsbX^* = \tilde \bsbX \Proj_{ \bsbU^*}$, $\tilde \bsbX^{*,\perp} = \tilde \bsbX \Proj_{ \bsbU^*}^\perp$ with $\tilde \bsbX^* +\tilde \bsbX^{*,\perp}=\tilde \bsbX  $.
\begin{theorem} \label{th:pattrecov}
In the    federated regression setting with  $\lambda>0$, define \begin{gather*}
\bsbh =[\bsbh_1^T, \ldots, \bsbh_p^T]^T
=   \{( \tilde \bsbX^{*})^+   \tilde \bsbX^{*,\perp}    \}^T\,  \overunderline{\bsbs}^*. 
\end{gather*}
Assume a regular design  and a properly large signal in terms of $\bsbh$, $G^*$:
\begin{align}
&\max_{1\le j\le p}\  {\overunderline M_j^*} \| \bsbh_j\|_{\infty}\le \alpha<1 \label{mutcohcond1}\\
 &2 \|  ( \tilde \bsbX^{*T}    \tilde \bsbX^{*}    )^{+}\, {\overunderline \bsbs}^{*} \|_{\infty}\lambda/G^* \le  \alpha'<1.
\end{align}
Then $p_e\ge \EP[ \mathcal E_1 \cap \mathcal E_2]$, where
\begin{align}
\mathcal E_1 & =\{ \bsbeps:  m\| ( \tilde \bsbX^{*,\perp}  )^T\Proj^{*,\perp} \bsbeps \|_{\infty}   \le  (1-\alpha) \lambda\},\label{defE1}\\
\mathcal E_2 &=  \{\bsbeps: 2 \|  (\tilde \bsbX^{ *})^{+} \Proj^{* } \bsbeps  \|_{\infty}  \le  (1-\alpha') G^*\}.   \label{defE21}
\end{align}
Additionally, $\mathcal E_1$ can be written as $ \{ \bsbeps:  m\max_{1\le j \le p} \| (
{\tilde \bsbX}_j[:,  \overunderline {\mathcal M}_j^*] )^T\Proj^{*,\perp} \bsbeps \|_{\infty}
\le (1-\alpha) \lambda\}$, and $\mathcal E_2   \supset  \{\bsbeps: 2 \|  (\tilde \bsbX \bsbU^{ *})^{+}   \bsbeps  \|_{\infty}  \le  (1-\alpha') G^*\}$. \end{theorem}

 Based on the proof,  $ \bsbh_{j}[      {\mathcal M}_j^*] = \bsb0$, and \eqref{mutcohcond1} can be relaxed to
$
 \overline{M}_j^* \|\bsbh_{j}[     \overline{\mathcal M}_j^*]\|_\infty \allowbreak \vee \underline{M}_j^* \|\bsbh_{j}[     \underline{\mathcal M}_j^*]\|_\infty \le \alpha$ ($ 1\le j\le p$).
Let  $\bsbZ  = [\bsbZ_1, \ldots , \bsbZ_p]= \tilde \bsbX \bsbU^*$,  where $\bsbZ_j = \tilde \bsbX _j \bsbU_j^*$ has dimensions   $ N\times (2+M_j^*)$ for $j\in \mathcal J^*$ and $N\times 1$ for $j\in \mathcal J^{*c}$.
Then $\bsbh$ can also be written as
$$
\bsbh  = ( \tilde \bsbX^{*,\perp})^T \bsbZ (\bsbZ^T \bsbZ)^{-1} \bsbt^*. $$
 This is a  pivotal metric in both   nonasymptotic and asymptotic analyses of exact pattern recovery.  It enables the extension of the  \textit{irrepresentable conditions} \citep{Zhao2006} to asymptotic studies concerning range regularizers.  Clearly,    a small value of  $\| \bsbh\|_\infty$ is advantageous  for recovering extreme-value patterns, achievable through a design  exhibiting low coherence (between  $ \tilde \bsbX \bsbU^{*}$ and  $ \tilde \bsbX \bsbU^{*,\perp}$) and   maintaining the nondegeneracy  of    $ \tilde \bsbX \bsbU^{*} $.
 In the special   case of i.i.d. Gaussians, we can apply Anderson's shifted-ball inequality to calculate $p_e$.
\begin{corollary}\label{cor:pattrecovGaussian}
In the context of Theorem \ref{th:pattrecov}, assume $\bsbeps=[\epsilon_i]$ has i.i.d. $\mathcal N(0,\sigma^2)$ entries. Let   $\Phi, \varphi$ denote the   distribution and density functions, respectively, for the standard normal. Let $df^*= M^*+p+J^*$.
Define two numbers, $\omega_0, \mu_0$, to quantify the degrees of \emph{mean incoherence} and \emph{model nondegeneracy}, respectively:
\begin{align}
\| (\tilde \bsbX^{*,\perp} )^T\bsbZ \|_{2, \infty } /\sqrt{df^*}\le \omega_0 ,  \quad \lambda_{\min}  (\bsbZ^T\bsbZ) \ge \mu_0.
\end{align}
 Then
\begin{align}
p_e \ge \{1 - 2\Phi(-A  )\}^{pm-M^*} \{1 - 2\Phi(-A')\}^{ df^*} \label{corpattprobbound}
\end{align}
and thus
$
1-p_e \le 2 pm \varphi(A )/A  +2   df^*  2\varphi(A')/A',
$
where $A  =\frac{\lambda}{\sigma\| \tilde \bsbX\|_2 } (\frac{1}{m}-\frac{\omega_0\tau_{\mathcal J^*} \sqrt{df^*}}{\mu_0}  ) >0, A' =    \frac{\sqrt{\mu_0}}{\sigma} (   \frac{G^*}{2}- \lambda \frac{\tau_{\mathcal J^*}}{\mu_0})>0$, and   $\tau_{\mathcal J^*}$ is short for  $ \sqrt{\sum_{j \in \mathcal J^*}  \tau_j ^{*2}} $.

\end{corollary}

  Theorem \ref{th:pattrecov} and  Corollary \ref{cor:pattrecovGaussian} establish  specific probability bounds. Interestingly, the required rate   of $\lambda$    in Section \ref{subsec:errrate}   for achieving satisfactory statistical error      is   \emph{lower} than  the rate  needed for exact pattern recovery here, indicating   high accuracy can  be maintained without necessarily achieving exact pattern recovery.
For example,  in \eqref{defE1}, under proper  incoherence assumptions, setting $\lambda$   sufficiently  large
  to suppress   $
 m\| ( \tilde \bsbX^{*,\perp}  )^T\Proj^{*,\perp} \bsbeps \|_{\infty}
 =m\cdot\max_{1\le j\le p} \max_{  k\in \overunderline{\mathcal M}_j^* } |(\tilde\bsbX_j[:,k])
^T\Proj^{*,\perp} \bsbeps |$
with high probability leads  to a    rate  of $\sigma m \sqrt{\log  (m  p)}$ (omitting trivial factors),  which  contrasts with the rate of $ \sigma (m + \sqrt{m \log p}) $ in Section \ref{subsec:errrate}. The reader can also compare these rates to  those derived from seminorm analysis in Appendix \ref{subsec:dualnormanal}.

Further insight into the role of $m$, which is intrinsically more delicate, can be gained from the pattern recovery probability bound \eqref{corpattprobbound} under appropriate regularity conditions. As an illustration, consider the  factor $\{1 - 2\Phi(-A')\}^{ df^*}$. Under a beta-min type condition     $   G^*  = 2c \lambda  {\tau_{\mathcal J^*}}/{\mu_0}$ with $c>1$,  the quantity  $A'$   is of the same order as $\lambda  {\tau_{\mathcal J^*}}/(\sigma\sqrt\mu_0).$ Here,    $\mu_0 \le \max_k \lambda_{\max}(\bar \bsbX_k^T \bar \bsbX_k )$ and may therefore be viewed  as an $m$-stable restricted-eigenvalue quantity (recalling $\bar \bsbX_k$ are   scaled designs). Moreover,      $\tau_{\mathcal J^*}= \sqrt{\sum_{j \in \mathcal J^*}  \tau_j ^{*2}} $, with
  $1/\sqrt m \lesssim   {\tau_{j}^*} \lesssim 1$. Therefore, under the choices of   $\lambda$ considered in the paper,    $A'$ increases with $m$.
 On the other hand, with $M^* = \sum_{j\in \mathcal J^*} M_j^*$, the exponent $
df^*=M^*+p+J^*$ also increases with $m$  when   $M_j^*\asymp m$. Hence, this probability factor reflects a trade-off: as the number of clients grows, its base probability may improve, while the exponent may increase at the same time. This suggests that the role of $m$ in pattern recovery is inherently delicate.
\section{Computation}
\label{sec:comp}
Recall the objective function $f(\bsbb) = L(\bsbb) + \lambda \range^{(b)}(\bsbb)$.
Assuming $L$ is $\mathcal L$-strongly smooth such that $\breg_L\le \mathcal L\Breg_2$, a simple proximal gradient algorithm can be used to solve \eqref{sysexpress}:
\begin{align}
\bsbb^{(t+1)} = \prox_{\alpha_t\lambda\range^{(b)}} (\bsbb^{(t)} - \alpha_t \nabla L(\bsbb^{(t)})) \label{proxgrad}
\end{align}
where $\alpha_t \le 1/\mathcal  L$. \eqref{proxgrad}  is straightforward to implement and is much more efficient than operator-splitting methods such as ADMM.

Observing that according to \eqref{betabB},
\begin{align}
\nabla L(\bsbb) = \vect(([\nabla l  (\bsb{b}_1; \bsby_1, \bsbX_1), \ldots, \nabla l  (\bsb{b}_m; \bsby_m, \bsbX_m)])^T), \label{gradFed}
\end{align}
we can deploy a federated learning algorithm   as follows for $t\ge 0$:
\begin{itemize}
\item Client $k$: Retrieve  $\bsb{b}_k ^{(t)}$. Compute  $\bsbxi_k ^{(t+1)} = \bsb{b}_k ^{(t)}-\alpha_t \nabla l  (\bsb{b}_k^{(t)}; \bsby_k, \bsbX_k) $. Transmit  it to the server.
\item Server: Aggregate the updated  $\bsbxi_k ^{(t+1)}$ from all clients into   $\bsbXi^{(t+1)}=[\bsbxi_1 ^{(t+1)}, \ldots, \bsbxi_m ^{(t+1)} ]$. Apply a proximity operation:   $$\bsbb^{(t+1)} = \prox_{\alpha_t\lambda\range^{(b)}}(\vect((\bsbXi^{(t+1)})^T)  ).$$ Update the global parameter matrix $\bsbB^{(t+1)} = [\bsbb_1^{(t+1)}, \ldots \bsbb_p^{(t+1)}]^T = [\bsb{b}_1^{(t+1)}, \allowbreak\ldots, \bsb{b}_m^{(t+1)}]$ (so that $\vect((\bsbB^{(t+1))^T})=\bsbb^{(t+1)} $). 
\end{itemize}
The dominant computation lies in the loss-gradient evaluations, which are fully distributed across clients and can be carried out in parallel. Unlike some methods such as FedProx \citep{LSZ20}, which use a client-side proximal term for algorithmic stabilization under a shared model, our penalty is designed for structural recovery. Thus the regularization step is handled centrally through the proximal operator, which is computationally lightweight (cf. Theorem \ref{thm:prox}) and requires no access to raw client data. 

  This proximal gradient algorithm typically exhibits  a linear convergence rate  assuming $L$ is $\mu$-strongly convex and $\mathcal L$-strongly smooth ($\mu\Breg_2\le \breg_L\le \mathcal L\Breg_2$). In the literature,   its   iteration complexity can be quantified as
\begin{align}
\mathcal O(\text{\customvarkappa}\log \frac{1}{\epsilon}), \label{itcomprate}
\end{align} where $\text{\customvarkappa}  =\mathcal L / \mu$ represents the condition number, and $\epsilon$ denotes the target accuracy level.  Moreover, this fast global convergence is often maintained in high-dimensional problems that demonstrate restricted strong convexity,  as explored in references such as \cite{Agarwal2012}, \cite{Loh2015}, \cite{Shebregman2021}.
 This is particularly relevant when employing a group range regularizer to capture the structural parsimony of  the problems under consideration, exemplified by      Theorem \ref{thm:seesawerr}.

Reducing communication costs is a critical challenge in federated learning, and various strategies have been developed. In the literature, methods such as model compression---through techniques like quantization \citep{reisizadeh2020fedpaq}, pruning \citep{jiang2022model}, and delta encoding  \citep{sattler2021cfd}---help to minimize the data transmitted between clients and the central server. In addition, Federated Averaging leverages multiple local updates before synchronization with the central server \citep{mishchenko2022proxskip}, and client selection through partial participation optimizes resource usage across the network \citep{nishio2019client}. All these methods are well-established and can be seamlessly integrated into range-regularized federated learning.

Accordingly, we  focus on reducing communication cost through lower  \emph{iteration complexity}, which is essential for making the new method practical and deployable. Our strategy is particularly motivated by the theoretical implications of the (unknown) restricted strong convexity parameter in the results of Section \ref{subsec:errrate},   suggesting room for algorithmic gains.
 To this end, we develop a momentum-based acceleration scheme that strategically leverages the \emph{varying} degrees of restricted strong convexity, as characterized by $\{\mu_t\}$. The resulting  all-in-one scheme   handles        both          $\mu_t >0$ and   $\mu_t  = 0$ in a unified manner    and adapts dynamically    as $\mu_t$     varies across  iterations. 
\\

Given two {arbitrary} sequences   $\rho_t$  and
  $\mu_t$  ($t\ge 0$), which are positive and nonnegative, respectively, define a sequence of relaxation parameters $\theta_t$ starting with $\theta_0\in (0,1]$:
\begin{align}
\frac{\theta_t^2}{1-\theta_t} = \frac{\theta_{t-1}({ \rho_{t-1}}\theta_{t-1} + {\mu_{t-1}})}{\rho_t}, \quad \forall t\ge 1.\label{acc2_search2}
\end{align}
Clearly,   $\theta_t\in (0, 1)$, for all $t\ge 1$.
In the typical convex optimization scenario   where   $\rho_t\equiv \mathcal L $, $\mu_t\equiv 0$, \eqref{acc2_search2} simplifies to
$\frac{\theta_t^2}{1-\theta_t} =  \theta_{t-1}^2$, or more explicitly,
\begin{align}
\theta_{t} = \frac{1}{2}(\sqrt{\theta_{t-1}^4 + 4\theta_{t-1}^2} - \theta_{t-1}^2), \quad \forall t\ge 1 \label{thetaupd:nest}
\end{align}
which aligns with the classical Nesterov's  construction \citep{Nesterov1988}. However,  \eqref{acc2_search2} not only expands on \eqref{thetaupd:nest} by incorporating nontrivial smoothness and convexity parameters, but  accommodates {nonconstant}   $\{\rho_t\}$ and $\{\mu_t\}$. This is valuable  when there exist some positive \emph{local} restricted strong convexity parameters $\mu_t$  (though  their values are often theoretically difficult to determine).

We define an iterative algorithm that operates on three sequences of iterates. The
update formulas, starting   from the initial condition  $\bsb{\alpha}^{(0)}=\bsb{\beta}^{(0)}$, utilize the current values of   $\bsb{\alpha}^{(t)},\bsb{\beta}^{(t)} $, along with parameters $\theta_t, \rho_t, \mu_t$,  for   $t \ge 0$:
\begin{align}
&\bsb{\gamma}^{(t)}   =  (1-\theta_t)\bsb{\beta}^{(t)} + \theta_t\bsb{\alpha}^{(t)}, \label{acc2-alg1}
\\
&\bsb{\alpha}^{(t+1)}   =  \mathop{{\arg}{\min}}_\bsbb f(\bsb{\beta}) {-} \breg_{L}(\bsb{\beta},\bsb{\gamma}^{(t)}) {+} {\mu_t \Breg_{2}(\bsb{\beta},\bsb{\gamma}^{(t)})}{+} \theta_t\rho_t\Breg_2(\bsb{\beta},\bsb{\alpha}^{(t)})\notag\\
&\qquad\ \ = \prox_{\frac{\lambda}{\mu_t+\theta_t\rho_t} \range^{(b)}}\,(\frac{\mu_t\bsbg^{(t)}+\theta_t\rho_t\bsba^{(t)}-\nabla L(\bsbg^{(t)})}{\mu_t+\theta_t\rho_t} ),
\label{acc2-alg2}\\
&\bsb{\beta}^{(t+1)}  = (1-\theta_t)\bsb{\beta}^{(t)} + \theta_t\bsb{\alpha}^{(t+1)}. \label{acc2-alg3}
\end{align}
One   novel feature   of the update  is the addition   of \( \mu_t \Breg_{2}(\bsb{\beta}, \bsb{\gamma}^{(t)}) \) in \eqref{acc2-alg2}, which complements the traditional Bregman term
\( \Breg_2(\bsb{\beta}, \bsb{\alpha}^{(t)}) \)      when    \( \mu_t \ne 0\).

Next, we present a convergence theorem. Given any function $\psi(\cdot)$, define  $\mathbf C_\psi(\bsb{\alpha},\bsb{\beta},\theta): = \theta \psi(\bsb{\alpha}) + (1-\theta)\psi(\bsb{\beta}) - \psi(\theta\bsb{\alpha}+(1-\theta)\bsb{\beta})$ and use $\mathbf C_2$ to denote   $\mathbf C_{\|\cdot\|_2^2/2}$.
 Let
\begin{align}
 \bar \psi_t(\cdot) = L(\cdot) - \mu_t\|\cdot\|_2^2/2
\end{align} and introduce
\begin{align}
& \mathcal E_t(\bsb\beta) = \breg_{\bar \psi_t}(\bsb{\beta},\bsb{\gamma}^{(t)}) + \lambda\breg_{\range^{(b)}}(\bsb\beta,\bsb\alpha^{(t+1)}) \label{Etdef}
\end{align}
and
\begin{align}
& R_t   =\theta_t^2\rho_t\Breg_2(\bsb{\alpha}^{(t+1)},\bsb{\alpha}^{(t)}) - \breg_{\bar\psi_t}(\bsb{\beta}^{(t+1)},\bsb{\gamma}^{(t)}) + (1-\theta_t)\breg_{\bar \psi_t}(\bsb{\beta}^{(t)},\bsb{\gamma}^{(t)})  + (\lambda\mathbf C_{ \range^{(b)} }+  \mu_t \mathbf C_2)(\bsb\alpha^{(t+1)},\bsb\beta^{(t)},\theta_t). \label{Rtdef}
\end{align}
  $R_t$ is defined using the iterates, while $\mathcal E_t(\bsb\beta)$ also  depends on $\bsbb$. If the loss $L$
exhibits     (global) convexity  or   strong convexity, satisfying  $\breg_L\ge \mu
\Breg_2$    for some $\mu\ge 0$, then setting $\mu_t=\mu$ ensures $\mathcal E_t(\bsb\beta)\ge 0$ for all $\bsbb$. (Of course, since our focus typically centers on a minimizer or the statistical truth, local convexity is often sufficient for this purpose.)
\begin{theorem} \label{th:acc}
Assume the loss function $L$ is differentiable (possibly nonconvex). For \emph{arbitrary} sequences  $\rho_t>0$ and
$\mu_t\ge 0$, the iterates defined in \eqref{acc2-alg1}, \eqref{acc2-alg2}, \eqref{acc2-alg3} satisfy the following bound   for    any $\bsbb$ and all $T\ge 0$:
\begin{align}
& f(\bsb{\beta}^{(T+1)})-f(\bsb{\beta} )   +\theta_T^2(\rho_T +\frac{\mu_T}{\theta_T}) \Breg_2(\bsb{\beta},\bsb{\alpha}^{(T+1)}) + {\mbox{\large$\Sigma$}}_{t=0}^T \big(\Pi_{s=t+1}^T (1-\theta_s)\big) (R_t + \theta_t \mathcal E_t(\bsb{\beta}))\notag\\
\le\,& \Big ({\mbox{\small$\prod$}}_{t=1}^T(1-\theta_t)\Big)\big[(1-\theta_0)(f(\bsb{\beta}^{(0)})-f(\bsb{\beta}))+\theta_0^2\rho_0 \Breg_2(\bsb{\beta},\bsb{\beta}^{(0)})\big],
\label{accerrbnd}
\end{align}
where by convention  $\prod_{s = l}^u a_s=1$ as $l>u$.
\end{theorem}

A key advantage of Theorem \ref{th:acc} is its flexibility: it places no essential restriction on   $\{\rho_t\}$ or $\{ \mu_t \}$,  and in particular allows $\mu_t$ to vanish at some iterations. Thus, the result adapts to varying local curvature, covering both locally non-strongly convex and locally strongly convex regimes within a single framework.  This is especially appealing in practice, since the local curvature  may vary over iterations and is typically unknown in advance. The result is also established at the level of a general convex regularizer, so its scope is not tied to a specific range penalty form. Moreover, it requires no regularity conditions.

When $L$ satisfies $\mu\Breg_2\le \breg_L\le \mathcal L\Breg_2$ with $0<\mu\le \mathcal L$,  taking   $\mu_{t} \le \mu$, and $\rho_t \ge \mathcal L -\mu_t$ ensures $$\mathcal E_t(\bsb\beta)=  \breg_{\bar \psi_t}(\bsb{\beta},\bsb{\gamma}^{(t)}) + \lambda\breg_{\range^{(b)}}(\bsb\beta,\bsb\alpha^{(t+1)}) \ge  (\breg_{L}-\mu_t\Breg_2)(\bsb{\beta},\bsb{\gamma}^{(t)}) \ge 0$$ and
\begin{align*}
R_t &\ge\theta_t^2\rho_t\Breg_2(\bsb{\alpha}^{(t+1)},\bsb{\alpha}^{(t)}) - \breg_{\bar\psi_t}(\bsb{\beta}^{(t+1)},\bsb{\gamma}^{(t)}) + (1-\theta_t)\breg_{\bar \psi_t}(\bsb{\beta}^{(t)},\bsb{\gamma}^{(t)}) \\
&\ge\theta_t^2\rho_t\Breg_2(\bsb{\alpha}^{(t+1)},\bsb{\alpha}^{(t)}) -(\mathcal L-\mu_t) \Breg_{2}(\bsb{\beta}^{(t+1)},\bsb{\gamma}^{(t)})   \\
 & \ge   \theta_t^2 (\rho_t+\mu_t-\mathcal L)  \Breg_2(\bsb{\alpha}^{(t+1)},\bsb{\alpha}^{(t)})\ge 0,
\end{align*} where the third inequality uses \eqref{acc2-alg1} and  \eqref{acc2-alg3}.

On the other hand, under the conditions
$$
\frac{\mu_{t-1}}{\rho_t}\ge \frac{1}{r}, \quad  \frac{\rho_{t-1}}{\rho_t} \ge 1 - \frac{c}{\sqrt r}$$ for some $r\ge 1$ and  constant $c\in [0, 1]$,   induction reveals that the relaxation parameters determined by   \eqref{acc2_search2} satisfy (details omitted)
\begin{align*}
\frac{1}{\theta_t } &\le     \frac{ 1}{2  }(\sqrt {(c \sqrt r+1)^2+4(1-\frac{c}{ \sqrt r})r} + c \sqrt r+1) = \frac{1}{2}(\sqrt{(c \sqrt r-1)^2 + 4r} +c \sqrt r + 1),
\end{align*}
 provided that $\theta_{ 0}$ is chosen to  conform to the same bound (such as $\theta_0=1$). Hence the product of contraction factors ${\mbox{\small$\prod$}}_{t=1}^T(1-\theta_t)$ on the right-hand side of \eqref{accerrbnd} results in     an iteration complexity of $\mathcal O(\sqrt r \log \frac{ 1}{\epsilon})$.

As an example, employing   constant  $\mu_t\equiv \mu, \rho_t \equiv \mathcal L - \mu$ leads to   $r =   (\mathcal L - \mu)/\mu   =  \text{\customvarkappa}-1$ and   $\theta_t \equiv   {2}/({\sqrt{4\text{\customvarkappa}-3} + 1})  \, (\forall t\ge 0)$  in accordance with \eqref{acc2_search2}, as well as  $\mathcal E_t(\bsbb)\ge 0$ and $R_t\ge 0$.  Setting  $\bsbb$  as a minimizer in \eqref{accerrbnd},  both  $f(\bsb{\beta}^{(T+1)})-f(\bsb{\beta} ) $ and $ \Breg_2(\bsb{\beta} ,\bsb{\alpha}^{(T+1)})$  converge to zero geometrically fast with a convergent  factor of  $ (\frac{\sqrt{4 \text{\customvarkappa} -3} - 1}{\sqrt{4 \text{\customvarkappa} -3} + 1})^T$, or an iteration complexity of
\begin{align}
\mathcal O(\sqrt \text{\customvarkappa} \log \frac{ 1}{\epsilon}). \label{accitcomprate}
\end{align}
which aligns with the literature (cf. \cite{Tseng2010} for example).
\eqref{accitcomprate} achieves  a significantly faster convergence compared with  \eqref{itcomprate} for standard  proximal gradient descent.

The true effectiveness of the algorithm and theorem becomes apparent when handling       iteration-varying  $\mu_t$, a common scenario in our applications. For simplicity, suppose that $\mathcal{L}$ is known, as is the case in our applications. This assumption allows us to set $\rho_t$ as a constant, such as $\mathcal{L} - \mu$ or $\mathcal{L}$ (which ultimately leads to reduced communication). In this setup,   the dynamic updates of   $\theta_t$ and $\bsbg^{(t)}$ are straightforward,   whereas   $\bsba^{(t+1)}$ and $\bsbb^{(t+1)}$ rely  on $\mu_t$, the value of which can be  determined through  a line search   according to Theorem \ref{th:acc}.       For example,   one can adjust     $\mu_t$ based on $\mu_{t-1}$ to maximize   $R_t+ \theta_t \mathcal E_t(\bsbb)$, with $\bsbb$     the iterate yielding the lowest   $f$-value,     as indicated in \eqref{accerrbnd}. A more effective search criterion is described in Appendix \ref{subsec:simus}. 

Even when $f$ is known to be convex or strongly convex with a given   $\mu\ge 0$  (which can serve as a reference for $\mu_t$),  searching for an optimal   $\mu_t$  often proves valuable    in reducing iteration complexity.   Indeed, unless $\lambda$ is   very small,   the structural parsimony of our problem enables the use of a local, iteration-specific     $\mu_t$  that exceeds   $\mu$ (although  deriving its value accurately can be challenging).  Of course, it is advisable to limit the number of such searches per iteration; additional details are provided in Appendix \ref{subsec:simus}. 

Finally, we discuss   the implementation of the accelerated algorithm within a federated learning framework.
From \eqref{acc2_search2} to \eqref{acc2-alg3}, the only operation  involving sensitive client data is   \eqref{acc2-alg2}, where      $\nabla L(\bsbg^{(t)})$ can be aggregated as in \eqref{gradFed} to preserve privacy.
 Moreover, due to      $R_t\ge R_t'$ where
\begin{align*}
R_t'  :=\, &\theta_t^2(\rho_t+\mu_t- \mathcal L)\Breg_2(\bsb{\alpha}^{(t+1)},\bsb{\alpha}^{(t)})  + (1-\theta_t)\breg_{\bar \psi_t}(\bsb{\beta}^{(t)},\bsb{\gamma}^{(t)})
 + (\lambda\mathbf C_{ \range^{(b)} }+  \mu_t \mathbf C_2)(\bsb\alpha^{(t+1)},\bsb\beta^{(t)},\theta_t),
\end{align*}
 the server can conduct  a $\mu_t$-line search  requiring only three   function evaluations: $L(\bsbg^{(t)})$, $L(\bsbb^{(t)})$, and  $L(\bsbb)$. (Notably,    neither $R_t'$ nor $\mathcal E_t$ involves the trial vectors $\bsba^{(t+1)}, \bsbb^{(t+1)}$  through their loss function values or gradients.)     $L(\bsbb)$  can be determined by the server using historical values from  $ L(\bsbg^{(s)}), 1\le s\le t$. Consequently,  after the server updates $\theta_t$ and $\bsbg^{(t)}$,  each client is required to transmit  a gradient vector and a   function value  pertaining to its own data, which are then   aggregated to form $\nabla L(\bsbg^{(t)})$,  $L(\bsbg^{(t)})$. 

In summary, compared with standard proximal gradient descent, the accelerated algorithm adds only minimal exchange overhead: each client needs to send just one additional scalar beyond the $p$-dimensional gradient vector.  Once these quantities are collected, neither the generation of trial iterates  $\bsba^{(t+1)}, \bsbb^{(t+1)}$ nor the evaluation of the search criterion  requires further communications between the server and clients. The accelerated algorithm offers a significant advantage in terms of iteration count, thanks to its substantially lower rate parameter compared to standard proximal gradient descent. This  is especially beneficial in high-dimensional problems where restricted strong convexity is present. For implementation details and experimental results,
refer to Appendix \ref{subsec:simus}. 

\section{Summary}
\label{sec:summ}
As modern practice reshapes statistical methods around deployment constraints, the resulting methodological shift deserves its own statistical analysis. Range regularization adaptively narrows the range of weights or coefficients \emph{without} shrinking intermediate values, thereby offering benefits for communication, quantization, energy efficiency, regularization, and privacy in machine learning tasks. In federated learning, group range penalization not only facilitates the identification of shared parameters and the automatic polar clustering of personal parameters, but also admits a fast computable proximal mapping, supports efficient quantization and coding, and improves interpretability and statistical accuracy.

However, the nonasymptotic analysis of the range seminorm presents significant challenges due to its nondecomposable nature. For example, we showed that for a group-range regularizer,      $   \sum_{j\in \mathcal J(\bsbb)} (\tau^2(\bsbb_j))=\sum_{j\in \mathcal J(\bsbb)} (\frac{1}{\overline M(\bsbb_j)}+\frac{1}{\underline M(\bsbb_j)})$ serves as an analogue to the `support size'. Although $\range$ is subadditive, $\tau(\cdot)$ is not, in contrast to measures such as cardinality or rank used in variable selection and low-rank matrix estimation.
 Indeed, traditional approaches based on H\"older  or dual (semi)norms, when combined with union bounds or chaining, overlook the intricate structures of both the estimator and the statistical truth, and thus lead to the suboptimal parameter choice of $\sigma m \sqrt{\log m + \log p}$. Furthermore, establishing a less stringent cone restriction goes beyond simple index splitting or space decomposition, as used for the group lasso or nuclear norm, and instead requires nonlinear arguments that exploit the weak differentiability of the range seminorm. The same proof device also covers the simpler $\ell_{\infty}$-norm penalization, where the corresponding $\tau$ measure is  $\tau(\bsbb) =({M_{\infty}(\bsbb)})^{-1/2}$ with $M_{\infty}(\bsbb) =|\{j: |\beta_j| = \|\bsbb\|_\infty\}| $.

This paper introduces a novel  seesaw  proof device that differs from traditional dual approaches. First, an orthogonal decomposition of the stochastic term into a star component, associated with the true model's degrees of freedom, and a hat component, associated with the estimator in the orthogonal complement subspace, effectively addresses the nonsubadditivity of $\tau(\cdot)$  and allows for weaker regularity conditions. Next, by integrating statistical and optimization analysis, we bound both the hat component and the model support size in a seesaw manner. Rather than relying solely on the basic inequality common in analyses of regularized M-estimators, this integration leads to sharp error rates and principled regularization choices, while providing new insight into the high-dimensional analysis of nondecomposable regularizers.

The structural parsimony induced by the group range regularizer suggests that exploiting restricted strong convexity can substantially reduce the number of iterations in federated learning, even in high-dimensional applications. Yet identifying the relevant convexity parameter(s) precisely is challenging, especially since only local convexity is needed along the iterates. The recursive bound in Theorem \ref{th:acc}, valid for any   $\mu_t$,  whether positive or zero, provides a basis for a versatile all-in-one acceleration scheme that dynamically adjusts convexity parameters to improve convergence speed.

While the present paper focuses on linear systematic components, extending the range-regularization framework to structured nonlinear function classes is a natural direction for future work. Another theoretical direction is to study the behavior of the framework under heavy-tailed errors and heterogeneous noise scales across clients. Such developments would further expand the scope and practical value of range regularization in distributed learning.

\numberwithin{equation}{section}
\numberwithin{theorem}{section}
\numberwithin{definition}{section}
\numberwithin{table}{section}
\numberwithin{figure}{section}

\appendices
\renewcommand\thesubsection{\thesection.\arabic{subsection}}
\renewcommand\thesubsectiondis{\thesection.\Roman{subsection}}
\section{More Technical Details}
\label{sec:tech}
\subsection{Theorem \ref{th:minimaxclut}}
\label{subsec:minimax}
Recall the notations $\bsby, \tilde \bsbX, \bsbb$ in the federated learning context  as described  in Section \ref{sec:flsetting}.
In this part, we use $q(\bsba)$ to denote the number of clusters of $\bsba$. (This quantity differs slightly from the
$q$ in Section \ref{sec:nonasymstat}, which does not account for clusters formed by intermediate values.) We establish a minimax lower bound for recovering the clustering structure of each $\bsbb_j$. Let $I(\cdot)$ be an arbitrary nondecreasing function with $I(0)=0, I\not\equiv 0$; some particular examples are $I(t) = t$ and $I(t) = 1_{t\ge c}$.\begin{theorem} \label{th:minimaxclut}
        Assume $\bsby  =  \tilde \bsbX \bsbb^* + \bsbeps$, where   $\bsbeps=[\epsilon_i]$ with $\epsilon_i \overset{i.i.d.}\sim \mathcal N(0,\sigma^2)$ and $p, m\ge 2$.
Given $2\le q\le m, 1\le J \le p$, define a signal class by    $$\bsbb^*\in \mathcal S(q, J)= \{\bsbb=[\bsbb_1^T, \ldots, \bsbb_p^T]^T \in \mathbb R^{p  m}:   q(\bsbb_j)\leq q, \  |\{j: q(\bsbb_j)> 1\}| \le  J \}.$$  Let  \begin{align}
        P (q, J) = qJ + p  +   \frac{Jm\log q}{  q^2 } +  J\log(\frac{ep}{J}) . \label{specailminimaxrate-complete}
        \end{align}
 Assume that  $  \| \tilde \bsbX  \bsbb  \|_2^2 \le
       \overline \kappa \|    \bsbb\|_2^2$ for all $\bsbb  \in   \mathcal S(q, J)$.  Then  there exist positive constants $c, c'$, depending on $I(\cdot)$ only,  such that
        \begin{align}
        \inf_{\hat \bsbb} \sup_{\bsbb^* \in \mathcal S(q, J)} \EE\big[I\big(\|   \hat\bsbb-  \bsbb^* \|_2^2{\big/}\{c \sigma^2     P (q, J)/ \overline\kappa\}\big)\big]  \geq c' >0, \label{estgeneralminimaxrate-complete}
        \end{align}
\end{theorem}
We can also derive a  minimax rate for the prediction error $\|\tilde \bsbX\hat \bsbb -\tilde \bsbX \bsbb^*\|_2^2$, and  the Gaussian assumption can be relaxed to cover the exponential family  (cf.  Theorem A.1 in \cite{SheetalCRL22}).
We omit the details.

The rate of $\overline \kappa$ typically is of order $\mathcal O(n)$. At one extreme, a globally shared federated model corresponds to   $J = 0$. At the other extreme, in a fully personalized setting with       $J= p$ and a constant
   $q$,   an idealized clustered personalized federated scenario,     $P(q,J)$    is of order $\mathcal O(mp) $. Thus, even when each personalized row has only a small number of clusters, the overall complexity remains linear in both $m$ and $p$. In other words, somewhat surprisingly, within-row clustering across clients does not automatically translate into an order-level improvement in statistical accuracy.

In a partially personalized regime, a more parsimonious assumption is that only $J$ rows are personalized, with $q$ kept small ($q\asymp 1$). In that case,  $P(q,J)$ is at least
$$
   p  +  Jm +  J\log(\frac{ep}{J})
$$
up to a multiplicative constant, which matches the rate discussed in Remark \ref{rem:errrate}. Notably, the second term remains of order $Jm$, rather than $q J$ or a constant multiple of $J$ for constant  $q$, reflecting the statistical cost of identifying the  clustering structure.   To the best of our knowledge, no existing theory establishes a matching upper bound for pairwise-difference penalization. Hence, methods based on sparsifying pairwise differences do not currently enjoy a minimax-optimal statistical guarantee.

In addition,  \eqref{specailminimaxrate-complete} suggests that, if one could optimize over the true clustering structure, then balancing the first and third terms would lead to the choice  $q\asymp m^{\frac{1}{3}}$, yielding the rate  $J m^{\frac{1}{3}} \log m + p + J\log(ep/J)$. However, such a choice is not operational, since the true clustering structure is unknown (often assumed to be small) and cannot be controlled in practice.

\begin{proof}
        We first introduce a Gilbert-Varshamov bound for
$q$-ary codes, which can be derived from the
$q$-ary entropy function via a greedy algorithm \cite{SheetalCRL22}.

        \begin{lemma}[Lemma A.2, \cite{SheetalCRL22}]
                Let $\Omega = \{\bsba= [a_1, \ldots, a_m]: a_j \in \mathcal A  \}$, where   $|\mathcal A| = q$, $m\ge q > 1 $. Then there exists a subset $ \{\bsba^0, \ldots, \bsba^M\} \subset \Omega$ such that $\bsba^{0}\in \mathcal A^m$ is arbitrarily chosen,      and
                \begin{align}
                &\log M \ge  c_1 m\log q , \\&  \rho (\bsba^{j}, \bsba^{k}) \ge c_2 m , \ \forall 0\le j < k  \le M,
                \end{align}
                where $\rho(\bsba, \bsba') = \| \bsba - \bsba'\|_0$, and $c_1, c_2$ are universal positive constants.\label{qarysepa}
        \end{lemma}

        \textit{Case (i):} $     (Jm\log q)/ q^2 \gtrsim P(q, J) $.
        Define a set $\mathcal C$   $$\mathcal   C=\{ 0, 1, \ldots, q-1\}.$$  Construct   \begin{align*}
        {\mathcal B}^1(q, J)   =\{  \bsbb &  =[\bsbb_{1}, \ldots, \bsbb_p]^T: \bsbb_j  / ( \gamma R) \in \mathcal C^m, 1\le j \le J, \  \bsbb_j= \bsb0, J+1\le j \le p\},
        \end{align*}
        where  $\gamma>0$ is a small constant to be chosen later, and $$R=   \sigma ( \log q )^{1/2}/( q^2 \overline{ \kappa })^{1/2}.      $$
        Clearly,  ${\mathcal B}^{1}(q,J) \subset \mathcal {S}(q,J) $.

Define the Hamming distance function by  $$\rho(\bsbb, \bsbb') = \sum_{j=1}^p \sum_{k=1}^m 1_{\bsbb_j[k]\ne \bsbb'_j[k]}.$$
        Applying  Lemma \ref{qarysepa} to $[\bsbb_j]_{1\le j \le J}$ (which satisfies   $Jm\ge q \ge 2$),  we can find  a subset ${\mathcal B}^{10}(q, J)\subset {\mathcal B}^{1}(q,J)$ such that $\bsb0\in {\mathcal B}^{10}(q, J)$ and
\begin{align}
                &\log( | {\mathcal B}^{10}(q, J)| -1) \ge c_1 J m \log q,
               \label{case1card}  \\
                &\rho(\bsbb, \bsbb') \geq c_2 Jm, \forall \bsbb, \bsbb' \in \mathcal B^{10},  \bsbb\neq \bsbb'                \label{case1sep}
\end{align}
        for some universal constants $c_1, c_2>0$.
Hence
\begin{align}
\| \bsbb - \bsbb'\|_2^2 \ge   \rho(\bsbb, \bsbb') \cdot 1 \cdot \gamma^2 R^2\geq   c_2     \gamma^2 R^2 Jm, \ \forall \bsbb, \bsbb' \in \mathcal B^{10},   \bsbb\neq \bsbb'. \label{separationLBoundcase1}
\end{align}


Let $P_{\tilde\bsbX\bsbb}$ denote the distribution of $\mathcal N(\tilde\bsbX\bsbb, \sigma^2 \bsbI)$.       Then the Kullback-Leibler divergence of   $P_{\tilde\bsbX\bsbb'} $ from  $P_{\tilde \bsbX\bsbb} $ satisfies $\mbox{KL}(  P_{\tilde \bsbX\bsbb} \|    P_{\tilde \bsbX\bsbb'})   = \Breg_{2} (  \tilde \bsbX  \bsbb',  \tilde \bsbX  \bsbb)/\sigma^2$. It follows from the regularity condition that for any $\bsbb \in {\mathcal B}^{10}(q, J)$,
        \begin{align*}
&        \mbox{KL} (P_{\tilde \bsbX\bsbb}\| P_{\bsb{0}})= \frac{1}{\sigma^2} \Breg_2 ( \tilde\bsbX \bsbb, \bsb0) \le \frac{\overline{\kappa}}{\sigma^2} \Breg_2 ( \bsb0,\bsbb) \leq \frac{1}{2\sigma^2}\overline{\kappa}    (q-1)^{2}\gamma^2 R^2 \rho(\bsb{0}, \bsbb)  \leq \frac{q^{2}\gamma^2}{\sigma^2}   \overline{\kappa}R^2 Jm.
        \end{align*}
        Therefore,
        \begin{align}
        \frac{1}{|\mathcal B^{10}|-1}\sum_{\bsbb\in \mathcal B^{10}\setminus \{\bsb0\}}   \mbox{KL} (  P_{\tilde\bsbX\bsbb}\|P_{\bsb{0}}) \leq    \gamma^2  Jm \log q. \label{KLUBoundcase1}
        \end{align}
 which  aligns with \eqref{case1card}.
         Choosing a sufficiently small value for  $\gamma$, we can apply Theorem 2.7 of \cite{tsybakov2009introduction} to get the    desired lower bound from   \eqref{separationLBoundcase1}.

\textit{Case (ii):} $qJ    \gtrsim P(q, J)  $. Consider a signal  subclass
        \begin{align*}
        {\mathcal B}^2(q,J)=\{\bsbb:   &    \ \bsbb_{j}[k]\in\{0 , \gamma R \}  \mbox{ if }  1\le j \le J, 1\le k \le q, \mbox{ and } 0 \mbox{ otherwise}\},
        \end{align*}
where $R=  {\sigma}/{{\overline{\kappa}^{1/2}}  } $ and $\gamma>0$ is a sufficiently small constant.
         Hence      $|{\mathcal B}^2 (q,J)|= 2^{Jq}$. By  Lemma \ref{qarysepa} and
 $Jq\ge 2$\textgreater1,        there exists  a subset ${\mathcal B}^{20}(q,J)\subset {\mathcal B}^{2}(q,J)$ such that $\bsb0\in {\mathcal B}^{20},$
        \begin{align*}
               & \log (| {\mathcal B}^{20}(q,J)|-1) \geq c_1 Jq,\\
               &                \rho(\bsbb, \bsbb') \geq c_2 Jq, \forall \bsbb, \bsbb' \in \mathcal B^{20},  \bsbb\neq \bsbb'
        \end{align*}
        for some universal constants $c_1, c_2>0$.
        Then  for any $\bsbb, \bsbb' \in \mathcal B^{20}:  \bsbb\neq \bsbb'$ $$\| \bsbb  -  \bsbb'  \|_2^2 \ge   \gamma^2 R^2 \rho(\bsbb, \bsbb') \geq c_2   \sigma^2\gamma^2  Jq/\overline{\kappa}.$$ Furthermore,
        for any $\bsbb\in \mathcal B^{20}(q,J)$, we have
        \begin{align*}
        \mbox{KL} (P_{\tilde \bsbX\bsbb}\| P_{\bsb{0}})  \leq \frac{1}{2\sigma^2}\overline{\kappa}   \gamma^2 R^2 \rho(\bsb{0}, \bsbb) \leq  {\gamma^2}    Jq.
        \end{align*}
        The afterward treatment follows the 
        same lines as in (i) and the details are omitted.

\textit{Case (iii):} $p\gtrsim P(q, J) $. Construct
        \begin{align*}
        {\mathcal B}^3(q,J)=\{\bsbb:     \bsbb_{j}=\bsb0 \mbox{ or } \gamma R \cdot \bsb1, 1\le j \le p\},
        \end{align*}
where $R=  {\sigma}/{{\overline{\kappa}^{1/2}}  } $.     In this case we define $\rho(\bsbb, \bsbb') = \sum_{j=1}^p 1_{\bsbb_j[1]\ne \bsbb_j'[1]}$.
The argument is similar to case (ii) and thus omitted.

\textit{Case (iv):} $ J\log\frac{ep}{J} \gtrsim P(q, J)  $. Construct
        \begin{align*}
        {\mathcal B}^4(q,J)=\{\bsbb:\ &    \bsbb_{j}[1]= 0 \mbox{ or } \gamma R , 1\le j \le p, \  | \{j: \bsbb_{j}[1 ]\ne 0\} | \le J, \ \bsbb_{j}[k] = 0, 1\le j \le p, 2\le k \le m\},
        \end{align*}
where $R=  {\sigma}\{\log(ep/J)/{{\overline{\kappa}\}^{1/2}}  } $ and $\gamma$ is a small constant to be chosen later. Clearly, ${\mathcal B}^4(q,J)\subset \mathcal {S}(q,J) $. By Stirling's approximation, $\log | \mathcal B^4|\ge \log {p\choose J}\ge J \log (p/J)\ge c J \log (ep/J) $ for some universal constant $c>0$.     Here we use   $\rho(\bsbb, \bsbb') = \sum_{j=1}^p 1_{\bsbb_j[1]\ne \bsbb_j'[1]}$.

By Lemma A.3 in   \cite{Rigollet11}, there exists a subset $\mathcal B^{40} \subset \mathcal B^4$ such that $\bsb{0}\in  \mathcal B^{40}$ and
\begin{align*}
&\log |\mathcal B^{40}| \ge c_1 J\log(ep/J), \\
&\rho(\bsbb, \bsbb') \geq c_2 J, \forall \bsbb, \bsbb' \in \mathcal B^{40},  \bsbb\neq \bsbb'
\end{align*}
for some universal constants $c_1,c_2>0$.
Then \begin{align*}
\| \bsbb - \bsbb'\|_2^2 \ge   \rho(\bsbb, \bsbb') \gamma^2 R^2\geq   c_2      \gamma^2 R^2 J, \ \forall \bsbb, \bsbb' \in \mathcal B^{40},   \bsbb\neq \bsbb',
\end{align*}
and         \begin{align*}
        \mbox{KL} (P_{\tilde \bsbX\bsbb}\| P_{\bsb{0}})  \leq \frac{1}{2\sigma^2}\overline{\kappa}   \gamma^2 R^2 \rho(\bsb{0}, \bsbb) \leq  {\gamma^2}    J\log(ep/J).
        \end{align*}
        The subsequent argument follows the same reasoning as in case (i). The proof is complete.
\end{proof}
\subsection{Proof  of Theorem \ref{thm:semi}}

The seminorm claim is easy to verify noticing that $\max( \alpha_k + \beta_k) \le \max \alpha_k + \max \beta_k$ and
\begin{align}\label{range2max}
\range( \bsbb) = \max \beta_k + \max (-\beta_k).
\end{align} The kernel space of the range seminorm can be evaluated directly.

\subsection{Proof of  Theorem \ref{thm:subgrad}}
Using the max characterization \eqref{range2max}, the subgradient calculation   is straightforward and   omitted.

To verify the equality in \eqref{rangsubdiff}, we consider two cases: $m=1$, $m\ge 2$. The case $m=1$ is trivial. For $m\ge 2$,   given any  $\|\bsbs\|_1 = 2c \le 2$, we can construct $\bsbs_1 = [1-\epsilon, \epsilon, 0, \ldots, 0]^T$, $\bsbs_2 = [\epsilon, 1-\epsilon, 0, \ldots, 0]^T$ with $\epsilon = (1-c)/2$. The details are  omitted.

The second result can be easily obtained  using the connection between directional derivative  and subgradients: $
\delta \range  (\bsbg; \bsbb - \bsbg)  = \sup_{\bsbs\in \partial \| \bsbg\|_{\range}} \langle \bsbs , \bsbb - \bsbg\rangle$ (or the definition of one-sided directional derivatives).
The details are omitted.
\subsection{Proof of Theorem \ref{thm:prox}}

Let the maximum and minimum values of $\bsbb^o$    be represented   by $\overline \beta^o$ and $\underline\beta^o$ respectively. Abbreviate ${\overline{\mathcal M}(\bsbb^o)}$ to ${\overline{\mathcal M}}$, ${\underline{\mathcal M}(\bsbb^o)}$ to ${\underline{\mathcal M}}$, and  ${{\mathcal M}(\bsbb^o)}$ to ${{\mathcal M}}$. Let $\overline m = | \overline {\mathcal M}|$ and $\underline m = | \underline {\mathcal M}|$. Suppose  $\overline \beta^o> \underline \beta^o$. By Theorem \ref{thm:subgrad},
\begin{align}
&
\begin{cases}
\forall j \in  \overline{\mathcal M}: &\ \overline\beta^o  - y_j + s_j = 0, s_j\ge 0,    \sum_{  j \in  \overline{\mathcal M}} s_j =\lambda \\
  \forall j \in \mathcal M: & \
\beta^o_j - y_j   = 0,
\\
\forall j \in  \underline{\mathcal M}: &\ \underline\beta^o  - y_j - s_j = 0, s_j\ge 0,    \sum_{  j \in  \underline{\mathcal M}} s_j =\lambda,  \\
\end{cases} \label{kkt}
\\ \Longrightarrow &  \begin{cases}
 \overline m\cdot \overline\beta^o  =  \sum_{ j \in  \overline{\mathcal M}} y_j - \lambda, \overline\beta^o  \le  y_j, \ & \forall j \in  \overline{\mathcal M}  \\
  \beta^o_j = y_j   , \ & \forall j \in \mathcal M
\\
\underline m\cdot  \underline\beta^o  = \sum_{ j \in  \underline{\mathcal M}} y_j +\lambda , \underline\beta^o \ge y_j, & \forall j \in  \underline{\mathcal M}    \\
\end{cases}\label{kktsumform}
\end{align}
where $ \mathcal M$ can be empty. From \eqref{kkt} and \eqref{kktsumform}, the number of clusters in $\bsbb^o$ is nonincreasing in $\lambda$.

Recall the definitions of  $\bar{c}_{k}$ and $\ubar{c}_{k}$, which both increase with $k$, and thus $k_1$ and $k_2$ are well-defined increasing functions of $\lambda$. From  \eqref{kktsumform},
\begin{align}\lambda =     \sum_{  j \in  \overline{\mathcal M}} ( y_j - \overline\beta^o ) =      \sum_{  j: y_j \ge \overline \beta^o} ( y_j - \overline\beta^o ) . \label{lambdaexp1}\end{align}
Under  $\lambda \in  [\bar{c}_{   k_1}, \bar c_{k_1+1})$, we have $k_1 = \overline m$, $\mathcal I_1 = \overline {\mathcal M}$, and so
$$
k_1\cdot \overline\beta^o  =    \sum_{ j \in  \overline{\mathcal M}} y_j     - \lambda \in \Big(    \sum_{ j \in  \overline{\mathcal M}} y_j    -\bar c_{k_1+1},    \sum_{ j \in  \overline{\mathcal M}} y_j    -\bar{c}_{   k_1} \Big]
$$
But
\begin{align*}
  \sum_{ j \in  \overline{\mathcal M}} y_j     -\bar{c}_{   k_1} =  \sum_{ 1\le j \le  k_1} y_{(j)}     - \sum_{1\le j < k_1} y_{(j)} - y_{(k_1)} =y_{(k_1)}+ \sum_{1\le j < k_1} y_{(k_1)} = k_{1} y_{(k_1)}    \\
  \sum_{ j \in  \overline{\mathcal M}} y_j     -\bar{c}_{   k_1+1} =  \sum_{ 1\le j \le  k_1} y_{(j)}     - \sum_{1\le j < k_1+1} y_{(j)} - y_{(k_1+1)} = k_{1} y_{(k_1+1)}
\end{align*}
and so $  \overline\beta^o  \in(y_{(k_1+1)}, y_{(k_1)}] $.

Similarly, we have
\begin{align}\lambda =     \sum_{  j \in  \underline{\mathcal M}} (\underline\beta^o   - y_j ) =      \sum_{  j: y_j \le \overline \beta^o} (\underline\beta^o   - y_j ) , \label{lambdaexp2}\end{align}
and $k_{2}\cdot  \underline\beta^o  = \sum_{ j \in  \underline{\mathcal M}} y_j +\lambda $. Because
\begin{align*}
 \sum_{ j \in  \underline{\mathcal M}} y_j +\ubar{c}_{   k_2} =  \sum_{   m - k_2+1\le j \le m} y_{(j)}      +   \sum_{m-k_2 +1<  j \le m} y_{(m-k_{2}+1)} - y_{(j)} = k_2 y_{(m-k_{2}+1)} \\
 \sum_{ j \in  \underline{\mathcal M}} y_j +\ubar{c}_{   k_2+1} =  \sum_{   m - k_2+1\le j \le m} y_{(j)}      +   \sum_{m-k_2 <  j \le m} y_{(m-k_{2})} - y_{(j)} = k_2 y_{(m-k_{2})}
\end{align*}
we have $ \underline\beta^o \in [y_{(m-k_2+1)}, y_{(m-k_2)})$.

To get the cutoff value for uniform clustering, we study $  \underline\beta^o  <\overline\beta^o  $ or  $  \frac{1}{|\mathcal I_2|} \sum_{j\in \mathcal I_2} y_j +  \frac{1}{|\mathcal I_2|}\lambda< \frac{1}{|\mathcal I_1|} \sum_{j\in \mathcal I_1} y_j -\frac{1}{|\mathcal I_1|} \lambda  $, which, under $k_1 = |\mathcal I_1|$, $k_2 = |\mathcal I_2|$, $\mathcal I_1 \cap \mathcal I_2 = \emptyset$, indicates
\begin{align*}\lambda < &   \frac{k_1 k_2}{m}  \big(   \frac{1}{k_1} \sum_{j\in \mathcal I_1} y_j -\frac{1}{k_2} \sum_{j\in \mathcal I_2} y_j\big)
  =   \frac{1}{m}(k_2 \sum_{j\in \mathcal I_1} y_j -k_1 \sum_{j\in \mathcal I_2} y_j)
  =   \frac{1}{m}(k_2 \sum_{j\in \mathcal I_1} (y_j-\bar y) + k_1 \sum_{j\in \mathcal I_2} (\bar y -y_j)\\
    \le  &  \frac{1}{m}     \sum_{j \in [m]} k_{2}(y_j-\bar y)\vee 0  + k_1(\bar y -y_j)\vee 0
=   \frac{1}{m}     \sum  k_{2}\frac{|y_j-\bar y|+(y_j-\bar y)}{2}  + k_1 \frac{|y_j-\bar y|-(y_j-\bar y)}{2}  \\
= &  \frac{1}{m}\sum   (k_1+k_2)\frac{1}{2}  | y_j - \bar y| + (k_{2}  - k_{1}) \cdot 0
 \le  \sum   \frac{1}{2}  | y_j - \bar y|=\overline \lambda
\end{align*}
where the third equality uses $y_j\ge \bar y\ge y_{j'}, \forall j \in \mathcal I_1, j'\in \mathcal I_2$. Alternatively, we can use \eqref{lambdaexp1} and \eqref{lambdaexp2} to argue that    $\lambda \ge \overline \lambda$ is an \textbf{equivalent} condition for  $\bsbb^o$ to be uniform.  An alternative way is to notice that   $\bsbb^o\in \Proj_K\Leftrightarrow \langle \Proj_K^\perp \bsby, \bsbb\rangle \le \lambda \range(\bsbb)$ and so the dual seminorm expression of $\overline \lambda$ follows.

Suppose there exists a biclustering outcome and let $\lambda_0$ be the minimum value of $\lambda$ that produces  such a result. Abbreviate the associated $k_1(\lambda_0)$ and $k_2(\lambda_0)$  to $k_1, k_2$ which satisfy $k_1+k_2=m$.
We claim $k = k_1$ is the unique optimal solution to   \eqref{lambdabar1} and $\underline \lambda =\lambda_0$.

Assume $ \bar{c}_{   k_1} \ge \ubar{c}_{   k_2 } $ for now. For $k= k_1 + 1$, $ \bar{c}_{   k} \vee \ubar{c}_{   m - k } \ge \bar{c}_{   k} > \bar{c}_{   k_1} = \lambda_0 $, with the strict inequality due to $ [\bar{c}_{   k_1}, \bar c_{k_1+1})\ne \emptyset$. For $k= k_1 - 1$,   $\bar{c}_{   k} \vee \ubar{c}_{   m - k } \ge  \ubar{c}_{   m - k } = \ubar{c}_{   k_2+1 }>\bar{c}_{   k_1} = \lambda_0 $,  with the last inequality resulting from    $  [\bar{c}_{   k_1}, \bar c_{k_1+1})\cap  [\ubar{c}_{k_2 }, \ubar c_{k_2+1})\ne \emptyset $ or $\ubar c_{k_2+1} > \bar{c}_{   k_1}$.
The case of $ \bar{c}_{   k_1} \le \ubar{c}_{   k_2 } $  can be argued similarly.  Finally, using the definition in \eqref{lambdabar1}, where    $ \bar{c}_{   k}   $ is increasing and   $\ubar{c}_{   m - k  } $ is decreasing in $k$, we conclude that $\bar{c}_{   k} \vee \ubar{c}_{   m - k } > \lambda_0$ for any $k\ne k_1$.
\subsection{Proof of Theorem \ref{thm:dualnorm}}
\begin{proof}If $\langle 1, \bsba\rangle\ne 0$, we can set $\bsbb = c \langle \bsb1, \bsba\rangle\bsb1$ with $ c \rightarrow+\infty$ to get $\| \bsba \|_{\range^*} =+\infty$. Assume $\langle 1, \bsba\rangle= 0$. Recall $\Proj_{K} = \Proj_{\bsb1}$ and so $\bsba\in \Proj_K^\perp$. Then for any scalar   $c$ (which can be  possibly dependent on $\bsbb$),
\begin{align*}
\| \bsba\|_{\range^*} & = \sup_{\| \bsbb\|_{\range} \le 1} \langle \bsbb, \Proj_{K}^\perp \bsba\rangle = \sup_{\| \bsbb- c \bsb1\|_{\range} \le 1} \langle \Proj_{K}^\perp (\bsbb - c \bsb1), \Proj_{K}^\perp\bsba\rangle  \\ & =  \sup_{\| \bsbb- \frac{\max \beta_k + \min \beta_k}{2} \bsb1\|_{\range} \le 1} \langle\Proj_{K}^\perp (\bsbb - \frac{\max \beta_j + \min \beta_j}{2} \bsb1), \Proj_{K}^\perp\bsba\rangle \\
& =  \sup_{2\| \bsbg\|_{\infty} \le 1, \max \gamma_k = - \min \gamma_k} \langle\Proj_{K}^\perp\bsbg  , \Proj_{K}^\perp\bsba\rangle \\
 & =   \sup_{\| \bsbg\|_{\infty} \le 1/2, \max \gamma_k = - \min \gamma_k} \langle    \bsbg  ,  \bsba\rangle  \le  \sup_{\| \bsbg\|_{\infty} \le 1/2} \langle    \bsbg  ,  \bsba\rangle = \frac{1}{2} \|\bsba\|_1,
\end{align*}
where the equality is achievable because under  $\langle 1, \bsba\rangle= 0$, neither $\alpha_k> 0, \forall  k$  nor $\alpha_k< 0, \forall  k$ can occur  when applying H\"older's inequality.

 Similarly,  for  $
\range^* (   \bsba)  = \sup_{  \bsbb } \langle \bsbb, \bsba\rangle
 - \| \bsbb\|_{\range}$ with $\lambda=1$, when $\langle 1, \bsba\rangle\ne 0$, we get $+\infty$, and when   $\langle 1, \bsba\rangle= 0$, $\range^* (   \bsba)= \sup_{  \bsbb } \langle \bsbb - c(\bsbb) \bsb1, \bsba\rangle
 - \| \bsbb- c(\bsbb) \bsb1\|_{\range} =\sup_{  \bsbg: \max \gamma_k = - \min \gamma_k  } \langle \bsbg, \bsba\rangle
 - 2\| \bsbg\|_{ \infty}\le \sup_{  \bsbg } \langle \bsbg, \bsba\rangle  - 2\| \bsbg\|_{ \infty}= \iota_{\| \bsba\|_1\le2}$,   where, again, the equality is achievable due to  $\langle 1, \bsba\rangle= 0$.  The result can be trivially extended to a general $\lambda>0$. The case $\lambda=0$ is trivial.
\end{proof}
\subsection{Proof of Theorem  \ref{thm:tau}}


\begin{proof} First, notice that because $\range(\cdot)$ is convex, $0\le \breg_{\range}(\bsbzeta, \bsbzeta^0) = \range(  \bsbzeta)  - \range(\bsbzeta^0)    -  \delta\range (\bsbzeta^0;   \bsbzeta - \bsbzeta^0) $.

To prove the second inequality, the case of     $ \Proj_K^\perp\bsbzeta^0= \bsb0$ is trivial.  It suffices to prove for     $ \Proj_K^\perp\bsbzeta^0\ne \bsb0$. Define $\underline M ^0=| \underline {\mathcal M}(\bsbzeta^0)|$ and $\overline M ^0=| \overline {\mathcal M}(\bsbzeta^0)|$. From Theorem \ref{thm:subgrad}, \begin{align*}
-  \delta\|\cdot \|_{\range}(\bsbzeta^0;   \bsbzeta - \bsbzeta^0)& = -\max_{  \overline {\mathcal M}(\bsbzeta^0)}   \{   \zeta_k     -   \zeta_k^0\} +\min_{  \underline {\mathcal M}(\bsbzeta^0)}   \{  \zeta_k     -   \zeta_k^0\}  = \min_{  \overline {\mathcal M}(\bsbzeta^0)}   \{ \zeta_k^0-   \zeta_k      \} +\min_{  \underline {\mathcal M}(\bsbzeta^0)}   \{   \zeta_k     -   \zeta_k^0\}.
\end{align*}
Assume  $\zeta_k^0-  \zeta_k   = \alpha_k\ge 0,\forall k\in \overline {\mathcal M}(\bsbzeta^0)$ with $\alpha_1  \le \cdots\le  \alpha_{\overline {M}^0}$, and  $   \zeta_k     -   \zeta_k^0 = \gamma_k\ge 0, \forall k\in \underline {\mathcal M}(\bsbzeta^0)$   with $\gamma_1  \le \cdots\le  \gamma_{\underline {M}^0}$.
To show the desired inequality,  we need to minimize  the variance of  $   \bsbzeta - \bsbzeta^0$.

Given $\alpha_1, \gamma_1$, the minimum variance case occurs when   $\alpha_{\overline {M}^0}=\ldots = \alpha_1, \gamma_1 =\cdots =   \gamma_{\underline {M}^0} $ and the remaining  $\zeta_k - \zeta_k^0$ all   equal the mean, which is assumed to be 0 without loss of generality. Thus $\alpha_1 \overline M^0 = \gamma_1 \underline M^0$. Direct calculation shows
\begin{align*}
\frac{\alpha_1 + \gamma_1}{\|\bsbzeta\|_2} = \frac{1 +\overline M^0 /\underline M^0}{(\overline M^0 +(\overline M^0)^2 /\underline M^0 )^{1/2}}=\left( \frac{\overline M^0 +\underline M^0}{\overline M^0  \underline M^0}\right)^{1/2}.
\end{align*}
The conclusion follows.
\end{proof}

\subsection{Proofs of Theorem \ref{th:pattrecov} and  Corollary \ref{cor:pattrecovGaussian}}
\begin{proof} Due to the convexity of the problem,   $\hat \bsbb$ is a globally optimal solution if and only if \begin{align}
\tilde \bsbX  ^T\tilde \bsbX \hat \bsbb - \tilde \bsbX ^T\bsby + \lambda \bsbs = \bsb0, \label{kkteq}
\end{align}
where  $\bsbs = [\bsbs_1^T, \ldots, \bsbs_p^T]^T$,  $\bsbs_j \in \partial \range(\hat \bsbb_j)$. Note the dependence of $\bsbs$ on $\hat \bsbb$,  though  often omitted for brevity.
Define an event:
\begin{align}
{\overline{\mathcal M}} (\hat \bsbb_j) = \overline{\mathcal M}_j^*, {\mathcal M}(\hat \bsbb_j)=  {\mathcal M}_j^*, \underline{\mathcal M}(\hat \bsbb_j)=\underline{\mathcal M}_j^*, 1\le j\le p. \label{pattconsismeaning}
\end{align}
Then we can express $p_e$ as
\begin{align}
p_e = \EP[\,  \exists \hat \bsbb \text{ satisfying } \eqref{kkteq} \mbox{ and } \eqref{pattconsismeaning}\, ].
\end{align}

The following result is broadly applicable and can also be used for pattern recovery based on other norms (such as the $\ell_1$-type).
\begin{lemma}\label{lem:regkktsimp}
Suppose   $\bsbb^*,\hat \bsbb \in \Proj_\bsbU$ for some orthogonal $\bsbU$,  i.e., $\bsbb^* = \bsbU \bsbg^*$, and      $\hat \bsbb = \bsbU \hat \bsbg$, and  $\bsbZ  = \tilde \bsbX \bsbU$ has   full column rank. Let    $\bsbt = \bsbU^T \bsbs$. Then
\begin{align}
\hat \bsbb  &=\bsbb^* + \bsbU (\bsbZ^T \bsbZ)^{-1} \bsbZ^T \bsbeps - \lambda\bsbU (\bsbZ^T\bsbZ )^{-1}  \bsbt \label{lem:betahatsol}
\end{align}
and
\begin{align}
\begin{split}
\lambda \bsbs& =  \lambda\Proj_{\bsbU} \bsbs +  \lambda (\tilde \bsbX  \Proj_\bsbU^\perp )^T\bsbZ (\bsbZ^T \bsbZ)^{-1} \bsbt +  \tilde \bsbX  ^T \Proj_{\bsbZ}^\perp \bsbeps,
\end{split}\label{redkkt1}
\end{align}
where $ (\tilde \bsbX  \Proj_\bsbU^\perp )^T\bsbZ (\bsbZ^T \bsbZ)^{-1} \bsbt $ also equals $  (\tilde \bsbX  \Proj_\bsbU^\perp )^T\tilde  \bsbX\Proj_\bsbU\cdot \{(\tilde  \bsbX\Proj_\bsbU)^T \tilde  \bsbX\Proj_\bsbU\}^{+} \Proj_\bsbU  \bsbs$ or $  (\tilde \bsbX  \Proj_\bsbU^\perp )^T  \{(\tilde  \bsbX\Proj_\bsbU)^T \}^{+} \Proj_\bsbU  \bsbs$.
\end{lemma}
\begin{proof}
 By multiplying both sides of \eqref{kkteq} by   $\bsbU^T$,  we obtain
\begin{align}
\bsbZ^T\bsbZ \hat \bsbg - \bsbZ^T \bsby + \lambda \bsbt = \bsb0,
\end{align}
from which it follows that
\begin{align}
\hat \bsbg  = (\bsbZ^T\bsbZ )^{-1} ( \bsbZ^T \bsby  -  \lambda \bsbt)=\bsbg^* + (\bsbZ^T \bsbZ)^{-1} \bsbZ^T \bsbeps - \lambda(\bsbZ^T\bsbZ )^{-1}  \bsbt
\end{align}
and
$$
\tilde\bsbX\hat \bsbb =    \bsbZ\bsbg^* + \Proj_{\bsbZ} \bsbeps - \lambda \bsbZ(\bsbZ^T\bsbZ )^{-1}  \bsbt .
$$
Plugging the expression back into \eqref{kkteq} yields
\begin{align}
\lambda \bsbs & =  \tilde \bsbX  ^T  \bsby - \tilde \bsbX^T \tilde \bsbX\bsbg^*- \tilde \bsbX  ^T \Proj_{\bsbZ}  \bsbeps + \lambda \tilde \bsbX^T \bsbZ (\bsbZ^T \bsbZ)^{-1} \bsbt = \tilde \bsbX  ^T \Proj_{\bsbZ}^\perp \bsbeps + \lambda \tilde \bsbX^T \bsbZ (\bsbZ^T \bsbZ)^{-1} \bsbt.
\end{align}
Moreover,
given that     $\bsbU$ is orthogonal,
 \begin{align*}
\tilde \bsbX^T \bsbZ (\bsbZ^T \bsbZ)^{-1} \bsbt &= (\tilde \bsbX ( \bsbU \bsbU^T + \bsbI -  \bsbU \bsbU^T))^T \bsbZ (\bsbZ^T \bsbZ)^{-1} \bsbt = \bsbU  (\bsbU^T   \bsbs) +  (\tilde \bsbX  \Proj_\bsbU^\perp )^T \tilde \bsbX\bsbU (\bsbU ^T \tilde \bsbX^T \tilde \bsbX\bsbU   )^{+} \bsbU^T \bsbs\\
& = \Proj_{\bsbU} \bsbs +   (\tilde \bsbX  \Proj_\bsbU^\perp )^T\tilde  \bsbX\Proj_\bsbU \Proj_\bsbU(\tilde  \bsbX^T \tilde  \bsbX )^{+} \Proj_\bsbU\Proj_\bsbU  \bsbs = \Proj_{\bsbU} \bsbs +   (\tilde \bsbX  \Proj_\bsbU^\perp )^T\{\tilde
\bsbX\Proj_\bsbU \{(\tilde  \bsbX\Proj_\bsbU)^T \tilde  \bsbX\Proj_\bsbU\}^{+} \}\Proj_\bsbU  \bsbs\\
 & = \Proj_{\bsbU} \bsbs +   (\tilde \bsbX  \Proj_\bsbU^\perp )^T(\tilde  \bsbX\Proj_\bsbU)^{+T}  \Proj_\bsbU  \bsbs.
\end{align*}
The conclusion in the lemma follows.
\end{proof}

Under  \eqref{pattconsismeaning}, we choose    $\bsbU$  as $\bsbU(\bsbb^*)$. Then $\Proj_{\bsbZ} = \Proj^{*}$, $\Proj_{\bsbZ}^\perp  = \Proj^{*,\perp}$, and    using the definitions of $\bsbt^* $, $\overunderline\bsbs^*$, and particularly $\bsbh $ which can be expressed as
\begin{align}
\bsbh  = ( \tilde \bsbX^{*,\perp})^T \bsbZ (\bsbZ^T \bsbZ)^{-1} \bsbt^*, \end{align}    \eqref{redkkt1} simplifies to
\begin{align}
\lambda \bsbs = \lambda \overunderline \bsbs^* + \lambda \bsbh +  \tilde \bsbX  ^T \Proj^{*,\perp}\bsbeps,
\end{align}
 and for the $j$-th block,
\begin{align}
\lambda \bsbs_j   =    \lambda \overunderline \bsbs^*_j +  \lambda \bsbh_j +  \tilde \bsbX_j ^T \Proj^{*,\perp} \bsbeps.\label{sjkkt}
\end{align}
This identity can   determine when the pattern of  $\hat\bsbb_j$ matches  that of  $\bsbb_j^*$.

Recall       $\bsbU^* = \diag\{\bsbU_j^*\}$. Assuming the components of $\bsbb_j^*$ are ordered consecutively as  $\overline{\mathcal M}_j^*,  {\mathcal M}_j^*, \underline{\mathcal M}_j^* $ in each block,   $
\bsbU_j^* = \diag\{ \frac{1}{\sqrt{\overline M_j^*}} \bsb1, \bsbI_{M_j^*},  \frac{1}{\sqrt{\underline M_j^*}}\bsb1 \}
$,   reducing to
$\bsbU_j^* = \bsb1 /\sqrt m $       for $j\in \mathcal J^{*c}$.
Hence, $\Proj_{\bsbU_j^*} = \diag\{\Proj_{\bsb1_{\overline M_j^*}}, \bsbI, \Proj_{\bsb1_{\underline M_j^*}}\}$ and $\Proj_{\bsbU_j^*}^\perp = \diag\{\Proj_{\bsb1_{\overline M_j^*}}^\perp, \bsb0, \Proj_{\bsb1_{\underline M_j^*}}^\perp\}$  for $j\in \mathcal J^*$;    $\Proj_{\bsbU_j^*} =\Proj_{\bsb1_m}$ and $\Proj_{\bsbU_j^*}^\perp =\Proj_{\bsb1_m}^\perp$ for $j\in \mathcal J^{*c}$.  Intuitively,    $\Proj_{\bsbU_j^*}^\perp $    centers the columns of $\tilde \bsbX_j$   corresponding to    $\overline{\mathcal M}_j^* , \underline{\mathcal M}_j^* $   individually and zeros out the   columns associated with the intermediate set    $   {\mathcal M}_j^*$. Furthermore,   simple calculation shows
\begin{align}
  \overunderline \bsbs^*_j  =
\begin{cases}
\begin{bmatrix}\frac{1}{{\overline M_j^*}}\bsb1\\ \bsb0_{M_j^*} \\ -\frac{1}{{\underline M_j^*}}\bsb1\end{bmatrix}, &j \in \mathcal J^* \\
\bsb0, &j \in \mathcal J^{*c},
\end{cases}
 \quad \bsbt_j^* =
\begin{cases}
\begin{bmatrix}\frac{1}{\sqrt{\overline M_j^*}}  \\ \bsb0_{M_j^*} \\ -\frac{1}{\sqrt{\underline M_j^*}} \end{bmatrix}, & j \in \mathcal J^* \\
0, & j \in \mathcal J^{*c}.
\end{cases}
\end{align}
Because $  \Proj^{*,\perp}\tilde \bsbX_j =  \big[ \Proj^{*,\perp}\tilde \bsbX_j[:, \overline{\mathcal M}_j^*], \bsb0 ,  \Proj^{*,\perp}\tilde \bsbX_j[:, \underline{\mathcal M}_j^*]  \big]$ and $( \tilde \bsbX_j \Proj_{\bsbU_j^*}^\perp)[:,\mathcal M_j^*]=   \bsb0 $, \begin{align}
(\tilde \bsbX_j[:,\mathcal M_j^*]) ^T \Proj^{*,\perp} \bsbeps=\bsb0,\ \bsbh_j[\mathcal M_j^*]= \bsb0,
\end{align}
i.e., the ${\mathcal M}_j^*$-blocks of all terms in \eqref{sjkkt} are $\bsb0$. In addition,   due to the construction of $\bsbU$,  $\langle \bsb1, \bsbs_j[\overline{\mathcal M}_j^*]  \rangle=1$ and $\langle \bsb1, \bsbs_j[\underline{\mathcal M}_j^*]  \rangle=-1$.

Therefore, for $j\in \mathcal J^*$, the following conditions are sufficient to ensure faithful pattern recovery of $\bsbb_j^*$:
\begin{align}
\min\hat \bsbb_j[\overline{\mathcal M}_j^*] > \max\hat \bsbb_j[\mathcal M_j^*], \ \min\hat \bsbb_j[\mathcal M_j^*]> \max\hat \bsbb_j[\underline{\mathcal M}_j^*]  \label{betaorderred}\\
\bsbs_j[\overline{\mathcal M}_j^*] \succeq \bsb0, \bsbs_j[\underline{\mathcal M}_j^*] \preceq \bsb0.  \label{snnegred}
\end{align}
Based on \eqref{sjkkt}, \eqref{snnegred}  is implied by
\begin{align*}
\begin{cases}
\frac{1}{{\overline M_j^*}}\lambda   \ge -\min   (\tilde \bsbX_j[:,   \overline{\mathcal M}_j^*])^T \Proj^{*,\perp} \bsbeps  - \lambda\min \bsbh_{j}[     \overline{\mathcal M}_j^*] \\
\frac{1}{{\underline M_j^*}}\lambda   \ge \max  (\tilde \bsbX_j[:,  \underline{\mathcal M}_j^*])^T \Proj^{*,\perp}\bsbeps   + \lambda\max \bsbh_{j}[     \underline{\mathcal M}_j^*]
\end{cases}
\end{align*}
or
\begin{align}
\begin{cases}\frac{1}{{\overline M_j^*}}\lambda   \ge \| (\tilde \bsbX_j[:,   \overline{\mathcal M}_j^*])^T \Proj^{*,\perp} \bsbeps \|_\infty +  \lambda \| \bsbh_j[     \overline{\mathcal M}_j^*]\|_{\infty} \\
\frac{1}{{\underline M_j^*}}\lambda   \ge \| (\tilde \bsbX_j[:,  \underline{\mathcal M}_j^*])^T \Proj^{*,\perp}\bsbeps \|_\infty +  \lambda \| \bsbh_j[     \underline{\mathcal M}_j^*]\|_{\infty}.
\end{cases}\label{lamb-extrBnd1}
\end{align}
It follows from
\begin{align*}
&\max_{1\le j\le p} \,\{{\overline M_j^*}  \| (\tilde \bsbX_j[:,   \overline{\mathcal M}_j^*])^T \Proj^{*,\perp} \bsbeps \|_\infty \vee \underline M_j^* \| (\tilde \bsbX_j[:,  \underline{\mathcal M}_j^*])^T \Proj^{*,\perp}\bsbeps \|_\infty \} \\
 \le\, &  m\max_{1\le j\le p}    \|(\tilde\bsbX_j[:,\overunderline{\mathcal M}_j^*])  ^T\Proj^{*,\perp} \bsbeps \|_\infty = m\| \tilde\bsbX  ^T\Proj^{*,\perp} \bsbeps \|_{\infty}
= m\|  ( \tilde \bsbX^{*,\perp}  )^T   \Proj^{*,\perp} \bsbeps \|_{\infty}
\end{align*}
that $ m\| \tilde\bsbX  ^T\Proj^{*,\perp} \bsbeps \|_{\infty} \le  (1-\alpha) \lambda$ indicates \eqref{snnegred}.

According to \eqref{lem:betahatsol}, \eqref{betaorderred} is implied by
\begin{align}
\begin{split}
G^* &  >  2 \| \bsbU ^* (\bsbZ^T \bsbZ)^{-1} \bsbZ^T \bsbeps  \|_{\infty}+ 2\lambda \| \bsbU ^*(\bsbZ^T\bsbZ )^{-1}  \bsbt^{*} \|_{\infty}   =  2 \| \Proj_{\bsbU ^* } (\tilde \bsbX^T \tilde \bsbX)^{+}\Proj_{\bsbU ^* } \tilde \bsbX^T \bsbeps  \|_{\infty}+ 2\lambda \| \bsbU ^*(\bsbZ^T\bsbZ )^{-1}  \bsbU^{*T}{\overunderline \bsbs}^{*} \|_{\infty}  \\
& =  2 \|  \Proj_{\bsbU ^* }\tilde \bsbX ^{+} \Proj^*\bsbeps  \|_{\infty}+ 2\lambda \|\Proj_{\bsbU ^* }  (\tilde \bsbX^T\tilde \bsbX )^{+}  \,{\overunderline \bsbs}^{*} \|_{\infty} =  2 \|  (\tilde \bsbX \Proj_{\bsbU ^* })^{+} \bsbeps  \|_{\infty}+ 2\lambda \|  \{(\tilde \bsbX \Proj_{\bsbU ^* })^T (\tilde \bsbX \Proj_{\bsbU ^* })\}^{+} \,{\overunderline \bsbs}^{*} \|_{\infty}.  \\
\end{split}\label{lamb-intermed}
\end{align}
Moreover, due to the block diagonal structure of $\bsbU^*$ and $\bsbU_j^*$ and $ \overline M_j^* , \underline M_j^*\ge 1$, a sufficient condition for \eqref{lamb-intermed} is
\begin{align}
G^*\ge 2 \| (\bsbZ^T \bsbZ)^{-1} \bsbZ^T \bsbeps  \|_{\infty}+ 2\lambda \| (\bsbZ^T\bsbZ )^{-1}  \bsbt^{*} \|_{\infty},
\end{align}
where
the infinity norms   involve significantly fewer components.

As  $j\in \mathcal J^{*c}$, it is easy to verify that the left-hand side of \eqref{sjkkt} already has mean 0, and so faithful pattern recovery is guaranteed by the condition $
\| \bsbs_j \|_1 \le 2
$ (cf.   Theorem \ref{thm:subgrad}). But under
 \eqref{lamb-extrBnd1}, \begin{align*}
\lambda\|\bsbs_j\|_1 &= \| \tilde \bsbX_j ^T  \Proj^{*,\perp} \bsbeps + \lambda \overunderline \bsbs_j^* +  \lambda \bsbh_j\|_1=    \| \tilde \bsbX_j ^T  \Proj^{*,\perp} \bsbeps + \bsb0  +  \lambda \bsbh_j\|_1 =    \| (\tilde \bsbX_j[:,\overunderline {\mathcal M}_j^*]) ^T  \Proj^{*,\perp} \bsbeps  +  \lambda \bsbh_j[ \overunderline {\mathcal M}_j^*]\|_1  \\&\le \overline{M}_j^* \| (\tilde \bsbX_j[:,   \overline{\mathcal M}_j^*])^T \Proj^{*,\perp} \bsbeps \|_\infty +  \lambda\overline{M}_j^* \|\bsbh_{j}[     \overline{\mathcal M}_j^*]\|_\infty + \underline{M}_j^* \| (\tilde \bsbX_j[:,   \underline{\mathcal M}_j^*])^T \Proj^{*,\perp} \bsbeps \|_\infty +  \lambda\underline{M}_j^* \|\bsbh_{j}[     \underline{\mathcal M}_j^*]\|_\infty \le 2\lambda.
\end{align*}
The conclusion in Theorem \ref{th:pattrecov} follows.
\\

Next, we prove Corollary \ref{cor:pattrecovGaussian}. It is easy to see that \begin{align*}\| (\bsbZ ^T \bsbZ)^{-1}\|_2 \le 1/\mu_0  , \\   \|\bsbt_j^*\|_2 = \tau_j^* , \forall j\in \mathcal J^* ,  \\ \|\overunderline{\bsbs}^*\|_2= \|\bsbt^*\|_2 =  \tau_{\mathcal J^*},
\end{align*}
where $\tau_{\mathcal J^*}$ is short for  $ \sqrt{\sum_{j \in \mathcal J^*}  \tau_j ^{*2}} $.
 Rewrite $ \tilde \bsbX^{*,\perp}$ as  $ [\tilde \bsbX_1^{*,\perp}, \ldots, \tilde \bsbX_p^{*,\perp}]$ with $\tilde \bsbX_j^{*,\perp} = \tilde \bsbX_j \Proj_{\bsbU_j^*}^\perp$.
  Applying the Cauchy-Schwarz inequality yields
  $$
\| \bsbh_j \|_{\infty} \le \| (\tilde \bsbX_j^{*,\perp} )^T\bsbZ \|_{2, \infty } \|  (\bsbZ^T\bsbZ)^{-1}\|_2\|\bsbt^*\|_2 \le \frac{\omega_0}{\mu_0} \sqrt {df^*}\tau_{\mathcal J^*}
$$
and
$$
 \|  ( \tilde \bsbX^{*T}    \tilde \bsbX^{*}    )^{+} \,{\overunderline \bsbs}^{*} \|_{\infty}\le \|     ( \bsbZ^T     \bsbZ    )^{-1}\,  {\bsbt}^{*} \|_\infty\le \|
(\bsbZ^T\bsbZ)^{-1}\|_2\| {\bsbt}^*\|_2 \le \frac{\tau_{\mathcal J^*}}{\mu_0}.
$$
For example, letting $  ( \bsbZ^T     \bsbZ    )^{-1} = \bsbU_0 \bsbD_0^{-1} \bsbU_0^T$ with $\bsbU_0$ orthogonal and $\bsbD_0$ nonsingular, $\|    ( \bsbZ^T     \bsbZ    )^{-1}\,  {\bsbt}^{*} \|_{\infty}\le \| \bsbU_0  \|_{2,\infty} \allowbreak \| \bsbD_0^{-1} \bsbU_0^T{\bsbt}^{*}\|_2 \le 1 \cdot (1/\mu_0) \cdot \tau_{\mathcal J^*}$.

Using the independence between $\Proj^{*} \bsbeps$ and $\Proj^{*,\perp} \bsbeps$, we obtain
\begin{align*}
p_e \ge \, & \EP[ \| ( \tilde \bsbX^{*,\perp}  )^T\Proj^{*,\perp} \bsbeps \|_{\infty}   \le    \lambda (\frac{1}{m}-\frac{\omega_0}{\mu_0}\tau_{\mathcal J^*}  \sqrt{df^*})]    \times \EP[   \|  (\tilde \bsbX^{ *})^{+} \Proj^{* } \bsbeps  \|_{\infty}  \le    \frac{G^*}{2}- \lambda \frac{\tau_{\mathcal J^*}}{\mu_0}] \\
\ge \, &  \EP[\max_{1\le j \le p} \| (
{\tilde \bsbX}_j[:,  \overunderline {\mathcal M}_j^*] )^T\Proj^{*,\perp} \bsbeps \|_{\infty}   \le    \lambda (\frac{1}{m}-\frac{\omega_0}{\mu_0}\tau_{\mathcal J^*}  \sqrt{df^*})]    \times \EP[ \|  \bsbZ^+ \bsbeps  \|_{\infty}  \le    \frac{G^*}{2}- \lambda \frac{\tau_{\mathcal J^*}}{\mu_0}].
\end{align*}
\begin{lemma}\label{lem:gaussconcen}
Let   $\bsbA\in \mathbb R^{q\times N}$ with $\| \bsbA\|_2 \le 1$ and $q$ possibly larger than $ N$.
Assume $\bsbeps=[\epsilon_i]\in\mathbb R^N$ and $\bsbz=[z_k]\in\mathbb R^q$ consist of i.i.d.    $\mathcal N(0, \sigma^2)$ entries. Then $\EP[
\bsbA \bsbeps\in R ]\ge \EP[  \bsbz\in R]$ for any convex, closed and symmetric  $R$.
\end{lemma}

This result follows from a standard application of Anderson's inequality, observing that    $\bsbI - \bsbA \bsbA^T$ is positive semidefinite; see, for example,
Corollary 4.8 in \cite{AlbertsKhoshnevisan2018}.

Now, by setting   $R$ as the $\ell_\infty$-polytope and noting  that $\| ( \tilde \bsbX^{*,\perp}  )^T\Proj^{*,\perp}\|_2 \le \|  \tilde \bsbX\|_2$,  $\|   \bsbZ ^{+}  \|_2\le 1/\sqrt{\mu_0}$, $\sum_{1\le j \le p}\overunderline {M}_j^* = pm - M^*$ and $\sum_{j\in \mathcal J^*} (2+M_j^*)+p -  J^{*}
  = df^*$,
 the conclusion of the corollary can be derived from the standard Gaussian tail bounds.
\end{proof}

\subsection{Proof of Theorem \ref{th:acc}}
\begin{proof}
Let $\phi(\cdot)$ denote $\|\cdot\|_2^2/2$, which is differentiable.
 The desired result can be obtained from the proof of Part (ii) of Theorem 6 in \cite{Shebregman2021},  which   requires only  directional differentiability of $L$ and $\phi$ and incorporates a more flexible linearization term.
For the sake of completeness, we provide a proof below, which holds for any differentiable function $\phi$.

For convenience, let $h_{t}(\bsb{\beta}) \allowbreak= f(\bsb{\beta}) - \breg_{\bar\psi_t}(\bsb{\beta},\bsb{\gamma}^{(t)})$. Applying Lemma A.2 in  \cite{Shebregman2021} to \eqref{acc2-alg2} yields
   $(\bsb{\Delta}_f - \breg_{\breg_{\bar\psi_t}(\cdot,\bsb\gamma^{(t)})} + \theta_t\rho_t\breg_{\breg_{\phi}(\cdot,\bsb\alpha^{(t)})})(\bsb{\beta}, \bsb{\alpha}^{(t+1)})
\le h_{t}(\bsb{\beta}) + \theta_t\rho_t\breg_{\phi}(\bsb{\beta},\bsb{\alpha}^{(t)}) - h_{t}(\bsb{\alpha}^{(t+1)}) - \theta_t\rho_t\breg_{\phi}(\bsb{\alpha}^{(t+1)},\bsb{\alpha}^{(t)})$, or
\begin{equation} \label{tri_acc2}
\begin{split}
&h_{t}(\bsb{\alpha}^{(t+1)}) -h_{t}(\bsb{\beta}) + \theta_t\rho_t\breg_{\phi}(\bsb{\alpha}^{(t+1)},\bsb{\alpha}^{(t)})
\le\theta_t\rho_t\breg_{\phi}(\bsb{\beta},\bsb{\alpha}^{(t)}) -  (\theta_t\rho_t\breg_{\breg_{\phi}(\cdot,\bsb\alpha^{(t)})} + \bsb{\Delta}_{f(\cdot) - \breg_{\bar\psi_t}(\cdot,\bsb\gamma^{(t)})})(\bsb{\beta}, \bsb{\alpha}^{(t+1)})
\end{split}
\end{equation}
  for any  $ \bsb\beta$.
Multiplying \eqref{tri_acc2}   by $\theta_t$    and adding it to
$\mathbf C_{h_t}(\bsb{\alpha}^{(t+1)},\bsb{\beta}^{(t)},\theta_t)= \theta_t h_{t}(\bsb{\alpha}^{(t+1)}) + (1-\theta_t)h_{t}(\bsb{\beta}^{(t)}) - h_{t}(\bsb{\beta}^{(t+1)})
$, we obtain \begin{equation}\nonumber
\begin{split}
&h_{t}(\bsb{\beta}^{(t+1)}) - (1-\theta_t)h_{t}(\bsb{\beta}^{(t)}) - \theta_t h_{t}(\bsb\beta) \\& + \theta_t^2\rho_t\breg_{\phi}(\bsb{\alpha}^{(t+1)},\bsb{\alpha}^{(t)})  + \mathbf C_{h_t}(\bsb{\alpha}^{(t+1)},\bsb{\beta}^{(t)},\theta_t)  + \theta_t^2\rho_t\breg_{\breg_{\phi}(\cdot,\bsb\alpha^{(t)})-\phi(\cdot)}(\bsb\beta,\bsb\alpha^{(t+1)})\\
\le\,&\theta_t^2\rho_t\breg_{\phi}(\bsb{\beta},\bsb{\alpha}^{(t)}) - (\theta_t^2\rho_t\breg_{\phi} + \theta_t\bsb{\Delta}_{f(\cdot)-\breg_{\bar\psi_t}(\cdot,\bsb\gamma^{(t)})})(\bsb{\beta},\bsb{\alpha}^{(t+1)}),
\end{split}
\end{equation}
and so
\begin{equation}\label{befrec1_acc2}
\begin{split}
&f(\bsb{\beta}^{(t+1)})-f(\bsb{\beta}) - (1-\theta_t)[f(\bsb{\beta}^{(t)})-f(\bsb{\beta})]  +\theta_t\breg_{\bar\psi_t}(\bsb{\beta},\bsb{\gamma}^{(t)})  \\
&  + \theta_t \{ (\breg_{f(\cdot)-\breg_{\bar\psi_t}(\cdot,\bsb\gamma^{(t)})}+\theta_t\rho_t\breg_{\breg_{\phi}(\cdot,\bsb\alpha^{(t)})-\phi(\cdot)})(\bsb\beta,\bsb\alpha^{(t+1)})\}+ R_t
\\
\le\,&\theta_t^2\rho_t(\breg_{\phi}(\bsb{\beta},\bsb{\alpha}^{(t)}) - \breg_{\phi}(\bsb{\beta},\bsb{\alpha}^{(t+1)})), \forall t\ge 0
\end{split}
\end{equation}
where  $R_t$ is given by
\begin{align*}
& \theta_t^2\rho_t\breg_{\phi}(\bsb{\alpha}^{(t+1)},\bsb{\alpha}^{(t)}) - \breg_{\bar\psi_t}(\bsb{\beta}^{(t+1)},\bsb{\gamma}^{(t)}) + (1-\theta_t)\breg_{\bar\psi_t}(\bsb{\beta}^{(t)},\bsb{\gamma}^{(t)}) + \mathbf C_{f(\cdot)-\breg_{\bar\psi_t}(\cdot,\bsb\gamma^{(t)})}(\bsb\alpha^{(t+1)},\bsb\beta^{(t)},\theta_t) \\
= \,& \theta_t^2\rho_t\breg_{\phi}(\bsb{\alpha}^{(t+1)},\bsb{\alpha}^{(t)})- \breg_{\bar\psi_t}(\bsb{\beta}^{(t+1)},\bsb{\gamma}^{(t)}) + (1-\theta_t)\breg_{\bar\psi_t}(\bsb{\beta}^{(t)},\bsb{\gamma}^{(t)})   + (\lambda\mathbf C_{ \range^{(b)} }+  \mu_t \mathbf C_{\phi})(\bsb\alpha^{(t+1)},\bsb\beta^{(t)},\theta_t),
\end{align*}
using the idempotence and linear properties of $\breg, \mathbf C$ as derived in \cite{Shebregman2021} and the differentiability of $L$ and $\phi$.

 \eqref{befrec1_acc2} can be rewritten as
\begin{equation}\label{befrec1_acc2-2}
\begin{split}
&f(\bsb{\beta}^{(t+1)})-f(\bsb{\beta}) - (1-\theta_t)[f(\bsb{\beta}^{(t)})-f(\bsb{\beta})]+ R_t
   \\
& +\theta_t\breg_{\bar\psi_t}(\bsb{\beta},\bsb{\gamma}^{(t)})+ \theta_t\breg_{f(\cdot)-\breg_{L}(\cdot,\bsb\gamma^{(t)})}(\bsb\beta,\bsb\alpha^{(t+1)}) \\ & + \theta_t  (\mu_t\breg_{\breg_{\phi}(\cdot,\bsb\gamma^{(t)})-\phi(\cdot)}+\theta_t \rho_t\breg_{\breg_{\phi}(\cdot,\bsb\alpha^{(t)})-\phi(\cdot)})(\bsb\beta,\bsb\alpha^{(t+1)})
\\
\le\,&\theta_t^2\rho_t\breg_{\phi}(\bsb{\beta},\bsb{\alpha}^{(t)}) -\theta_t^2(\rho_t +\frac{\mu_t}{\theta_t}) \breg_{\phi}(\bsb{\beta},\bsb{\alpha}^{(t+1)}).
\end{split}
\end{equation}
Again, since $L$ and $\phi$ are differentiable,
\begin{align*}
&  \breg_{\bar\psi_t}(\bsb{\beta},\bsb{\gamma}^{(t)})+  \breg_{f(\cdot)-\breg_{L}(\cdot,\bsb\gamma^{(t)})}(\bsb\beta,\bsb\alpha^{(t+1)}) \\ &+  (\mu_t\breg_{\breg_{\phi}(\cdot,\bsb\gamma^{(t)})-\phi(\cdot)}+\theta_t \rho_t\breg_{\breg_{\phi}(\cdot,\bsb\alpha^{(t)})-\phi(\cdot)})(\bsb\beta,\bsb\alpha^{(t+1)})\\
=\, &\breg_{\bar\psi_t}(\bsb{\beta},\bsb{\gamma}^{(t)})+ \lambda\breg_{\range^{(b)}}(\bsb\beta,\bsb\alpha^{(t+1)}) + 0 =\mathcal E_t(\bsb\beta).
\end{align*}
Therefore, we have
\begin{equation}\nonumber
\begin{split}
& f(\bsb{\beta}^{(t+1)})-f(\bsb{\beta})   +\theta_t^2(\rho_t +\frac{\mu_t}{\theta_t}) \breg_{\phi}(\bsb{\beta},\bsb{\alpha}^{(t+1)})+ \theta_t \mathcal E_t(\bsb\beta)  +  R_t
\le (1-\theta_t)[f(\bsb{\beta}^{(t)})-f(\bsb{\beta})]+\theta_t^2\rho_t \breg_{\phi}(\bsb{\beta},\bsb{\alpha}^{(t)}), \forall t\ge 0
\end{split}
\end{equation}
and from  \eqref{acc2_search2},
\begin{equation}\nonumber
\begin{split}
& f(\bsb{\beta}^{(t+1)})-f(\bsb{\beta})   +\theta_t^2(\rho_t +\frac{\mu_ t}{\theta_t}) \breg_{\phi}(\bsb{\beta},\bsb{\alpha}^{(t+1)})+ \theta_t \mathcal E_t(\bsb\beta)  +  R_t \\
\le\,& (1-\theta_t)\big[f(\bsb{\beta}^{(t)})-f(\bsb{\beta})+\theta_{t-1}^2(\rho_{t-1} +\frac{\mu_{t-1}}{\theta_{t-1}})\breg_\phi(\bsb{\beta},\bsb{\alpha}^{(t)})\big], \forall t\ge 1.
\end{split}
\end{equation}
Given that   $1-\theta_t> 0, \forall t\ge 1$, which follows from  $\rho_t> 0, \mu_t\ge 0$,  the conclusion in the theorem  can be obtained  by a recursive argument and noting that   $ R_T +\theta_T \mathcal E_T(\bsb\beta )  +(1-\theta_T)(R_{T-1} + \theta_{T-1} \mathcal E_{T-1}(\bsb\beta)  )+ \cdots (1-\theta_T)\cdots (1-\theta_1)(R_{0}+ \theta_{0} \mathcal E_{0}(\bsb\beta )) = \sum_{t=0}^T (\prod_{s=t+1}^T (1-\theta_s))( R_t + \theta_t \mathcal E_t(\bsb{\beta} )$). \end{proof}

\subsection{Dual Seminorm Analysis}
\label{subsec:dualnormanal}

   Given a block vector   $ \bsba =[\bsba_1^T, \ldots, \bsba_p^T]^T
$ (which aligns with the structure of  $\bsbb^*$), we     introduce a short notation
\begin{align}\bsba_{ \overunderline {\mathcal M}^*_{\mathcal J^*} }& = \Big[ \bsba_{j }[\overunderline {\mathcal M}^*_j ] \Big] _{j \in \mathcal J^*}\in \mathbb R^{\sum_{j\in \mathcal J^*} \overunderline M_j^*},
\end{align}
by removing the components indexed by $\mathcal M_j^*$ ($j\in \mathcal J^*$), where the block vector $\bsba_{ \overunderline {\mathcal M}^*_{\mathcal J^*} } $  consists of $J^*$ blocks, each   of size $\overunderline  M_j^*$.
Similarly, we use $\bsba_{ \overunderline {\mathcal M}^* }  $ to denote the block vector $ \big[ \bsba_{j }[\overunderline {\mathcal M}^*_j ] \big] _{1\le j \le p }$.
Let $\uprho_0 $  denote the maximum column $\ell_2$-norm of  $\tilde \bsbX$:
\begin{align}
\uprho_0 :=    \| \tilde \bsbX^T\|_{2,\infty}=  \max_{1\le k \le m} \|\bsbX_k^{\circ  T}\|_{2,\infty}    , \label{uprho0choice}
\end{align}
which  is slightly different from the $\varrho_0$ as given  in \eqref{rho0choice} (and satisfies $\uprho_0 \le \varrho_0$).  Furthermore, due to the special structure (column orthogonality) of $\tilde \bsbX_j$, it can also be defined through the operator norm:
$$
\uprho_{0} = \max_{1\le j \le p} \| \tilde \bsbX_j\|_2.
$$
If
 we set $\rho_k =   \| \bsbX_k^T\|_{2,\infty}$ in  implementation, $\uprho_0 =1$.

For comparison with Theorem \ref{thm:seesawerr-scvx}, we present a theorem where  $l_0$ or $\tilde l_0$ is strongly convex in the systematic component, without assuming \eqref{gradLip}.  Yet a more general   result akin to   Theorem \ref{thm:seesawerr} is attainable according to the proof.
\begin{theorem}\label{th:dualmethoderrrate}
Assume the loss $\tilde l_0$ is differentiable  and   $\nu$-strongly convex in  $\bsbe$ (as is the case in regression,  where $\nu=1$).

Assume for some $\vartheta, \kappa_0>0 $ the following regularity condition holds:
\begin{align}
  &\frac{\kappa_0^2}{J^*}   \{ \range^{(b)}(   \bsba_{ \overunderline {\mathcal M}^*}) \}^2   \le     \| \tilde \bsbX \bsba\|_2^2
\label{regcond-bound}
\end{align}
   for any  $\bsba$ restricted by
 \begin{align}   \vartheta       \sum_{j=1}^p \|   \bsba_j[ \overunderline {\mathcal M}_j^* ]  \|_{\range}    \ge       \sum_{j=1}^p  \min_{  \overline {\mathcal M} (\bsbb_j^* + \bsba_j)}     \bsba_j[k] -  \max_{  \underline {\mathcal M} (\bsbb_j^* + \bsba_j)} \bsba_j[k].
\label{regcond-region0}
\end{align}
Let  $\lambda = (A  /\vartheta) \lambda_0$ with the constant $A\ge 2$\   and   $\lambda_0$ satisfying
\begin{align}
&2 \lambda_0 \ge    \max_{1\le j\le p} \|\tilde\bsbX_j  ^T\Proj^{*,\perp} \bsbeps\|_{1}. \label{lambda0choice}
\end{align}
Then
\begin{align}
   \|  \tilde \bsbX (\hat \bsbb -     \bsbb^*)\|_2^2    \lesssim    \frac{1}{\nu(\nu \wedge 1)}   \|  \Proj^{*} \bsbeps \|_2^2  +   \frac{ \lambda_0^2 J^*}{\nu \kappa_0^2}.\label{th:errdualmethod}
\end{align}
\end{theorem}
\begin{corollary}\label{cor:dualmethod}
In the context  of  Theorem \ref{th:dualmethoderrrate}, assume the  $\psi_2$-norm  of the random vector $\bsbeps$  is bounded by $\sigma$. For   $\lambda_0= c_{0}\uprho_0\sigma     m \sqrt{ \log( m p)}$,  the error bound in \eqref{th:errdualmethod},   replacing  $\|  \Proj^{*} \bsbeps \|_2^2  $  with $\sigma^2 (M^* + p)$,
 holds with probability at least $1- C    p^{-c   } $, where $c_0, c, C$ are sufficiently large     constants.

Furthermore, in the special case where  $\bsbeps=[\epsilon_i]$ comprises  i.i.d. $\mathcal N(0,\sigma^2)$ entries,     setting  $\lambda_0$ to $ c_{0} \uprho_0\sigma    ( m+ \sqrt{   m  \log p})$ is sufficient to maintain the same     probability level.
 \end{corollary}

In the following, we compare   the results obtained using the dual approach  with those in Section \ref{subsec:errrate}.

The corollary guiding  the selection       of  $\lambda$ or $\lambda_0$  typically indicates  a suboptimal rate
       $$ \sigma m \sqrt{ \log m+ \log p}.$$     Yet  in the special   case of i.i.d. Gaussians, it seems possible   to {significantly} lower the parameter choice  to  $$\sigma (m + \sqrt{m \log p}),$$
although the resulting error rate \eqref{subopterrrate2} remains unsatisfactory.

Technically, applying the \emph{dual seminorm} requires projecting the stochastic term onto the kernel component space. Thus, \eqref{lambda0choice} incorporates    $\Proj^{*,\perp} $, which can introduce dependencies that elevate the threshold level in various scenarios.     In contrast, the proof techniques employed in  Theorem \ref{thm:seesawerr-scvx}      successfully establish   the   optimal choice  of $\lambda_0$  for a  general {subGaussian} $\bsbeps$.

Observe that the requirement of  $\vartheta>0$ is necessary in Theorem \ref{th:dualmethoderrrate}. The left-hand term of \eqref{regcond-region0},   composed of the sum of ranges of subvectors corresponding to the extreme values of $\bsbb_j^*$,  (inevitably) arises from applying the dual seminorm.
 But it    results in a notably broader restriction region     than $\mathcal C_0$   in Theorem \ref{thm:seesawerr-scvx}, which does not involve the range of the difference vector $\bsba$ at all.

For a clearer understanding of the error rate,
 we apply Theorem \ref{thm:tau} and     the Cauchy-Schwarz inequality to get
$
\range^{(b)}(   \bsba_{ \overunderline {\mathcal M}^*}) \le (\sum_{j\in \mathcal J^*}\tau( \bsba_{ \overunderline {\mathcal M}_j^*}))^{1/2}   \|  \bsba_{ \overunderline {\mathcal M}^*}\|_2$, which, in the worst-case scenario where \begin{align}
  \tau( \bsba_{ \overunderline {\mathcal M}_j^*}) \le \sqrt 2, \label{normcvt1}
\end{align} leads to   $\range^{(b)}(   \bsba_{ \overunderline {\mathcal M}^*}) \lesssim (J^*)^{1/2}\|  \bsba_{ \overunderline {\mathcal M}^*}\|_2$.   Hence, with  $  \kappa_0\gtrsim 1$,     we can achieve an error rate of $M^* + p+ \lambda_0^2 J^*$,  ignoring trivial factors and constants. According to the corollary,     this simplifies to
\begin{align}
p + m^2 J^*\log m+m^2 J^*   \log p \label{subopterrrate1}
\end{align}
in general subGaussian scenarios, and to
\begin{align}
p+m^2 J^*  + m J^* \log p \label{subopterrrate2}
\end{align}
in the cases of i.i.d. Gaussians.  Both     rates are however suboptimal compared to   $$p +m J^*+  J^* \log p ,$$   derived using our seesaw device in Theorems  \ref{thm:seesawerr-scvx}, \ref{thm:seesawerr}.

Of course, the (semi)norm conversion    in the worst-case scenario \eqref{normcvt1}  tends to be conservative. To achieve the desired rate using Theorem \ref{th:dualmethoderrrate}, one would need
\begin{align}
 \tau( \bsba_{ \overunderline {\mathcal M}_j^*})\asymp \frac{1}{\sqrt m }, \ \forall j \in \mathcal J^*. \label{normcvt2}
\end{align}
Recalling that $\bsba$ in the regularity condition resembles $\hat \bsbb - \bsbb^*$, \eqref{normcvt2} is implied by   \emph{faithful} pattern recovery, as discussed in Section \ref{subsec:pattrec}.  However, this   not only requires
more stringent regularity conditions, but also leads to the less favorable rate for $\lambda$ at $\sigma m\sqrt{\log p + \log m}$.
In contrast, Theorems  \ref{thm:seesawerr-scvx}, \ref{thm:seesawerr}  offer a much more viable approach with fewer requirements and a sharper rate.

\begin{proof}Recall the basic inequality obtained in \eqref{seesawbasiceq1} and the decomposition  $\langle \bsbeps,  \tilde  \bsbX (\hat {\bsbb} - \bsbb^*)\rangle     =    \langle \bsbeps, \Proj^{*}   \tilde  \bsbX (\hat {\bsbb} - \bsbb^*)\rangle+   \langle \bsbeps, \Proj^{*,\perp}   \tilde  \bsbX  \hat {\bsbb}  \rangle
$.
For the first stochastic term,
we employ  H\"older's inequality to obtain \begin{align}
\langle \bsbeps, \Proj^{*}   \tilde  \bsbX (\hat {\bsbb} - \bsbb^*)\rangle \le \|  \Proj^{*} \bsbeps \|_2 \cdot \| \Proj^{*}   \tilde  \bsbX (\hat {\bsbb} - \bsbb^*)\|_2  ,\label{stochterm1}
\end{align}
where
$
 \| \Proj^{*}   \tilde  \bsbX (\hat {\bsbb} - \bsbb^*)\|_2^2= \|  \Proj^{*}     (  \hat {\bsbe} - \bsbe^*) \|_2^2 \   \le  \|   \hat {\bsbe} - \bsbe^* \|_2^2.$

 Let
$
 \bsba =[\bsba_1^T, \ldots, \bsba_p^T]^T= \hat \bsbb - \bsbb^*.
$  It follows from  $\tilde \bsbX \bsbb^*\in \Proj^*,\tilde \bsbX_j \Proj_K \bsba\in \Proj^*,   \tilde \bsbX_j[:,   {\mathcal M}^*_j] \bsba_j[  {\mathcal M}^*_j] \in \Proj^*$,           $\bsbb_j^*[\overline {\mathcal M}^*_j]\in \Proj_K $   and   $\bsbb_j^*[\underline {\mathcal M}^*_j]\in \Proj_K $  for all $j$ that
\begin{align*}
    \Proj^{*,\perp}   \tilde  \bsbX (\hat {\bsbb} - \bsbb^*) & =    \Proj^{*,\perp}\sum_{j\in \mathcal J^*} \tilde \bsbX_j  \hat  \bsbb_j +   \Proj^{*,\perp}\sum_{j\in \mathcal J^{*c}} \tilde \bsbX_j    \hat\bsbb_j  \\
&=    \Proj^{*,\perp}    \sum_{j\in \mathcal J^*}      \tilde \bsbX_j[:, \overunderline {\mathcal M}^*_j]  \hat  \bsbb_j[\overunderline {\mathcal M}^*_j]     +   \Proj^{*,\perp}\sum_{j\in \mathcal J^{*c}} \tilde \bsbX_j  \Proj_{K}^\perp  \hat\bsbb_j\\
&=    \Proj^{*,\perp}    \sum_{j\in \mathcal J^*}      \tilde \bsbX_j[:, \overunderline {\mathcal M}^*_j]   \Proj_K^\perp \hat  \bsbb_j[\overunderline {\mathcal M}^*_j]     +   \Proj^{*,\perp}\sum_{j\in \mathcal J^{*c}} \tilde \bsbX_j  \Proj_{K}^\perp  \hat\bsbb_j\\
&=        \Proj^{*,\perp}    \sum_{j\in \mathcal J^*}      \tilde \bsbX_j[:, \overunderline {\mathcal M}^*_j]   \Proj_K^\perp    \bsba_j[\overunderline {\mathcal M}^*_j]     +   \Proj^{*,\perp}\sum_{j\in \mathcal J^{*c}} \tilde \bsbX_j  \Proj_{K}^\perp    \bsba_j.
\end{align*}
Applying Theorem \ref{thm:dualnorm} gives
\begin{align*}
 \langle \bsbeps, \Proj^{*,\perp}   \tilde  \bsbX (\hat {\bsbb} - \bsbb^*)\rangle
   = &    \sum_{j\in \mathcal J^*}   \langle   (\tilde \bsbX_j[:, \overunderline {\mathcal M}^*_j])^T \Proj^{*,\perp} \bsbeps,  \Proj_K^\perp \bsba_j[\overunderline {\mathcal M}^*_j]  \rangle   + \sum_{j\in \mathcal J^{*c}} \langle \tilde \bsbX_j  ^T\Proj^{*,\perp} \bsbeps,  \Proj_{K}^\perp  \hat \bsbb_j\rangle \\
 \le & \sum_{j\in \mathcal J^*} \lambda_0  \| \bsba_{j  }[\overunderline {\mathcal M}^*_j]\|_{\range}    + { \range}^{*}_{\lambda_0}(\tilde \bsbX_j[:, \overunderline {\mathcal M}^*_j])^T \Proj^{*,\perp} \bsbeps) +  \sum_{j\in \mathcal J^{*c}}\lambda_0 \|   \hat\bsbb_j\|_{\range}+{ \range}^{*}_{\lambda_0}( \tilde \bsbX_j  ^T\Proj^{*,\perp} \bsbeps).
\end{align*}
Thanks to the   orthogonal decomposition, not only do            $ \Proj_K^{\perp} \bsba_j[\overunderline {\mathcal M}^*_j]\in\Proj_{K}^\perp , \Proj_{K}^\perp  \hat\bsbb_j\in \Proj_{K}^\perp $, but $  \bsbX_j[:, \overunderline {\mathcal M}^*_j])^T \Proj^{*,\perp} \bsbeps $ and $ \tilde \bsbX_j  ^T\Proj^{*,\perp} \bsbeps$   also belong to the associated kernel   complements---for example, $\Proj^{*,\perp}\tilde \bsbX_j[:, \overunderline {\mathcal M}^*_j] \bsb1 =\Proj^{*,\perp}\tilde \bsbX_j[:, \overline {\mathcal M}^*_j] \bsb1 +\Proj^{*,\perp} \tilde \bsbX_j[:, \underline {\mathcal M}^*_j] \bsb1 = \bsb0$), preventing the bound from  becoming   $+\infty$.
Based on the dual seminorm expression derived in Theorem \ref{thm:dualnorm}, choosing   $\lambda_0$ according to
\begin{align}
&2 \lambda_0 \ge\max_{j\in \mathcal J^*}\|   (\tilde \bsbX_{j }[:,\overunderline {\mathcal M}^*_j] ) ^T \Proj^{*,\perp} \bsbeps\|_1 , 2\lambda_0 \ge  \max_{j\in \mathcal J^{*c}} \|\tilde \bsbX_j  ^T\Proj^{*,\perp} \bsbeps\|_{1}, \label{lambda0raw}
\end{align}
ensures
\begin{align}
\langle \bsbeps, \Proj^{*,\perp}   \tilde  \bsbX (\hat {\bsbb} - \bsbb^*)\rangle  & \le    \lambda_0 \sum_{j\in \mathcal J^*}  \| \bsba_{j  }[\overunderline {\mathcal M}^*_j]\|_{\range}  + \lambda_0 \sum_{j\in \mathcal J^{*c}} \|   \hat\bsbb_j\|_{\range}=\lambda_0 \range^{(b)}(   \bsba_{ \overunderline {\mathcal M}^*}). \label{stochterm2}
\end{align}
Furthermore, because $$(\tilde \bsbX_j[:,k])^T\Proj^{*,\perp} \bsbeps =\langle\Proj^{*,\perp} \tilde \bsbX_j[:,k], \bsbeps\rangle =0,\forall k \in \mathcal M_j^*$$ \eqref{lambda0raw} is equivalent to
\begin{align}
&2 \lambda_0 \ge    \max_{1\le j\le p} \|\tilde\bsbX_j  ^T\Proj^{*,\perp} \bsbeps\|_{1}. \label{lambda0choice1}
\end{align}
The bound in \eqref{stochterm2} excludes $\bsba_j[\mathcal M_j^*]$ for each $j \in \mathcal J^*$ in the first term, and removes $\bsbb_j^*$ from $\bsba_j$ for each $j \in \mathcal J^{*c}$   in the second term.

Combining \eqref{stochterm1}, \eqref{stochterm2} and  \eqref{lambda0choice1} yields a bound on the stochastic term  \begin{align}
 &\langle \bsbeps, \tilde \bsbX (\hat {\bsbb} - \bsbb^*)\rangle  \le  \frac{1}{a} \Breg_2 (\Proj^{*} \hat \bsbe,\Proj^{*}  \bsbe^*) + \frac{a}{2} \|  \Proj^{*} \bsbeps \|_2^2+ \lambda_0 \range^{(b)}(   \bsba_{ \overunderline {\mathcal M}^*}) \label{stochoverallbound}
\end{align}
 for any  $a>0$.
Let  $\lambda =   A\lambda_0/\theta$ or $\lambda_0 = \theta \lambda/A$ with $\theta>0$ and constant $A\ge 1$.
Invoke   the following condition  \begin{align}
& \vartheta\lambda\range^{(b)}(   \bsba_{ \overunderline {\mathcal M}^*}) + \lambda \range^{(b)} (\bsbb^*) -  \lambda \range^{(b)} (\hat \bsbb ) \le   (2- \alpha_{1})\bregs_{\tilde l_0}   (\tilde \bsbX \hat \bsbb, \tilde \bsbX   \bsbb^*)  \notag \\ &- \alpha_{2} \Breg_2    (\Proj^{*} \tilde \bsbX \hat \bsbb, \Proj^{*} \tilde \bsbX   \bsbb^*)   +(1 -\alpha')   \lambda \sum_{j=1}^p \breg_{\range}(\bsbb_j^*,\hat \bsbb_j )  +C(F \lambda^2 + \|  \Proj^{*} \bsbeps \|_2^2)\label{regcond0}
\end{align}
for some $\vartheta> 0, \alpha>0$, $F>0$, $\alpha'\ge 0$,   and some constant  $C>0$. Then,  with   $a = 2/\alpha_{2}, \theta = \vartheta/2$, we have
\begin{align}
& \alpha_{1}\bregs_{\tilde l_0}   (\tilde \bsbX \hat \bsbb, \tilde \bsbX   \bsbb^*) \vee \alpha_{2} \|\Proj^{*}    \tilde \bsbX (\hat \bsbb -     \bsbb^*)\|_2^2\vee \frac{\alpha'}{\vartheta}  \lambda_0  \sum  \breg_{\range}(\hat \bsbb_j, \bsbb_j^*) \vee \lambda_0\range^{(b)} (  (\hat \bsbb - \bsbb^*)_{ \overunderline {\mathcal M}^*} )\notag \\
&\lesssim    \frac{C}{\vartheta^2}  F \lambda_0^2 +(\frac{1}{\alpha_{2}}+C)\|  \Proj^{*} \bsbeps \|_2^2. \label{genres1}
\end{align}

In the following, we set $\alpha'=0$ and assume $\tilde l_0$ is $\nu$-strongly convex,  indicating   $\bregs_{\tilde l_0} (\bsbe^* + \tilde \bsbX \bsba,    \bsbe^*)\ge \nu \| \tilde \bsbX\bsba\|_2^2/2\ge \nu \| \Proj^{*}\tilde \bsbX\bsba\|_2^2/2$.   Then a     sufficient condition for \eqref{regcond0} is  \begin{align}
 \vartheta\range^{(b)}(   \bsba_{ \overunderline {\mathcal M}^*})  \le\, &    2  \sqrt{C F  (2-\alpha_{1}-\frac{\alpha_2}{\nu})}  \bregs_{\tilde l_0}^{1/2}     (\tilde \bsbX   \bsbb^* +\tilde \bsbX \bsba, \tilde \bsbX   \bsbb^*)
     -\sum_{j=1}^p\delta \range(\hat{ \bsbb}_j; \bsbb_j^*  - \hat{\bsbb}_j) .\label{regcond-overall}
\end{align}
\eqref{regcond-bound} in the theorem implies
\begin{align}
  &\frac{\nu\kappa_0^2}{J^*} \{ \range^{(b)}(   \bsba_{ \overunderline {\mathcal M}^*}) \}^2   \le2        \bregs_{\tilde l_0} (\bsbe^* + \tilde \bsbX \bsba,    \bsbe^*)\label{compcondboundproof}
\end{align}
   and
\begin{align*}    -\delta \range(\hat{ \bsbb}_j; \bsbb_j^*  - \hat{\bsbb}_j) & =-(  \sum_{j=1}^p  \max_{  \overline {\mathcal M} (\bsbb_j^* + \bsba_j)}     (-\bsba_j[k]) -  \min_{  \underline {\mathcal M} (\bsbb_j^* + \bsba_j)} (-\bsba_j[k])) = \sum_{j=1}^p  \min_{  \overline {\mathcal M} (\bsbb_j^* + \bsba_j)}     \bsba_j[k] -  \max_{  \underline {\mathcal M} (\bsbb_j^* + \bsba_j)} \bsba_j[k].
\end{align*}
Take  $\alpha_1 = 1/2$, $\alpha_2 = \nu/2$, $F = \vartheta^2J^*/\kappa_0^2 $, $C =1$. Through case analysis,   \eqref{regcond-overall} is satisfied for any $\bsba$ in the region defined by \eqref{regcond-region0} or its complement.  The conclusion in the theorem follows.

 Under the
   $\psi_2$-norm   assumption on   $\bsbeps$,  $\| \langle \bsba, \bsbeps\rangle\|_{\psi_2}\le \sigma \|\bsba\|_2$, $\forall \bsba$. Given that
\begin{align}
\max_{1\le j\le p}\allowbreak\|\tilde\bsbX_j  ^T\Proj^{*,\perp} \bsbeps\|_{1}\le m\cdot\max_{1\le j\le p} \max_{1\le k \le m} |(\tilde\bsbX_j[:,k])  ^T\Proj^{*,\perp} \bsbeps | \notag 
\end{align} and $\|\Proj^{*,\perp} \tilde\bsbX_j[:,k]    \|_2\le 1\cdot \uprho_0=\uprho_0$, we obtain $$\EP[\max_{1\le j\le p} \|\tilde\bsbX_j  ^T\Proj^{*,\perp} \bsbeps\|_{1}  \ge c_{0}\uprho_0\sigma     m \sqrt{ \log( m p)}]\le C(pm)^{-c},$$      by applying a union bound and choosing a sufficiently large constant   $c_0$. The probability control of the event $\|  \Proj^{*} \bsbeps \|^2_2\gtrsim \sigma^2(M^*+p)$ is presented in the proof of Theorem \ref{thm:seesawerr}. Combining the two bounds gives the first conclusion in the corollary.

Assuming    $\epsilon_i$  are i.i.d. Gaussian, we can  obtain a better rate.
Indeed,  since the $\ell_1$-polytope is convex, closed, and symmetric about 0,   Lemma \ref{lem:gaussconcen} yields $$\EP[\| \bsbA \bsbeps\|_1\ge t]\le \EP[\| \bsbz\|_1\ge t]$$ for  any $\bsbA\in \mathbb R^{m\times N}$ and  $t>0$, where
 $\bsbz$ consists of $m$  i.i.d. $\mathcal N(0,\sigma^2)$ entries.
Noticing that $\|\tilde\bsbX_j  ^T\Proj^{*,\perp}\|_2 \le \uprho_0 \| \Proj^{*,\perp}\|_2 = \uprho_0$, we obtain
\begin{align*}
&\EP[   \max_{1\le j\le p} \|\tilde\bsbX_j  ^T\Proj^{*,\perp} \bsbeps\|_{1}\ge c \uprho_0\sigma    ( m+ \sqrt{   m  \log p}) ] \\
\le\, &  p  \EP[\| \bsbz\|_1  -   \sqrt{\frac{2}{\pi}} \sigma m\ge c(    \sqrt m+ \sqrt{     \log p})\sigma\sqrt m ]\\
\le \, & p C\exp ( - c (m + \log p)) \le C\exp ( - c (m + \log p))
\end{align*}
where   $C,c$ are sufficiently large positive constants  that may vary across different instances.  Here, we used the     independence among $z_i$
($1\le i \le m$) and the subGaussian tails of  the centered  $|z_i|$  (each bounded with a $\psi_2$-norm  of $c\sigma$).
The proof is complete. \end{proof}

\section{Experiments}
\label{sec:exps}
\subsection{Simulations}
\label{subsec:simus}
\subsubsection{Comparison of Regularization Performance}
\label{subsubsec:compregul}
This subsection conducts experiments in various regression and classification scenarios to compare the proposed method with other approaches.   For each client $k$, the predictor matrix $\boldsymbol{X}_{k} \in \mathbb{R}^{n_k \times p}$ is constructed with rows independently drawn  from a multivariate normal distribution, each having  a Toeplitz covariance matrix $\bsbSig= [r^{|i-j|}]$. For simplicity, we set   $N = nm$ and   $n_1 = n_2 = \ldots = n_m = n$. In  regression, the  response  vector $\bsby_k$ is generated from $ \mathcal N(\bsbX_k \bsb{b}_k^*, \sigma^2 \bsbI)$. For classification, $\bsby_k$ is a Bernoulli random vector with mean   $ \bsb{\pi}_{k}^*  = 1/(1+\exp(-\bsbX_{k} \bsb{b}_k^*)$ (with the division computed componentwise).
The true coefficients $\boldsymbol{B}^*$ are generated according to the following four setups. We use $(\{a_1\}^{s_1},\ldots, \{a_k\}^{s_k})$ to denote the column vector made by $s_1$ $a_1$'s, $\ldots$ , $s_k$ $a_k$'s consecutively.
\\

\noindent \textbf{Example 1}  {(Classification with moderate $p$.)}
    $n=300$, $p=20$, $m=10$, $r=0.5$. $\boldsymbol{\beta}_j^*=[\{-0.5\}^{3},-0.3,-0.1,0.1,0.3,\allowbreak \{0.5\}^{3}]^{T}, 1\le  j \le 10$, $\boldsymbol{\beta}_j^*=[-0.3,-0.1, \{-0.5\}^{3}, \{0.5\}^{3},0.1,0.3]^{T},10< j\leq 20$.

\noindent \textbf{Example 2}  {(Classification with negative $r$.)}
     $n=600$, $p=4$, $m=20$, $r=-0.5$. $\boldsymbol{\beta}_j^*=[ \{1.0\}^{20}]^{T},  j=1, 2$, $\boldsymbol{\beta}_3^*=[\{-2.0\}^{5},\allowbreak -0.9,-0.7,\ldots,0.9, \{2.0\}^{5}]^{T}$, $\boldsymbol{\beta}_4^*=[\{-3.0\}^{4}, \allowbreak -1.1,-0.9, \ldots,  1.1, \{3.0\}^{4}]^{T}$.

\noindent \textbf{Example 3}  {(Regression with large $n$.)}
    $n=120$, $p=5$, $m=20$, $\sigma=5$, $r = 0.5$. $\boldsymbol{\beta}_1^*=[-285,-255, \ldots,  285]^{T}$, $\boldsymbol{\beta}_2^*=[\{-15\}^{5}, -9,-7, \ldots,9, \{15\}^{5}]^{T}$,
    $\boldsymbol{\beta}_3^*=[\{-10\}^{7},  -6,-4, \ldots,4, \{10\}^{7}]^{T}$, $\boldsymbol{\beta}_4^*=[\{-15\}^{8}, -11,-9, \ldots,7, \{15\}^{2}]^{T}$, $\boldsymbol{\beta}_5^*=[\{10\}^{20}]^{T}$.

\noindent \textbf{Example 4}  {(Regression with small $n$.)}
    $n=25$, $p=30$, $m=20$, $\sigma=1$, $r = 0.5$. $\boldsymbol{\beta}_j^*=[\{-15\}^{5},-9, -7, \ldots,  9, \{15\}^{5}]^{T}, 1\leq j\leq 5$, $\boldsymbol{\beta}_j^*=[-9, -7 ,\ldots, -1,  \{-15\}^{5},  \{15\}^{5}, 1,3,\ldots,9]^{T}, 5< j\leq10$, $\boldsymbol{\beta}_j^*= [\{10\}^{20}]^{T},10< j \leq30$.
\\

The comparison methods include   a clustering  approach based on weighted fused lasso \cite{tang2016fused}, a centered group lasso,  our proposed range-penalized federated learning method, denoted as RFL, Personalized FedAvg (Per-FedAvg) \citep{fallah2020personalized}, and the   Adaptive Personalized Federated Learning (APFL) \citep{deng2020adaptive}. Since the latter two methods, Per-FedAvg and APFL, which represent a meta-learning-based approach and a baseline that adaptively learns a mixture of global and local models, perform less competitively, we focus the discussion on the first three methods.  The first uses the least squares estimator for feature ordering and penalty weighting  in the regression setting (see   \cite{tang2016fused} for more detail). When the sample size is smaller than the problem dimension, the Moore-Penrose inverse is used to calculate the least squares estimator. For practical implementation, we utilize the designated R package for computation.
The second method employs a group $\ell_1$ regularizer
\begin{align}
\lambda \sum_{j=1}^p \| \bsbb_j - c_j \bsb 1\|_2 \label{cglpen}
\end{align} where both $\bsbb_j, c_j$ are unknown. We refer to it as the \textit{centered group lasso} (CGL), designed  to identify homogeneous rows in matrix $\bsb B$.
While the ordinary group lasso is widely used for feature selection, this centered version, with
 $c_j$ optimized jointly, has not   been proposed before to the best of our knowledge. For detecting shared parameters, the underlying  $\ell_1$-type design logic is that each row is either homogeneous or nearly so, and hence all its entries are pulled toward a common center through block soft-thresholding.  By contrast, our range-based regularizer acts through the boundary values, with the extremes serving as the effective shrinkage targets, and often results in a more interpretable model.


  For bias calibration, clustering-based methods use the clusters identified  by    an estimate  $\hat \bsbb_j(\lambda)$ ($1\le j \le p$)  to  impose equality constraints on the calibrated estimation of  $\bsbb_j$. The same is   applicable to the centered group lasso, where    $\bsbb_j$ consist of identical components or remain unconstrained, depending on     $\hat \bsbb_j(\lambda)$. In comparison,   bias correction for RFL includes inequalities. Recall that  an estimate
$\hat \bsbb_j$   categorizes    three index sets---$\hat {\overline{\mathcal M}_j}, \hat {\mathcal M_j}, \hat {\underline{\mathcal M}_j}$---reflecting its structural characteristics.    To adhere to the max/min structure, it is essential to include \textit{order} constraints, along with  the consolidation of predictors. Practically, this can be achieved by formulating  (in)equality constraints:
$$  \gamma_j \bsb1 \preceq \bsbb_j[{\hat {\mathcal M}_j}]\preceq \alpha_{j} \bsb1, \bsbb[{\hat {\overline{\mathcal M}_j}}]= \alpha_{j} \bsb1,  \bsbb[{\hat {\underline{\mathcal M}_j}}]= \gamma_{j} \bsb1.$$ These {linear} constraints       involve   $\bsbb_j, \alpha_j, \gamma_j$ in the optimization process to achieve a calibrated estimate.
Bias calibration can significantly enhance the accuracy of a raw penalized estimate. Notably, during this process, RFL still maintains reduced  range control, a feature that the first two methods do not guarantee.

Given each simulation setting, we  repeat the experiment for 50 times. To ensure  a fair comparison and mitigate the impact of varying   tuning strategies, a large, independent validation dataset comprising 10,000 samples is employed for parameter tuning in each simulation. To assess the true potential of each method,     the best model is  selected based on the prediction error of each calibrated estimate on the large validation data. The primary evaluation metrics are    estimation error ($\text{Err}^{(e)}$), prediction error ($\text{Err}^{(p)}$) and pattern recovery accuracy ($\text{PattAcc}$). The estimation error is quantified  by $\| \hat{\boldsymbol{B}} - \boldsymbol{B}^*\|_F^2$ and the prediction error is calculated by $ \sum_{k=1}^m (\hat{\boldsymbol{b}}_k - \boldsymbol{b}_k^*)^T \bsbSig  (\hat{\boldsymbol{b}}_k - \boldsymbol{b}_k^*)$.  For pattern recovery, we use the normalized mutual information (NMI), expressed as a percentage.   In classification scenarios, the error evaluation also   includes the symmetric KL divergence  which compares $\bsb{\pi}^*$ and $\hat {\bsb{\pi}}$, offering greater sensitivity than plain misclassification error.
Tables \ref{tab:perfclassfication} and \ref{tab:perfregression} report the mean values of the evaluation metrics over 50 experiments. Since Per-FedAvg and APFL exhibit substantially larger errors (see, e.g., Table \ref{tab:perfclassfication}), we do not include them in subsequent comparisons, allowing the discussion to focus on the more competitive methods.

\begin{table}[H]
\centering
\small{
\caption{Classification: performance comparison in terms of estimation error ($\text{Err}^{(e)}$), prediction error ($\text{Err}^{(p)}$), symmetric KL divergence (KL) and pattern recovery accuracy (PattAcc). The numbers in parentheses indicate standard errors. \label{tab:perfclassfication} }
\vspace{.03in}
\setlength{\tabcolsep}{2.2mm}
\begin{tabular}{l cccc}
            \hline
          \multirow{2}{*}{} & \multicolumn{4}{c}{\textbf{Ex 1}} \\  \cmidrule(r){2-5}
         & $\text{Err}^{(e)}$ & $\text{Err}^{(p)}$ &  KL& PattAcc \\ \hline
{Per-FedAvg} & 28.4 (0.1) & 77.2 (0.1) & 1135.6 (3.1) & 41.4  \\
{APFL} & 18.2 (1.2) & 37.5 (3.1) & 601.2 (52) & 42.3  \\

{Clustering} & 13.5 (0.4) & 9.4 (0.3) & 135.3 (4.1) & 84.4 \\
{Centered group lasso} & 16.8 (0.7) & 11.4 (0.4) & 149.4 (5.8) & 88.4 \\
{RFL} & 7.1 (0.2) & 5.1 (0.2) & 75.9 (2.3) & 91.3 \\

        \hline

         \hline
          \multirow{2}{*}{} & \multicolumn{4}{c}{\textbf{Ex 2}} \\  \cmidrule(r){2-5}
         & $\text{Err}^{(e)}$ & $\text{Err}^{(p)}$ & KL & PattAcc \\ \hline
{Per-FedAvg} & 40.43 (0.24) & 21.9 (0.27) & 748.0 (12.7) &38.7  \\
{APFL} & 2.97 (0.22) & 1.73 (0.36) & 60.8 (12.1) &38.6  \\
{Clustering} & 0.63 (0.04) & 0.44 (0.03) &  21.3 (1.3)& 92.3 \\
{Centered group lasso} & 0.89 (0.06) & 0.61 (0.04) & 21.3 (1.0) &89.1  \\
{RFL} & 0.42 (0.02) & 0.30 (0.02) & 14.2 (0.8) &98.2  \\

\hline

\end{tabular}
}
\end{table}

\begin{table}[H]
\centering
\small{
\caption{Regression: performance comparison in terms of estimation error ($\text{Err}^{(e)}$), prediction error ($\text{Err}^{(p)}$) and pattern recovery accuracy (PattAcc). The numbers in parentheses indicate standard errors. \label{tab:perfregression}}
\vspace{.02in}
\setlength{\tabcolsep}{5.3mm}
 \begin{tabular}{l  ccc}
         \hline
          \multirow{2}{*}{} & \multicolumn{3}{c}{\textbf{Ex 3}} \\
          \cmidrule(r){2-4}
         & $\text{Err}^{(e)}$ & $\text{Err}^{(p)}$ & PattAcc \\ \hline

{Clustering} & 21.0 (1.0)  & 15.9 (0.6) & 95.2  \\
{Centered group lasso} & 25.1 (1.0) &  18.7 (0.6) & 80.4  \\
{RFL} & 17.1 (0.7)  & 13.2 (0.5) & 98.4  \\

\hline

\hline
          \multirow{2}{*}{} & \multicolumn{3}{c}{\textbf{Ex 4}} \\
          \cmidrule(r){2-4}
         & $\text{Err}^{(e)}$ & $\text{Err}^{(p)}$ & PattAcc  \\ \hline

{Clustering} & 27.8 (0.9) &  18.9 (0.7) & 93.2  \\
{Centered group lasso} & 25.7 (0.7)  & 16.0 (0.5) & 91.2  \\
{RFL} & 12.2 (0.4) & 7.9 (0.3)  & 99.3  \\

\hline

\end{tabular}
}
\end{table}

In the  experiments, RFL consistently outperforms other methods by achieving the lowest estimation and prediction errors, along with the highest pattern recovery accuracy, especially in the small
$n$ setup of Example 4. 
Moreover, RFL consistently demonstrates the smallest standard errors in all setups, underscoring its stability relative to competing methods.

\subsubsection{Power of Acceleration}
\label{subsubsec:accpower}
The experiments are performed in the setting of  Section \ref{sec:comp}   under       $\breg_L\le \mathcal L\Breg_2$. Here,     $\rho_t$ is set to $\mathcal{L}$, and   a very lenient   bound $\mu_{\min}=  0$    is provided for $\mu_t$. The algorithm is initiated with $\bsba^{(0)} = \bsbb^{(0)} = \bsb0$ and $\theta_0 = 1$.

At iteration $T$, the first step is to update $\theta_T$ and $\bsbg^{(T)}$. We then determine an appropriate value for $\mu_T$ to calculate $\boldsymbol{\alpha}^{(T+1)}$ and $\boldsymbol{\beta}^{(T+1)}$. Implementing a line search strategy for $\mu_T$ is crucial.  As per \eqref{accerrbnd}  in Theorem \ref{th:acc},  we minimize $W_T(\mu)$ defined as   {\small
\begin{align}
& (1-\theta_{T+1}(\mu)) \times \Big\{\Big ({\mbox{\small$\prod$}}_{t=1}^T(1-\theta_t)\Big)\big[(1-\theta_0)(f(\bsb{\beta}^{(0)})-f(\bsbb_{o}^{(T)}))+\theta_0^2\rho_0 \Breg_2(\bsbb_{o}^{(T)},\bsb{\beta}^{(0)})\big]\notag\\
 &\qquad\qquad\qquad\qquad-   {\mbox{\large$\Sigma$}}_{t=0}^T \big(\Pi_{s=t+1}^T (1-\theta_s)\big) (R_t' (\mu)+ \theta_t \mathcal E_t(\bsbb_{o}^{(T)}; \mu))\Big\}. \end{align}}Here, $\bsbb_o^{(T)} = \arg\min_{0\le t\le T} f(\bsbg^{(t)})$, the iterate with the minimum $f$-value from $0$ to $T$.
The search for $\mu_T$ can start with  increasing  $\mu_{T-1}$ by a factor of $\overline r$ or decreasing it by $\underline r $. The number of searches per iteration is limited to $M $.  In the experiments, we use $\overline r = 1.5, \underline r = 1.3, M=6$.

The line search for $\mu_t$ is not required at every iteration.   Initially, we use the default $\mu_{\min}$ for several iterations until  the number of iterations reaches a threshold (e.g., 15) or the relative change is small     (e.g., 5e-3). Moreover, once the convexity parameter stabilizes---for instance, when the average of the past three $\mu_t$-values closely matches the current one or the iterates show small changes---we can maintain $\mu_t$    at the current optimal value.

We test the accelerated algorithm and compare it with standard proximal gradient descent, using the settings from previous examples but with a smaller sample size of $n=15$.  The algorithm stops when the difference between the objective function value and the optimal value is less than 0.1.   Figure  \ref{fig:accPGDperf}   illustrates the effectiveness of applying the acceleration through a data example from each setting.
Table \ref{tab:accperf} presents more comprehensive results, reporting the median number of iterations from 50 experiments. Additionally, in terms of overall computational time, our accelerated algorithm can save at least 40\%.
Finally, we emphasize that our acceleration scheme employs a trivial   $\mu_{\min}=0$, \emph{without} relying on any precise convexity parameter information, as deriving a tight value can often be challenging.   Nonetheless, our algorithm adaptively searches for local restricted strong convexity and substantially reduces the number of iterations.
\begin{figure*}[!t]
\centering
\subfloat{\includegraphics[width=0.48\textwidth]{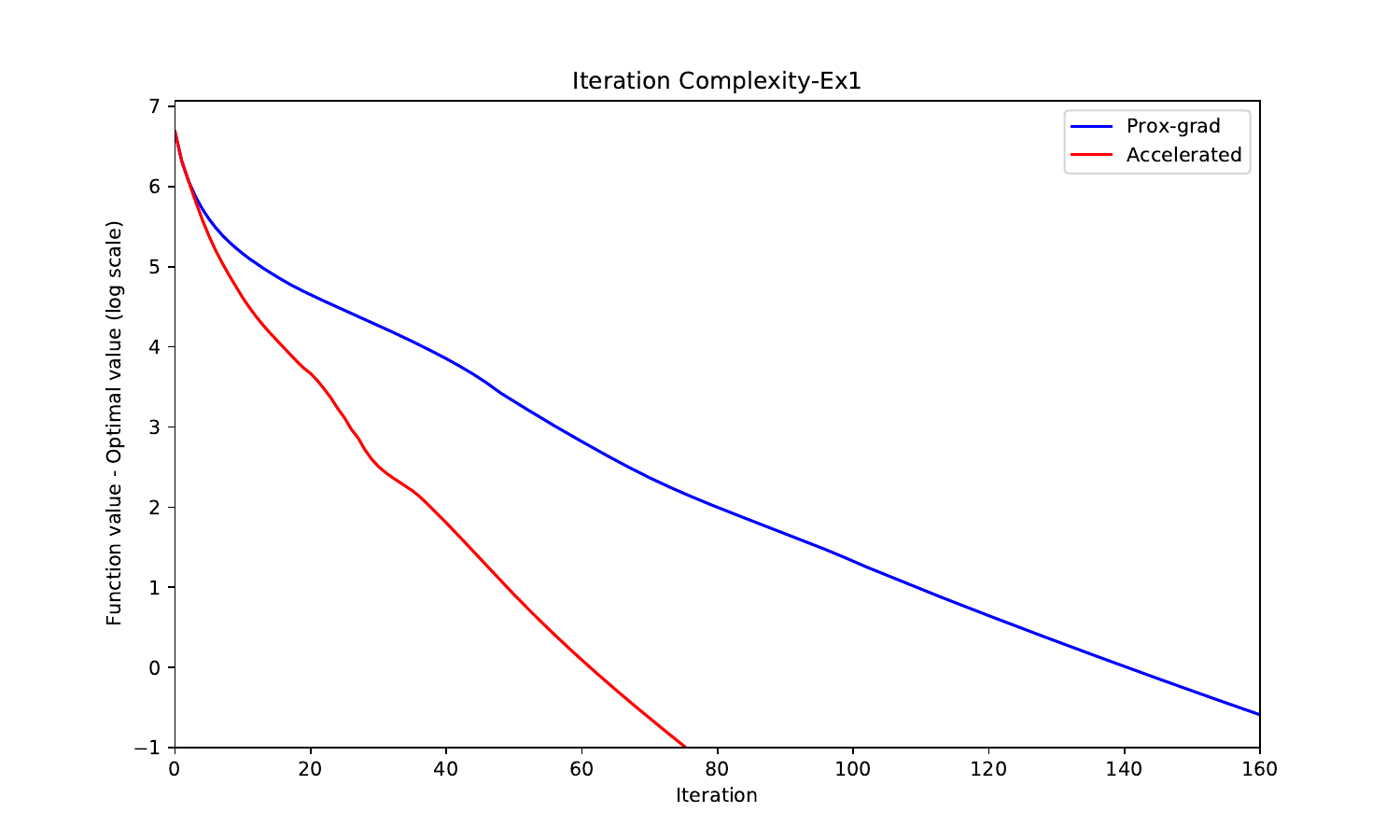}}
\hfil
\subfloat{\includegraphics[width=0.48\textwidth]{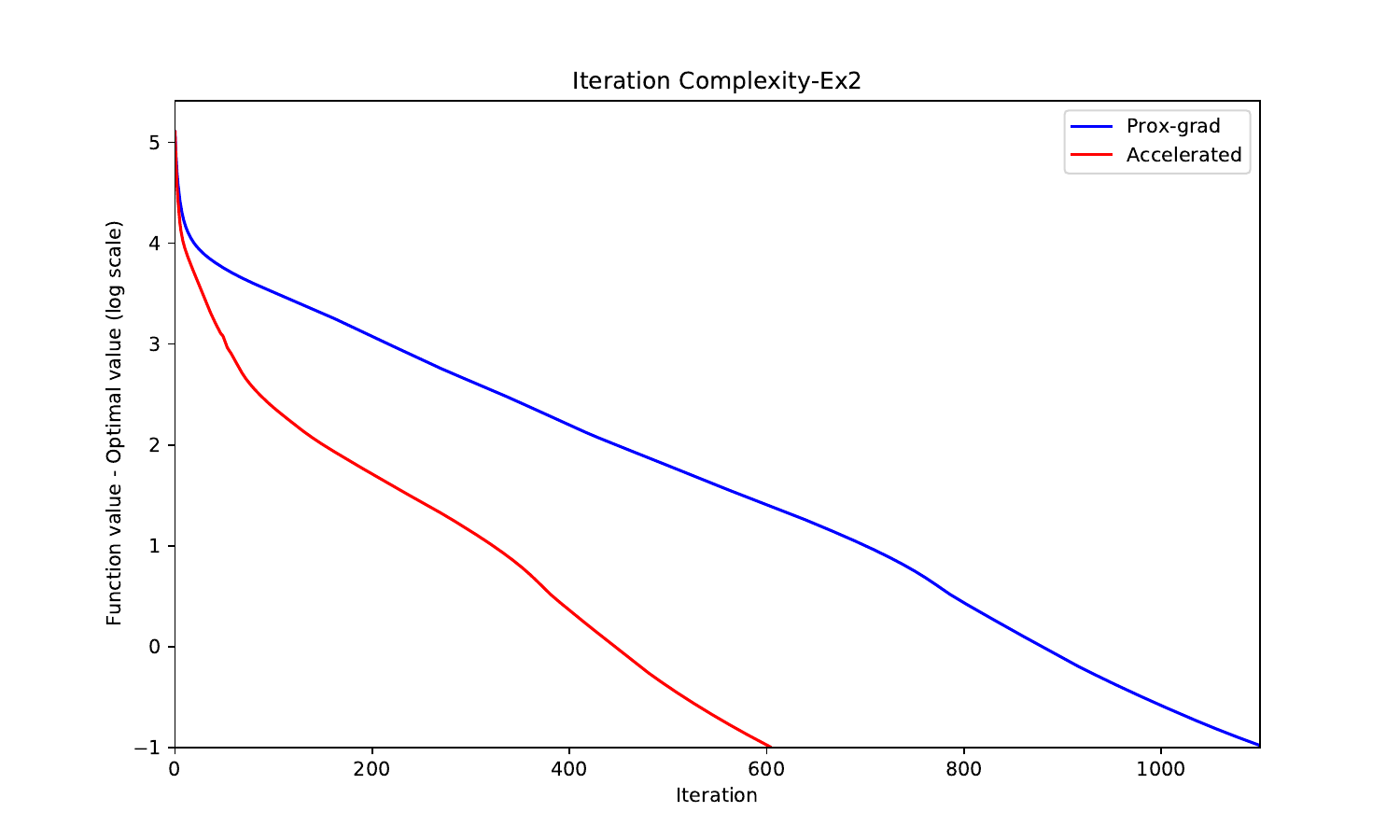}}
\vspace{0.8cm}
\subfloat{\includegraphics[width=0.48\textwidth]{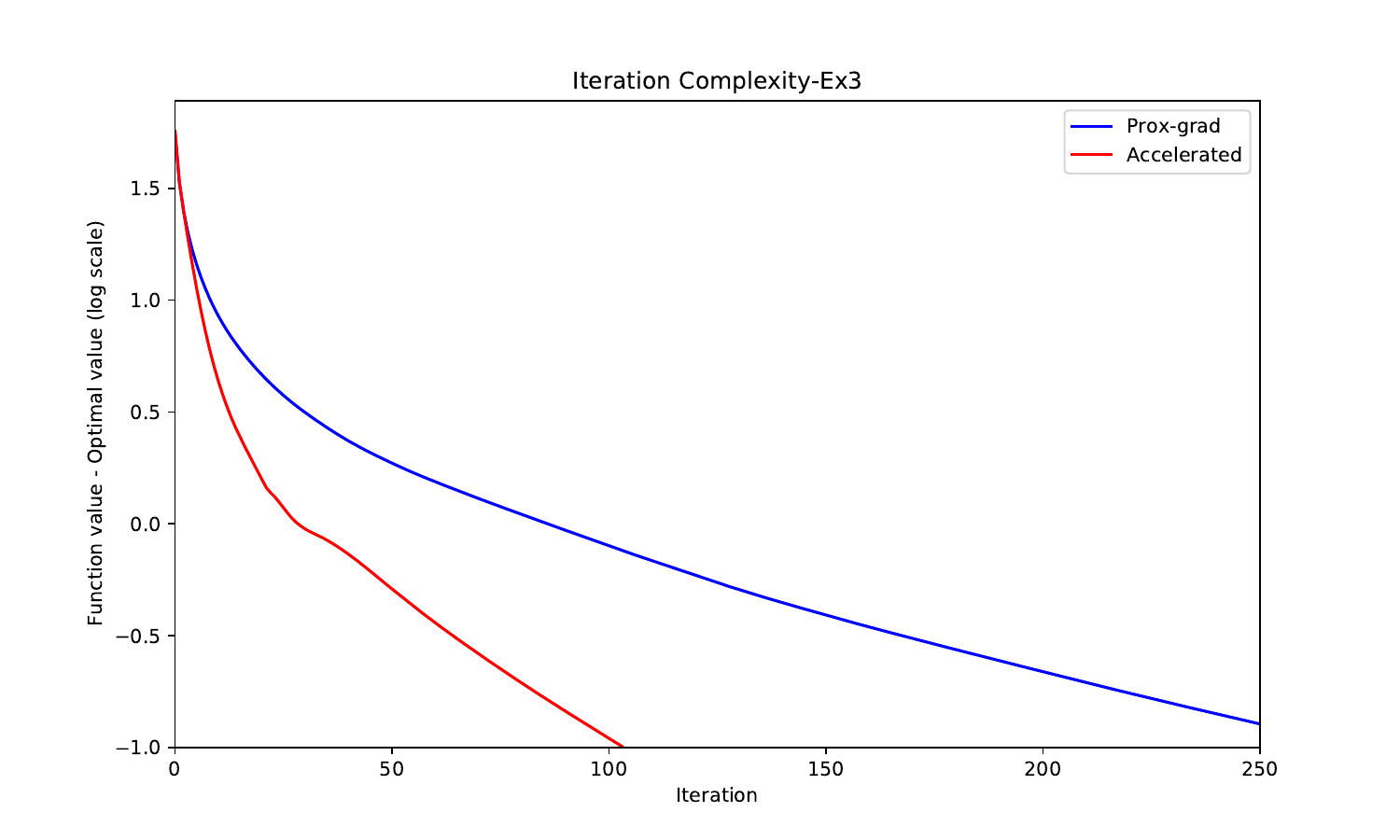}}
\hfil
\subfloat{\includegraphics[width=0.48\textwidth]{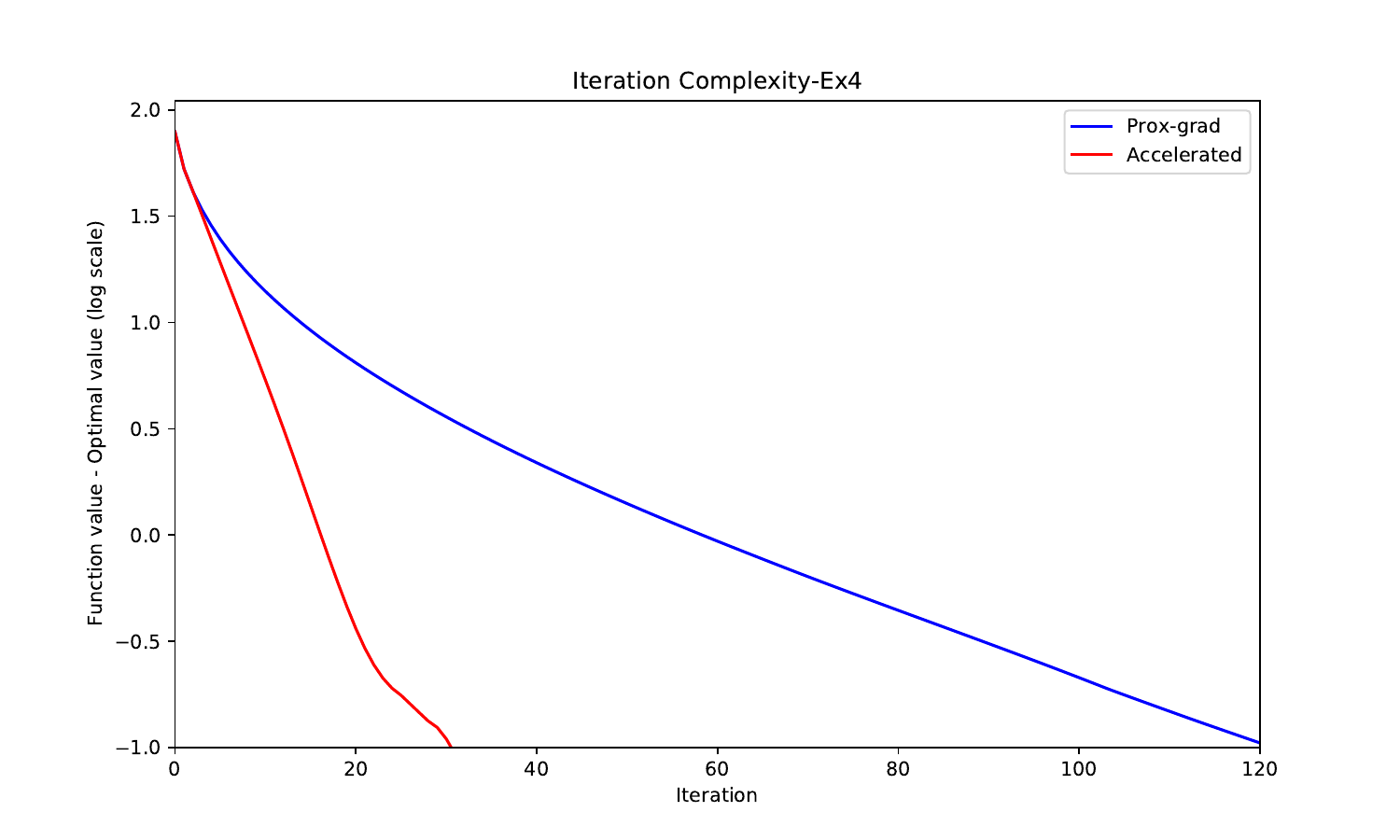}}
\caption{Log-scaled difference between the objective function value and the optimal value vs. number of iterations for both proximal gradient and the proposed acceleration.}
\label{fig:accPGDperf}
\end{figure*}

\begin{table}[H]
\centering
{
\caption{Median number of iterations across 50 experiments (with $n=15$). \label{tab:accperf}}
\vspace{.03in}
\setlength{\tabcolsep}{6.5mm}
 \begin{tabular}{c cccc  }
         \hline

         &  {Ex1} &  {Ex2} &  {Ex3} &  {Ex4}  \\    \cmidrule(r){1-3} \cmidrule(l){4-5}

Prox-grad   & 118  & 405   & 282& 90  \\
    Accelerated & 60  & 181  & 136 & 28\\

\hline
\end{tabular}
}
\end{table}

\subsubsection{Empirical Evidence on Range Reduction}
\label{subsubsec:range}

We provide additional empirical evidence that the proposed RFL method yields favorable range reduction in transmitted signals in practice. Since tighter numerical range is advantageous for quantization and coding, this offers further support for the communication-related utility of range regularization.

We compare RFL with FedPer \citep{arivazhagan2019federated}, a representative partial-personalization method. We consider the second classification setting in the simulation study and use a validation set of size 10,000 to tune
$\lambda$. To form a strong benchmark, FedPer is implemented using the ground-truth shared-feature indices from the true parameter matrix. (This gives FedPer ideal structural information unavailable in practice, making the comparison particularly stringent.)  Both algorithms are terminated when the absolute difference between two consecutive objective values is below 1e-3.
For each client and each method, we compute a communication-range measure by summing the parameter ranges appearing in both the uplink and downlink channels over the entire optimization path. We then estimate, for each client, the ratio of FedPer's total communication range to that of RFL, together with its standard error under a two-sample setup. Ratios consistently above one indicate that RFL produces a narrower numerical profile throughout the communication process.

\begin{table}[htbp]
\centering
\caption{Estimated range ratios (FedPer over RFL) for the 20 clients under the setting of Example 2. Values greater than one indicate that RFL yields a smaller overall numerical range in communication. Standard errors are shown in subscripts.}
\label{tab:range_ratio_compare}

\setlength{\tabcolsep}{4pt}
\resizebox{\linewidth}{!}{
\begin{tabular}{ccccccccccc}
\toprule
{Client index} & {1} & {2} & {3} & {4} & {5} & {6} & {7} & {8} & {9} & {10} \\ \midrule
{Ratio} & $1.54_{\scriptscriptstyle 0.05}$ & $1.54_{\scriptscriptstyle 0.05}$ & $1.67_{\scriptscriptstyle 0.05}$ & $1.68_{\scriptscriptstyle 0.05}$ & $1.67_{\scriptscriptstyle 0.05}$ & $1.13_{\scriptscriptstyle 0.03}$ & $1.14_{\scriptscriptstyle 0.03}$ & $1.12_{\scriptscriptstyle 0.03}$ & $1.13_{\scriptscriptstyle 0.03}$ & $1.20_{\scriptscriptstyle 0.04}$ \\
\midrule \midrule
{Client index} & {11} & {12} & {13} & {14} & {15} & {16} & {17} & {18} & {19} & {20} \\ \midrule
{Ratio} & $1.14_{\scriptscriptstyle 0.02}$ & $1.20_{\scriptscriptstyle 0.03}$ & $1.22_{\scriptscriptstyle 0.02}$ & $1.32_{\scriptscriptstyle 0.02}$ & $1.46_{\scriptscriptstyle 0.02}$ & $1.55_{\scriptscriptstyle 0.03}$ & $3.03_{\scriptscriptstyle 0.06}$ & $2.96_{\scriptscriptstyle 0.06}$ & $2.97_{\scriptscriptstyle 0.06}$ & $3.63_{\scriptscriptstyle 0.06}$ \\
\bottomrule
\end{tabular}
}
\end{table}

According to Table \ref{tab:range_ratio_compare}, all ratios are clearly greater than one, showing that the total transmitted range under FedPer is uniformly larger than that under RFL. This is noteworthy because FedPer is given oracle knowledge of the true shared structure, yet RFL still yields substantially tighter communication signals in practice. This provides direct numerical evidence that RFL produces substantially tighter communication signals in practice.

We also examine the uplink and downlink channels within RFL separately. Figure \ref{fig:rfl_range_comparison} compares the parameter ranges across iterations for the uplink and downlink channels under RFL. As seen in the figure, the uplink and downlink ranges are of similar magnitude across clients, indicating that the practical range benefit is not confined to the downlink alone.
This is consistent with the iterative structure of the algorithm, where the communicated quantities are repeatedly updated in conjunction with the range-regularized estimates. As a result, the favorable numerical profile is reflected in both communication channels throughout the optimization process.

\begin{figure}[!t]
    \centering
    \includegraphics[width=0.9\linewidth]{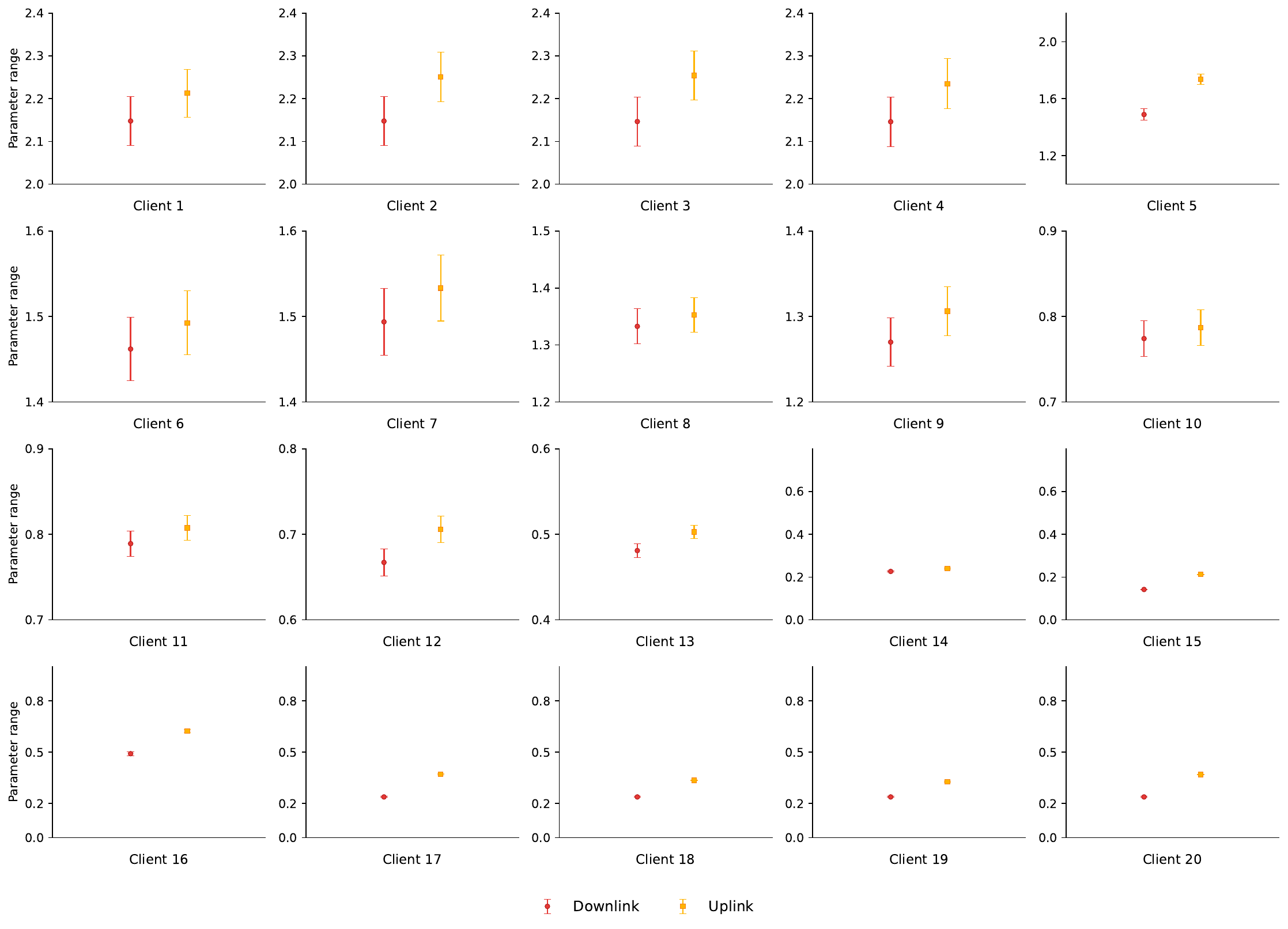}
    \caption{Parameter ranges across iterations for the downlink and uplink channels under RFL for the 20 clients in Example 2. 
}
    \label{fig:rfl_range_comparison}
\end{figure}

\subsubsection{Robustness to a Corrupted Client: Group $\ell_1$-Shrinkage versus Range}
\label{subsubsec:rob}
As suggested by a reviewer, it is informative to compare  $\ell_1$-type shrinkage and range regularization in the presence of a single corrupted or adversarial client. We consider the second classification setting in the previous subsection and rescale the feature matrix of the first client by a factor $c$, with $c \in \{1,1.25,1.5,1.75,2\}$. As $c$ increases, the corresponding client becomes progressively more outlying, so that its coefficient vector can dominate the overall range and potentially distort estimation.
We use estimation error and symmetric KL divergence as the primary evaluation metrics, the latter being more sensitive than plain misclassification error in this robustness study.

\begin{figure*}[!t]
\centering
\subfloat{\includegraphics[width=0.48\textwidth]{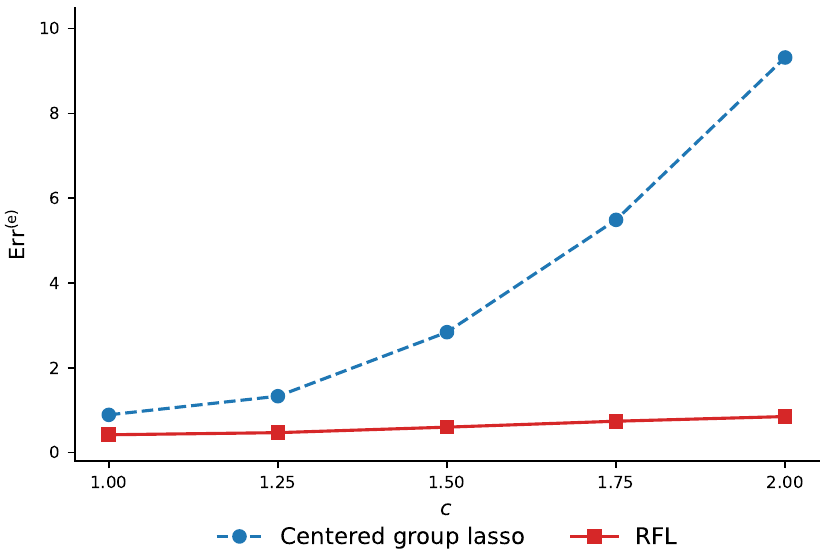}}
\hfil
\subfloat{\includegraphics[width=0.48\textwidth]{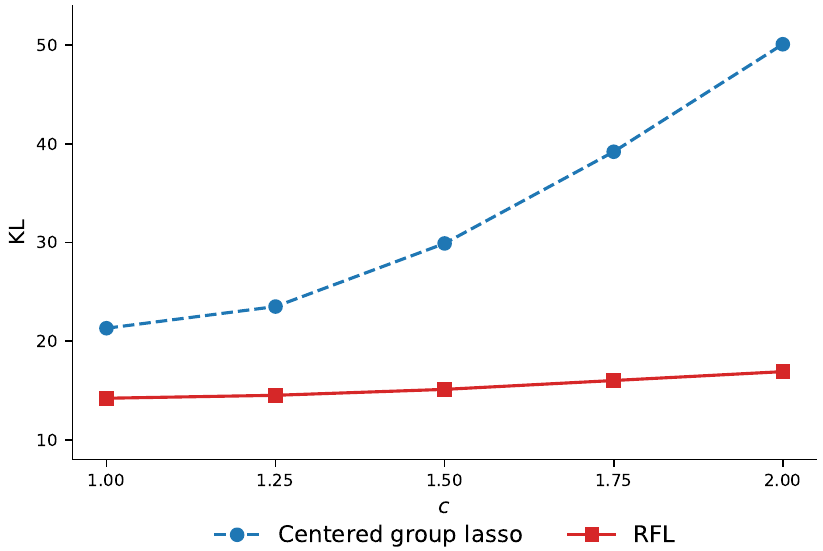}}
\caption{Comparison of centered group lasso and RFL in terms of estimation error and symmetric KL divergence as the outlying level of the first client, indexed by the rescaling factor $c$, increases.}
\label{fig:robustness}
\end{figure*}

Figure~\ref{fig:robustness} reports the mean results over 50 experiments. As
$c$ increases, the first client becomes increasingly outlying. The centered group lasso is pretty sensitive to this contamination: both its estimation error and symmetric KL divergence rise sharply with
$c$. In contrast, RFL remains much more stable across the entire range of perturbation levels. Overall,  range regularization is much less affected by a single corrupted client than group
$\ell_1$- shrinkage.

\subsection{Communities and crime data}
\label{subsec:exptcrime}
In this part, we analyze the communities and crime dataset \citep{miscCommunitiesAndCrime2009}. The response is the total number of violent crimes in 1995, with other variables derived from socioeconomic and law enforcement data. We selected 12 relevant predictors,  all standardized, which are commonly used in related studies \citep{sinha2015multiobjective,yang2019high,samara2023using}. Details on these variables are provided in Table \ref{tab:crimes_variables}. Given the political sensitivity associated with the dissemination of such data across different counties, the dataset is well suited for federated learning. 
After excluding observations with missing values and counties with fewer than ten observations, our final dataset comprises 571 observations distributed across 31 counties (clients), with individual county sample sizes ranging from 11 to 76.

 \begin{table}[H]
\footnotesize
\centering
\caption{List of variables and their descriptions\label{tab:crimes_variables}}
\begin{tabular}{l|l}
\hline
Variables & Description\\
\hline
{\tt ViolentCrimesPerPop} & Total number of violent crimes per 100K population\\
{\tt racepctblack} & Percentage of population that is African American\\
{\tt agePct12t21} & Percentage of population that is 12-21 in age\\
{\tt NumUnderPov} & Number of people under the poverty level\\
{\tt PctUnemployed} & Percentage of people 16 and over, in the labor force, and unemployed\\
{\tt PctKids2Par} & Percentage of kids in family housing with two parents\\
{\tt PctWorkMom} & Percentage of moms of kids under 18 in labor force\\
{\tt PctKidsBornNeverMar} & Percentage of kids born to never married\\
{\tt PctRecImmig5} & Percent of population who have immigrated within the last 5 years\\
{\tt PctSpeakEnglOnly} & Percent of people who speak only English\\
{\tt PctPersDenseHous} & Percent of persons in dense housing (more than 1 person per room)\\
{\tt HousVacant} & Number of vacant households\\
{\tt PctVacantBoarded} & Percent of vacant housing that is boarded up\\
\hline
\end{tabular}
\end{table}

This study evaluates the same three  methods based on   clustering, centered group lasso, and the proposed  RFL, detailed in Appendix \ref{subsec:simus}. For the clustering method, we use the  R package \citep{tang2016fused} for both computation and parameter tuning.
 Both the centered group lasso and RFL use 5-fold structural cross-validation \citep{she2019cross} on the training data for tuning.
For all compared methods, we scale each $ \bsbX_k$ by $1/\sqrt{n_k}$ during training and transform the fitted coefficients back when evaluating test prediction errors. For both RFL and centered group lasso, we use least-squares-based span weights, analogous to adaptive-lasso weights, to align coefficient spans across predictors.
Data are randomly split into training (80\%) and testing (20\%) sets. This process is repeated 20 times to validate the predictive performance of the methods.
To compare the performance of the  methods, we consider prediction error, the number of free parameters, and the range of the parameters as evaluation metrics. Concretely, due to the highly non-Gaussian nature of empirical prediction errors, we report 40\% trimmed means   of the prediction errors (multiplied by 100) across all experiments. The number of free parameters is calculated by the count of distinct components in $\hat{\bsb B}$, while the parameter range is defined by the median of the spans of the rows in $\hat{\bsb B}$ (excluding any uniform rows). We report their median numbers across various experiments. The detailed results are presented in Table \ref{tab:crimes}.

\begin{table}[H]
\footnotesize
\centering
\caption{\label{tab:crimes}Communities and crime data.}
\vspace{.03in}

\begin{tabular}{lccc}
    \hline
     & prediction error & \# of free params & param range \\
    \hline
    Clustering                    & 44   & 119  & 0.6   \\
    Centered group lasso          & 35   & 192  & 3.9   \\
    RFL                           & 25   & 32  & 0.2   \\
    \hline
\end{tabular}
\end{table}

According to the table, the centered group lasso outperforms  the clustering approach via weighted fused lasso in prediction accuracy. However, RFL achieves the lowest prediction error, cutting the centered group lasso's error by approximately        30\%. Furthermore, RFL results in only 32 free parameters---80\% fewer than the other methods. Additionally, it yields a minimal range measure of 0.2. Overall, RFL demonstrates superior test performance with a significantly smaller model size.

In our experiment, the clustering method produced many small clusters but failed to identify any row as homogeneous. In contrast,  the centered group lasso and RFL identified 6 and 8 uniform rows, respectively, suggesting the presence of shared parameters consistent across various clients' datasets.  For instance, the variable {\tt PctWorkMom}---the percentage of working mothers with children under 18---shows an interesting uniformly negative coefficient impacting crime rates across various counties. This finding is supported by other research \citep{yang2019high} and can be attributed to the universal effects of employment on reducing crime by boosting economic stability and improving access to community resources, regardless of regional variations.

RFL also identified several rows with polar clusters, such as {\tt NumUnderPov}, the number of people under poverty. The cluster with the highest values includes 13 counties, such as St. Louis, MO, and Orange, CA, exhibiting a positive correlation with crime rates. This correlation may stem from significant economic inequality and restricted access to social services during the data collection period; for further discussion, see \cite{matsueda2014social}. 

Conversely, a distinct cluster of 15 counties, including Dallas, TX, and Seattle, WA, displays the smallest negative coefficient, indicating that higher poverty levels correlate with lower crime rates. This phenomenon could be attributed to effective government responses and extensive social safety nets in these areas during the 1990s  \citep{watrus2007seattle}. 

\subsection{Air quality data}
\label{subsec:exptaiquality}

This subsection presents a classification experiment using the Beijing multi-site air quality dataset \citep{miscBeijingMultiSiteAirQuality2019}, focusing on daily averaged data from April to June 2015. PM2.5 concentrations are categorized into high pollution (coded as 1) for 24-hour averages of 35 $\mu \text{g/m}^3$ or above, and low pollution (coded as 0) for lower values, adhering to international standards \citep{chae2021pm10}.   The predictors include concentrations of $\text{SO}_2$, $\text{NO}_2$, $\text{CO}$, and $\text{O}_3$, as well as temperature, pressure, dew point temperature, and wind speed. Data were collected from 12 nationally-controlled air-quality monitoring sites in Beijing, which are treated as clients. The geographical locations of these sites are depicted in Figure \ref{fig:map_beijing}.
We use the same methods, evaluation metrics, and initial-estimator-based span weights in Appendix \ref{subsec:exptcrime}, replacing least squares by maximum likelihood for this classification problem. In addition to the clustering method, centered group lasso, and RFL, we include several further baselines following a reviewer's suggestion: a global model in which all parameters are shared across clients (denoted by `global'), a fully local model in which every feature coefficient is client-specific, with no parameter sharing  (denoted by `local'), and a random-sharing model in which each feature is independently assigned, with probability $1/2$, to be either shared across all clients or client-specific (denoted by `random').  We also include a fine-tuning baseline based on FedAvg, which is a representative personalization baseline in the federated learning literature \citep{jiang2019improving,WLC22}. Specifically, one first trains a fully shared model and then perform local fine-tuning without further parameter sharing; following \cite{jiang2019improving}, each client runs 10 local fine-tuning steps.
The experimental results are reported in Table \ref{tab:airquanlityres}.

\begin{figure}[!t]
    \centering
    \includegraphics[width=0.7\linewidth, height=3in]{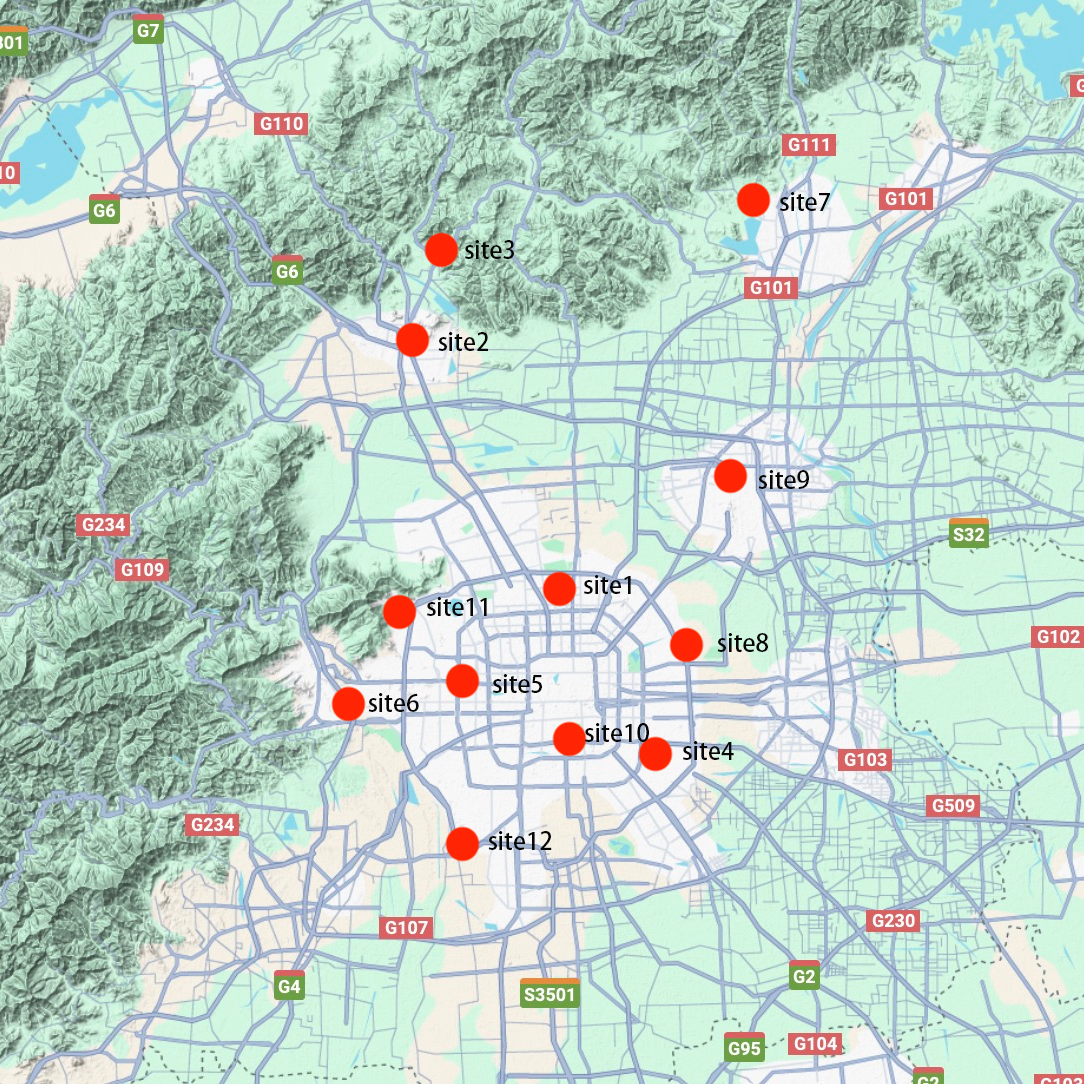}
    \caption{Geographical location of 12 monitoring sites.}
    \label{fig:map_beijing}
\end{figure}

 According to Table \ref{tab:airquanlityres}, the first three baselines support the value of \emph{selective} parameter sharing. The global, random, and local baselines yield misclassification error rates of 13.7\%, 14.4\%, and 16.1\%, respectively, showing that neither sharing everything, sharing nothing, nor assigning sharing patterns at random is optimal; rather, identifying which coefficients should be shared is essential.
The fine-tuning baseline improves upon the random and local baselines, indicating that initializing from a shared model is helpful, but still remains noticeably inferior to the more structured approaches. Among the latter, clustering  achieves a small model size but does not offer the best predictive performance, while centered group lasso provides a stronger trade-off. RFL attains the lowest misclassification error rate and, compared with centered group lasso, uses substantially fewer free parameters and a smaller parameter range, indicating a more effective and economical personalization pattern.

\begin{table}[H]
\footnotesize
\centering
\caption{\label{tab:airquanlityres}Air quality data.}
\vspace{.02in}
\begin{tabular}{cccc}
    \hline
     & misclassification error rate\,(\%) & \# of free params & param range \\
    \hline
    Global & 13.7   & 8  & --   \\
    Random & 14.4   & 30  & 1.6   \\
    Local & 16.1   & 96  & 8.9   \\
    Fine-tuning & 14.2   & 96  & 0.7   \\

    Clustering                    & 13.2   & 9  & 0.2   \\
    Centered group lasso          & 13.0   & 41  & 0.8   \\
    RFL                           & 11.8   & 14  & 0.5   \\
    \hline
\end{tabular}
\end{table}

Beyond  accuracy, we next examine model interpretability.
A more careful examination reveals that the centered group lasso and RFL identified precisely the same set of shared parameters. For instance,  pressure and temperature exhibit uniform  coefficients across various datasets. Given that all sites are located within the municipal area of Beijing, this uniformity is perhaps not surprising.

Moreover, RFL identified two polar clusters with opposing coefficient signs for wind speed. Sites 2, 3, 7, and 9, located in the northern mountainous regions, showed positive maximum coefficients. These areas are less affected by local pollution sources and the mountainous terrain promotes the dispersion of pollutants \citep{Han2015meteorological}. Consequently, higher wind speeds in these regions are associated with reduced PM2.5 levels due to the effective dispersion of pollutants \citep{song2023impacts}.
In contrast, sites 1, 4, 5, 6, 8, 10, 11, and 12 are situated in the plains of Beijing and displayed  negative minimum coefficients for wind speed. The flat terrain in these areas hinders the natural dispersion of pollutants and increased wind speeds may   trap pollutants in the area \citep{wang2019research}.  These findings suggest a potential influence of geographic and environmental factors on local air quality dynamics.

{
\bibliographystyle{IEEEtranN}
\bibliography{range}
}

\end{document}